\documentclass[twoside,11pt]{article}
\usepackage{subfigure}
\usepackage{float}
\usepackage[all]{xy}
\usepackage{enumitem}
\usepackage{algorithmic} 
\setlist[enumerate]{itemsep=0mm}
\usepackage{color}
\usepackage{jmlr2e}
\usepackage{wrapfig}
\usepackage{graphics}
\usepackage{epsfig}
\usepackage{float}
\usepackage[ruled]{algorithm2e}
\usepackage{booktabs}
\usepackage{url}
\usepackage{psfrag}
\usepackage{relsize}
\usepackage{amsmath,amsfonts,mathrsfs,mathtools,amssymb}
\usepackage{hyperref}
\usepackage{tikz,pgfplots}
\usepackage{tkz-tab}
\usepackage[format=hang,justification=justified,singlelinecheck=false,font=footnotesize]{caption}
\usepackage[T1]{fontenc}
\usepackage{booktabs}
\usepackage{threeparttable} 
\usepackage{multirow}

\newcommand{\eins}{\boldsymbol{1}}
\DeclareSymbolFont{wideparensymbol}{OMX}{yhex}{m}{n}
\DeclareMathAccent{\wideparen}{\mathord}{wideparensymbol}{"F3}

\newcommand{\argmin}{\operatornamewithlimits{arg \, min}}

\newtheorem{assumption}{Assumption}

\makeatletter

\newcommand{\Rmnum}[1]{\expandafter\@slowromancap\romannumeral #1@}
\makeatother

\usepackage{lastpage}
\jmlrheading{26}{2025}{1-\pageref{LastPage}}{11/23; Revised
3/25}{7/25}{23-1519}{Yuchao Cai, Hanfang Yang, Yuheng Ma, Hanyuan Hang}
\ShortHeadings{Bagged Regularized $k$-Distances for Anomaly Detection}{Cai, Yang, Ma, Hang}

\usepackage{tikz,pgfplots}
\pgfplotsset{compat=newest}
\newlength\figureheight
\newlength\figurewidth
\usetikzlibrary{shapes,arrows}

\usepackage{flowchart}

\begin{document}

\title{Bagged Regularized $k$-Distances for Anomaly Detection}

\author{\name Yuchao Cai \email yccai@nus.edu.sg\\
\addr Department of Statistics and Data Science \\
National University of Singapore \\
117546, Singapore
\vspace{0.2cm}
\\
\name Hanfang Yang \email hyang@ruc.edu.cn \\
\addr Center for Applied Statistics and School of Statistics \\ 
Renmin University of China \\
100872 Beijing, China 
\vspace{0.2cm}
\\
\name Yuheng Ma \email yma@ruc.edu.cn \\
\addr School of Statistics \\ 
Renmin University of China \\
100872 Beijing, China 
\vspace{0.2cm}
\\
\name Hanyuan Hang \email hanyuan0725@gmail.com\\
\addr Hong Kong Research Institute \\
Contemporary Amperex Technology (Hong Kong) Limited \\
Hong Kong Science Park, New Territories, Hong Kong
}

\editor{Aryeh Kontorovich}

\allowdisplaybreaks

\maketitle



\begin{abstract}We consider the paradigm of unsupervised anomaly detection, which involves the identification of anomalies within a dataset in the absence of labeled examples.
Though distance-based methods are top-performing for unsupervised anomaly detection, they suffer heavily from the sensitivity to the choice of the number of the nearest neighbors.
In this paper, we propose a new distance-based algorithm called \textit{bagged regularized $k$-distances for anomaly detection} (\textit{BRDAD}), converting the unsupervised anomaly detection problem into a convex optimization problem.
Our BRDAD algorithm selects the weights by minimizing the \textit{surrogate risk}, i.e., the finite sample bound of the empirical risk of the \textit{bagged weighted $k$-distances for density estimation} (\textit{BWDDE}).
This approach enables us to successfully address the sensitivity challenge of the hyperparameter choice in distance-based algorithms.
Moreover, when dealing with large-scale datasets, the efficiency issues can be addressed by the incorporated bagging technique in our BRDAD algorithm.
On the theoretical side, we establish fast convergence rates of the AUC regret of our algorithm and demonstrate that the bagging technique significantly reduces the computational complexity.
On the practical side, we conduct numerical experiments to illustrate the insensitivity of the parameter selection of our algorithm compared with other state-of-the-art distance-based methods. Furthermore, our method achieves superior performance on real-world datasets with the introduced bagging technique compared to other approaches.
\end{abstract}

\begin{keywords}
Unsupervised learning,
density estimation, 
anomaly detection, 
weighted $k$-distances,
regularization,
surrogate risk minimization (SRM),
bagging,
ensemble learning,
optimal convergence rates, 
learning theory
\end{keywords}

\section{Introduction} \label{sec::intro}

Anomaly detection refers to the process of identifying patterns or instances that deviate significantly from the expected behavior within a dataset \citep{chandola2009anomaly}.  
It has been widely and carefully studied within diverse research areas and application domains, including industrial engineering \citep{fahim2019anomaly,wang2021anomaly}, medicine \citep{fernando2021deep,tschuchnig2021anomaly}, cyber security \citep{folino2023ensemble,ravinder2023review}, earth science \citep{luz2022mapping, chen2022reference}, and finance \citep{lokanan2019detecting,hilal2022financial}, etc. For further discussions on anomaly detection techniques and applications, we refer readers to the survey of \citet{nassif2021machine}.

Based on the availability of labeled data, anomaly detection problems can be classified into three main paradigms.
The first is the supervised paradigm, where both the normal and anomalous instances are labeled. As mentioned in \citet{aggarwal2016introduction} and \citet{vargaftik2021rade}, researchers often employ existing binary classifiers in this case.
The second is the semi-supervised paradigm, where the training data only consists of normal samples, and the goal is to identify anomalies that deviate from the normal samples. \citep{akcay2018ganomaly, zhou2023semi}. 
Perhaps the most flexible yet challenging paradigm is the unsupervised paradigm \citep{aggarwal2016introduction, gu2019statistical}, where no labeled examples are available to train an anomaly detector.
For the remainder of this paper, we only focus on the unsupervised paradigm, where we do not assume any prior knowledge of labeled data.

The existing algorithms in the literature on unsupervised anomaly detection can be roughly categorized into three main categories:
The first category is distance-based methods, which determine an anomaly score based on the distance between data points and their neighboring points.
For example, $k$-nearest neighbors ($k$-NN) \citep{ramaswamy2000efficient} calculate the anomaly score of an instance based on the distance to its $k$-th nearest neighbor,
distance-to-measure (DTM) \citep{gu2019statistical} introduces a novel distance metric based on the distances of the first $k$-nearest neighbors, and
local outlier factor (LOF) \citep{breunig2000lof} computes the anomaly score by quantifying the deviation of the instance from the local density of its neighboring data points.
The second category is forest-based methods, which compute anomaly scores based on tree structures. 
For instance, isolation forest (iForest) \citep{liu2008isolation} constructs an ensemble of trees to isolate data points and quantifies the anomaly score of each instance based on its distance from the leaf node to the root in the constructed tree and
partial identification forest (PIDForest) \citep{gopalan2019pidforest} computes the anomaly score of a data point by identifying the minimum density of data points across all subcubes partitioned by decision trees.
The third category is kernel-based methods such as the one-class SVM (OCSVM) \citep{scholkopf1999support}, which defines a hyperplane to maximize the margin between the origin and normal samples.
It has been empirically shown \citep{aggarwal2015theoretical, aggarwal2016outlier, gu2019statistical} that distance-based and forest-based methods are the top-performing methods across a broad range of real-world datasets.
Moreover, experiments in \citet{gu2019statistical} suggest that distance-based methods show their advantage on high-dimensional datasets, as forest-based methods are likely to neglect a substantial number of features when dealing with high-dimensional data.
Unfortunately, it is widely acknowledged that distance-based methods suffer from the sensitivity to the choice of the hyperparameter $k$ \citep{aggarwal2012applications}. 
This problem is particularly severe in unsupervised learning tasks because the absence of labeled data makes it difficult to guide the selection of hyperparameters. 
To the best of our knowledge, no algorithm in the literature effectively solves the aforementioned sensitivity problem.
Besides, while distance-based methods are crucial and efficient for identifying anomalies, they pose a challenge in scenarios with a high volume of data samples, owing to the need for a considerable expansion in the search for nearest neighbors, leading to a notable increase in computational overhead. 
Therefore, there also remains a great challenge for distance-based algorithms to improve their computational efficiency.

In this paper, we propose a distance-based algorithm, \textit{bagged regularized $k$-distances for anomaly detection} (\textit{BRDAD}), which formulates the weight selection problem in unsupervised anomaly detection as a minimization problem.
Specifically, we first establish the \textit{surrogate risk}, a finite-sample bound on the empirical risk of the \textit{bagged weighted $k$-distances for density estimation} (\textit{BWDDE}). At each bagging round, we determine the weights by minimizing the surrogate risk on a subsampled dataset. Then, using an independently drawn subsample of the same size, we compute the corresponding $k$-distances. By combining the learned weights and these $k$-distances, we obtain the \textit{regularized $k$-distance}. The final anomaly scores are derived by averaging these regularized $k$-distances, referred to as \textit{bagged regularized $k$-distances}. BRDAD ranks the data in descending order of these scores and identifies the top $m$ instances as anomalies.
BRDAD offers two key advantages. First, the \textit{surrogate risk minimization} (\textit{SRM}) approach effectively mitigates the sensitivity of parameter choices in distance-based methods. Second, the incorporation of bagging enhances computational efficiency, making the method scalable for large datasets.

The contributions of this paper are summarized as follows.

\textit{(i)} We propose a new distance-based algorithm BRDAD, that prevents the sensitivity of the hyperparameter selection in unsupervised anomaly detection problems by formulating it as a convex optimization problem. 
Moreover, the incorporated bagging technique in BRDAD improves the computational efficiency of our distance-based algorithm.

\textit{(ii)} From the theoretical perspective, we establish fast convergence rates of the AUC regret of BRDAD.
Moreover, we show that with relatively few bagging rounds $B$, the number of iterations in the optimization problem at each bagging round can be reduced substantially.
This demonstrates that the bagging technique significantly reduces computational complexity.

\textit{(iii)} From an experimental perspective, we conduct numerical experiments to evaluate the effectiveness of our proposed BRDAD method. First, we empirically validate the reasonableness of SRM by demonstrating similar convergence behaviors for both SR and MAE. Next, we compare BRDAD with distance-based, forest-based, and kernel-based methods on anomaly detection benchmarks, highlighting its superior performance. Finally, we perform a parameter analysis on the number of bagging rounds $B$, showing that selecting an appropriate $B$ based on the sample size leads to improved performance.

The remainder of this paper is organized as follows.
In Section \ref{sec::methodology}, we introduce some preliminaries related to anomaly detection and propose our BRDAD algorithm. 
We provide basic assumptions and theoretical results on the convergence rates of BRDDE and BRDAD in Section \ref{sec::TheoreticalResults}. 
Some comments and discussions concerning the theoretical results will also be provided in this section.
We present the error and complexity analysis of our algorithm in Section \ref{sec::ErrorAnalysis}.
Some comments concerning the time complexity will also be provided in this section.
We verify the theoretical findings of our algorithm by conducting numerical experiments in Section \ref{sec::Experiments}.
We also conduct numerical experiments to compare our algorithm with other state-of-the-art algorithms for anomaly detection on real-world datasets in this Section.
All the proofs of Sections \ref{sec::methodology}, \ref{sec::TheoreticalResults}, and \ref{sec::ErrorAnalysis} can be found in Section \ref{sec::proofs}.
We conclude this paper in Section \ref{sec::Conclusion}.

\section{Methodology}\label{sec::methodology}
 
We present our methodology in this section. 
Section \ref{sec::Preliminaries} introduces basic notations and concepts. 
In Section \ref{sec::weightdkad}, we propose the \textit{bagged weighted $k$-distances for density estimation} (\textit{BWDDE}) to demonstrate how bagged weighted $k$-distances can be applied to anomaly detection.
Section \ref{sec::datadrivennnad} reformulates the weight selection problem for density estimation as a surrogate risk minimization problem, aiming to minimize the finite-sample bound of the empirical risk of BWDDE.
Finally, the weights obtained by solving the SRM problem are utilized to construct our main algorithm, named \textit{bagged regularized $k$-distances for anomaly detection} (\textit{BRDAD}).

\subsection{Preliminaries}\label{sec::Preliminaries}

We begin by introducing some fundamental notations that will frequently appear. 
Suppose that the data $D_n := \{ X_1, \ldots, X_n \}$ are independent and identically distributed ($\mathrm{i.i.d.}$) and drawn from an unknown distribution $\mathrm{P}$ that is absolutely continuous with respect to the Lebesgue measure $\mu$ and admits a unique density function $f$.
In this paper, we assume that  $f$ is supported on $[0,1]^d$, which we denote as $\mathcal{X}$.
Recall that for $1 \leq p < \infty$ and a vector $x\in \mathbb{R}^d$, the $\ell_p$-norm is defined as
$\|x\|_p := ( x_1^p + \cdots + x_d^p )^{1/p}$,
and the $\ell_{\infty}$-norm is defined as
$\|x\|_{\infty} := \max_{i = 1, \ldots, d} |x_i|$.  
For a measurable function $g:\mathcal{X}\to\mathbb{R}$, we define the $L_p$-norm as $\|g\|_{p}:=\big(\int_{\mathcal{X}} |g(x)|^p \, dx\big)^{1/p}$. Let $B(x,r):=\{x'\in\mathbb{R}^d: \|x'-x\|_2\leq r\}$ denote a ball in Euclidean space $\mathbb{R}^d$ centered at $x\in \mathbb{R}^d$ with radius $r\in(0,+\infty)$.
We use $V_d$ to denote the volume of the $d$-dimensional closed unit ball.
In addition, for $n\in \mathbb{N}_+$, we write $[n]:=\{1,\ldots,n\}$ as the set containing integers from $1$ to $n$ and $\mathcal{W}_n:=\{(w_1,\ldots,w_n)\in \mathbb{R}^n:\sum_{i=1}^n w_i=1,\ w_i\geq 0,\ i\in [n] \}$.  
For any $x\in \mathbb{R}$, let $\lfloor x \rfloor$ be the largest integer less than or
equal to $x$ and $\lceil x \rceil $ be the smallest ingeter larger than or equal to $x$.

Throughout this paper,  we use $a\vee b=\max\{a,b\}$ and $a\wedge b=\min\{a,b\}$. 
Moreover, we use the following notations to compare the magnitudes of quantities: $a_n \lesssim b_n$ or $a_n=\mathcal{O}(b_n)$ indicates that there exists a positive constant $c>0$ that is independent of $n$ such that $a_n \leq c b_n$; $a_n \gtrsim b_n$ implies that there exists a positive constant $c >0$ such that $a_n \geq c b_n$; and $a_n \asymp b_n$ means that $a_n\lesssim b_n$ and $b_n\lesssim a_n$ hold simultaneously.
In this paper, we focus on the case of a \textit{fixed} dimension, allowing the constant $c$ to depend on the dimension $d$.
Finally, we use $C$, $C'$, $c$, and $c'$ to represent positive constants.

\subsection{Bagged Weighted $k$-Distances for Anomaly Detection} \label{sec::weightdkad}

The learning goal of anomaly detection is to identify observations that deviate significantly from the majority of the data. Anomalies are typically rare and different from the expected behavior of the data set. 
In this paper, we adopt the contamination model proposed by \citet{huber1965robust} and further discussed in \citet[Section 1]{huber1992robust}:
\begin{assumption}\label{asp::huber}
The data $D_n$ consists of $\mathrm{i.i.d.}$ samples drawn from a distribution $\mathrm{P}$ satisfying the contamination model: 
\begin{align}\label{equ::hubermodel}
\mathrm{P} = (1-\Pi) \cdot \mathrm{P}_0 + \Pi \cdot \mathrm{P}_1,
\end{align}
where $\mathrm{P}_0$ and $\mathrm{P}_1$ denote the distributions of normal and anomalies, respectively, and $\Pi \in (0,1)$ is the contamination proportion.
Additionally, we assume that $\mathrm{P}_0$ has a probability density function $f_0$, while $\mathrm{P}_1$ is uniformly distributed over $[0,1]^d$ with density function $ f_1$.
\end{assumption}

Assumption \ref{asp::huber} describes a contamination model where the observed data is drawn from a mixture of normal and anomalous distributions. 
Assuming $\mathrm{P}_1$ to be uniform on $[0,1]^d$ is a common choice, reflecting an uninformative prior on anomalies \citep{steinwart2005classification}.
Within this framework, anomaly detection reduces to identifying low-density regions, as the density of normal data shares the same family of level sets as the density of the contaminated model.
This classical assumption facilitates theoretical analysis and underpins several well-known unsupervised methods, including OC-SVM \citep{scholkopf2001estimating} and deep learning-based approaches \citep{ruff2021unifying}.

Building on this framework, we explore the distance-based method for unsupervised anomaly detection, motivated by the connection between distance functions and $k$-nearest neighbor (kNN) density estimations highlighted by \citet{biau2011weighted}.
This method leverages the distances between data points and their nearest neighbors to assess anomalies.
For any $x \in \mathbb{R}^d$ and a dataset $D_n$, we denote $X_{(k)}(x;D_n)$ as the $k$-th nearest neighbor of $x$ in $D_n$. We then define $R_{n,(k)}(x):=\bigl\|x-X_{(k)}(x;D_n)\bigr\|_2$ as the distance between $x$ and $X_{(k)}(x;D_n)$, referred to as the \textit{$k$-nearest neighbor distance}, or \textit{$k$-distance} of $x$ in $D_n$.

Distance-based methods are important and effective for anomaly detection. However, when handling large datasets, the number of nearest neighbors that need to be searched grows substantially, leading to significant computational overhead.
To mitigate this issue, we incorporate the bagging technique by averaging the weighted 
$k$-distances computed on multiple disjoint sub-datasets randomly drawn from the original dataset $D_n$ without replacement. 
Let $B$ be the number of bagging rounds pre-specified by the user, and $\{D_s^b\}_{b=1}^B$ be $B$ disjoint subsets of $D_n$, each of size $s$.
Since the sub-samples are disjoint and the data $D_n$ is supposed $\mathrm{i.i.d.}$,
this procedure is mathematically equivalent to taking the first $s$ samples for bag $1$, the following $s$ samples for bag $2$, etc.
In each subset $D^b_s$, $b \in [B]$, let $R_{s,(k)}^b (x):=\bigl\|x-X_{(k)}(x;D_s^b)\bigr\|_2$ be the $k$-distance of $x$ in $D_s^b$ for any integer $k\leq s$, and let the \textit{weighted $k$-distance} be defined as $R^{w,b}_s(x):=\sum_{i=1}^s w_i^b R_{s,(i)}^b (x)$, with $w^b\in \mathcal{W}_s$. The \textit{bagged weighted 
$k$-distances} are obtained by averaging these weighted $k$-distances across the $B$ sub-datasets:
\begin{align*}
R_n^B(x):= \frac{1}{B} \sum_{b=1}^B R^{w,b}_s(x).
\end{align*}
Then, we introduce the \textit{bagged weighted $k$-distances for density estimation} (\textit{BWDDE}) as
\begin{align}\label{equ::flambdabx}  
f^B_n(x):=\frac{1}{V_d {R}_n^B (x)^d}
\biggl( \frac{1}{B} \sum_{b=1}^B \sum_{i=1}^s w_i^b \gamma_{s,i} \biggr)^d,
\end{align}
where $V_d = \pi^{d/2} / \Gamma(d/2+1)$ is the volume of the unit ball and
\begin{align}\label{equ::gammai}
\gamma_{s,i}:=\frac{\Gamma(i+1/d)\Gamma(s+1)}{\Gamma(i)\Gamma(s+1+1/d)}.
\end{align}
Here, $\gamma_{s,i}$ represents the expected value of the $d$-th root of the beta distribution $\mathrm{Beta}(i,s+1-i)$, which is associated with the probability of a ball having a radius of $R_{s,(i)}^b(x)$.
The choice of $\gamma_{s,i}$ facilitates the derivation of the concentration inequality for BWDDE. A detailed discussion is provided in Section \ref{sec::errorBWDDE}.
Potential anomalies can be identified in regions of low density using BWDDE. More specifically, the dataset $D_n = \{ X_1, \ldots, X_n \}$ can be sorted in ascending order based on their BWDDE values, denoted as $ \{ X_1', \ldots, X_n' \} $, such that  
$f_n^B(X_1') \leq \cdots \leq f_n^B(X_n')$.  
If the number of anomalies is predetermined as $m$, then the $m$ data points with the smallest BWDDE values are identified as anomalies.

\subsection{Bagged Regularized $k$-Distances for Anomaly Detection} \label{sec::datadrivennnad}

A key challenge in using bagged weighted $k$-distance for density estimation \eqref{equ::flambdabx} is selecting appropriate weights for the nearest neighbors. These weights play a crucial role in determining the accuracy of the density estimate and, consequently, the precision of anomaly detection.
The simplest way is to take $B = 1$ and, for a fixed in advance $k$, set $w_k = 1$ and $w_i = 0$ for $i \in [n] \setminus \{k\}$.
In this case, BWDDE reverts to the standard $k$-NN density estimation \citep{moore1977large, devroye1977strong, dasgupta2014optimal}. 
Notably, the standard $k$-NN density estimation relies solely on the distance to the $k$-th nearest neighbor, disregarding information from other neighbors. To address this limitation, a more general approach was proposed by \cite{biau2011weighted}, which investigated the general weighted $k$-nearest neighbor density estimation by associating the weights with a given probability measure on $[0,1]$.
More specifically, for a given probability measure $\nu$ on $[0,1]$ and a sequence of positive integer $\{k_n\}$, the weights are defined as $w_i = \int_{((i-1)/k_n,i/k_n]}\nu(dt)$, for $1\leq i\leq k_n$, with $w_i=0$ otherwise.
\paragraph{Challenges for the weight selection.}
Since density estimation is an unsupervised problem, we lack access to the true density function for direct hyperparameter selection. Existing literature has proposed two common approaches to address this challenge:

\vspace{+2mm}
\noindent
1. \textbf{ANLL-based Approach}: A common approach is to optimize hyperparameters by minimizing the Average Negative Log Likelihood (ANLL) \citep{chow1983consistent,lopez2013histogram,silverman2018density}, which is equivalent to minimizing KL divergence. This approach assumes that the density estimate integrates to one over $\mathbb{R}^d$. However, this assumption does not hold for our BWDDE.
Given the dataset $D_n$, let $M:=\max_{i\in [n]}\|X_i\|_2$.
By the definition of the bagged weighted $k$-distances, for any $x\in \mathbb{R}^d$, we have:
\begin{align*}
R_n^B(x) = \frac{1}{B} \sum_{b=1}^B R_s^{w,b}(x) \leq \frac{1}{B} \sum_{b=1}^B R_{s,(s)}^b(x) \leq \frac{1}{B} \sum_{b=1}^B (\|x\|_2 + M) = \|x\|_2 + M.
\end{align*}
Therefore, the density estimate satisfies
\begin{align*}
f_n^B(x) \geq \frac{1}{V_d \big(\|x\|_2 + M\big)^d} \Bigg( \frac{1}{B} \sum_{b=1}^B \sum_{i=1}^s w_i^b \gamma_{s,i}\Bigg)^d
\end{align*}
for any $x\in \mathbb{R}^d$.
Notably, the integral of $f_n^B(x)$ over $\mathbb{R}^d$ diverges, rendering the ANLL approach unsuitable. 

\vspace{+2mm}
\noindent
2. \textbf{$L_2$ Risk Approach}: An alternative approach is leave-one-out cross-validation based on $L_2$ risk of the density estimate \citep{tsybakov2008introduction,biau2011weighted}, defined as
\begin{align*}
\int_{\mathbb{R}^d} \left(f_n^B(x) - f(x)\right)^2 \, dx = \int_{\mathbb{R}^d} f_n^B(x)^2 \, dx - 2 \int_{\mathbb{R}^d} f_n^B(x) f(x) \, dx + \int_{\mathbb{R}^d} f(x)^2 \, dx.
\end{align*}
    
The last term is independent of \( f_n^B \) and can be ignored in the optimization. The second term can be estimated by $-2\sum_{i=1}^n f_{n,(-i)}^B(X_i)/n$, where $f_{n,(-i)}^B(x)$ is the density estimator with the $i$-th observation removed. Thus, we define the cross-validation loss as
\begin{align*}
L_{\mathrm{CV}}(f_n^B) := \int_{\mathbb{R}^d} f_n^B(x)^2 \, dx - \frac{2}{n} \sum_{i=1}^n f_{n,(-i)}^B(X_i).
\end{align*}
However, computing $\int_{\mathbb{R}^d} f_n^B(x)^2 \, dx$ requires Monte Carlo methods, which becomes infeasible in high dimensions due to the absence of a closed-form expression for this integral.

In summary, existing hyperparameter selection methods are not well-suited for BWDDE in high-dimensional cases due to the lack of closed-form integrals. To address this, we propose the Surrogate Risk Minimization (SRM) approach, which allows automatic weight selection through a more computationally feasible optimization, distinguishing our method from other nearest-neighbor-based density estimators.

\paragraph{Surrogate risk.}
In the context of density estimation, 
we consider the absolute loss function
$L : \mathcal{X} \times \mathbb{R} \to [0, \infty)$,
defined as $L(x, t) := |f(x) - t|$, to measure the discrepancy between an estimate and the true density function $f$.
Let $D_s^B = \cup_{b=1}^B D_s^b$ denote the union of the $B$ sub-datasets, where each $D_s^b = \{X_1^b, \ldots, X_s^b\}$.
The \textit{empirical risk} of $f_n^B$ with respect to $D_s^B$ is given by
\begin{align}\label{RiskAbsoluteHypo}
\mathcal{R}_{L,D_s^B}(f_n^B)
:= \frac{1}{Bs}\sum_{b=1}^B \sum_{i=1}^s\bigl| f_n^B(X_i^b) - f(X_i^b) \bigr|.
\end{align}
As noted in \cite{devroye2001combinatorial,hang2018kernel}, the absolute loss is a reasonable choice for density estimation due to its invariance under monotone transformations. Moreover, it is proportional to the total variation metric, providing a more interpretable measure of proximity to the true density.

Since the underlying density function $f$ in \eqref{RiskAbsoluteHypo} is unknown, standard optimization techniques for parameter selection cannot be directly applied to weight selection in density estimation. To address this, we seek a \textit{surrogate} for the empirical risk in \eqref{RiskAbsoluteHypo} and minimize it to determine the nearest neighbor weights. To proceed, we introduce the following regularity assumptions on the underlying probability distribution $\mathrm{P}$.

\begin{assumption}\label{asp::holder}

Assume that $\mathrm{P}$ has a Lebesgue density $f$ with support $\mathcal{X}= [0,1]^d$. 
\begin{enumerate}
\item[(i)][\textbf{Lipschitz Continuity}] The density $f$ is Lipschitz continuous on $[0,1]^d$, i.e., for all $x, y \in [0,1]^d$, there exists a constant $c_L > 0$ such that $|f(x) - f(y)| \leq c_L \|x - y\|_2$.
\item[(ii)][\textbf{Boundness}] There exist constants $\overline{c} \geq \underline{c} > 0$ such that $\underline{c}\leq f(x) \leq \overline{c}$ for all $x \in \mathcal{X}$.
\end{enumerate}

\end{assumption}

The smoothness assumption is necessary for bounding the variation of the density function and is a common approach in density estimation \citep{dasgupta2014optimal,jiang2017uniform}. This assumption helps prevent overfitting and provides a more stable density estimate.
The assumption of lower boundedness of the density has been commonly employed in prior work, such as \citep{dasgupta2014optimal, zhao2022analysis}, to establish finite-sample rates for $k$-NN density estimation. This assumption simplifies deriving finite-sample bounds for $k$-distances (see Lemma 12), thereby providing a clearer characterization of the surrogate risk in the following Proposition. We emphasize that although this assumption aids theoretical analysis, our proposed algorithm remains applicable in broader settings. Further discussions can be found after Theorem \ref{thm::bagaverage}.

Under these assumptions, along with additional conditions on the weights, the next proposition presents a surrogate for the empirical risk  \eqref{RiskAbsoluteHypo}.

\begin{proposition}[Surrogate Risk] \label{pro::I456sur}
Let Assumption \ref{asp::holder} hold, and let $L$ denote the absolute value loss. 
Let $\{D_s^b\}_{b=1}^B$ be $B$ disjoint subsets randomly drawn from $D_n$, with $D_s^b = \{X_1^b, \ldots, X_s^b\}$, and define $\overline{R}_{s,(i)}^b :=  \sum_{j=1}^s R_{s,(i)}^b(X_j^b)/s$ as the   average $i$-distances for any integer $i \leq s$ on the subset $D_s^b$.
Furthermore, let $f$ be the true density function and $f_n^B$ be the BWDDE as in \eqref{equ::flambdabx}. 
Moreover, let $k^b :=k(w^b ):=\sup\{i\in [s]: w^b_i\neq 0\}$,
$\underline{k}:=\min_{b \in [B]} k^b $,
and $\overline{k}:=\max_{b \in [B]} k^b $.
Finally, suppose that the following four conditions hold:
\begin{enumerate}
\item[(i)] There exists a sequence $c_n\asymp \log n$ such that
$\displaystyle \sum_{i=1}^{c_n} w_i^b \lesssim (\log n)/k^b$ for all $b \in [B]$;
\item[(ii)] $\underline{k}\gtrsim (\log n)^2$, $\underline{k} \asymp \overline{k}$, $B \geq 2(d^2+4)(\log n)/3$, $B \lesssim (\overline{k}/(\log n))^{1+2/d}$, $\log s\asymp \log n$, and $s\geq  \max\{c_1',2\overline{k}\}$, where $c_1'$ is a constant defined in Lemma \ref{lem::Rrho};
\item[(iii)] $\|w^b\|_2 \gtrsim \bigl( k^b \bigr)^{-1/2}$ and $\displaystyle \sum_{i=1}^s i^{1/d} w_i^b \asymp \bigl( k^b \bigr)^{1/d}$ for $b \in [B]$;
\item[(iv)]
There exist constants $C_{n,i}$ such that
$\displaystyle \max_{b \in [B]} w^b_i \lesssim C_{n,i}$ for $i\in [s]$ and $\displaystyle \sum_{i=1}^s i^{1/d-1/2} C_{n,i} \lesssim \overline{k}^{1/d-1/2}$.
\end{enumerate}
Then, there exists $N_1^*\in \mathbb{N}$, specified in the proof, such that for all $n\geq N_1^*$ and $X_i^b$ satisfying $B(X_i^b,R_{s,(k^{b'})}^{b'}(X_i^b))\subset [0,1]^d$ for all $b'\in [B]$, there holds
\begin{align}
\label{equ:bagupperbound}
L(X_i^b, f_n^B)
\lesssim  \sqrt{(\log s)/B} \cdot \|w^b\|_2 + R^{w,b}_s(X_i^b), 
\qquad 
i \in [s],
\, 
b \in [B],
\end{align}
with probability $\mathrm{P}^{Bs}$ at least $1-4/n^2$.
Furthermore, we have
\begin{align}\label{equ::bagerub}
\mathcal{R}_{L,D_s^B}(f_n^B)
\lesssim \mathcal{R}_{L,D_s^B}^{\mathrm{sur}}(f_n^B)&:=
\frac{1}{B}\sum_{b=1}^B\biggl(\sqrt{(\log s)/B} \cdot \|w^b\|_2 + \frac{1}{s}\sum_{i=1}^s R^{w,b}_s(X_i^b)\biggr)\nonumber\\
&=\frac{1}{B}\sum_{b=1}^B\biggl(\sqrt{(\log s)/B} \cdot \|w^b\|_2 + \sum_{i=1}^s w_i^b  \overline{R}_{s,(i)}^b\biggr).
\end{align}
\end{proposition}
The term on the right-hand side of \eqref{equ::bagerub} is referred to as the \textit{surrogate risk}. A smaller surrogate risk clearly corresponds to higher accuracy in BWDDE.
Condition \textit{(i)} and \textit{(ii)} imply that $\sum_{i=1}^{c_n} w_i^b \to 0$ as $n \to \infty$ for all $b \in [B]$. This ensures that the weights are not overly concentrated in the first $c_n$ nearest neighbors, promoting a more balanced distribution of weights across the data points.
The first requirement in \textit{(ii)} ensures that the number of nearest neighbors at each bagging round is at least of the order $(\log n)^2$, which aligns with the condition $k\to \infty$ in \cite{moore1977large, dasgupta2014optimal}. The second requirement mandates that the number of non-zero nearest neighbors across different subsets remains of the same order. The third and fourth requirements impose lower and upper bounds on the number of bagging rounds $B$. Since the subsets are drawn without replacement, a very large $B$ can result in a small sample size in each subset, which would increase the estimation error. Conversely, if $B$ is too small, the density estimator may not benefit sufficiently from bagging. The last two conditions in \textit{(ii)} impose bounds on $s$ for similar reasons.
Condition \textit{(iii)} on the relationship between $w^b$ and $k^b$ is satisfied for commonly used weight choices for nearest neighbors.
Finally, condition \textit{(iv)} requires that the moments of the weights be bounded by powers of $\overline{k}$.
The condition $B(X_i^b, R_{s,(k^{b'})}^{b'}(X_i^b)) \subset [0,1]^d$ for all $b'\in [B]$ ensures that the $k^{b'}$-distance ball is fully contained within the cube $[0,1]^d$ for $b'\in [B]$. This is crucial because if the ball extends beyond the cube, the density estimator may not be consistent under Assumption \ref{asp::holder}.

While the conditions in Proposition 1 may appear complex, they encompass commonly used weight choices. For example, under a uniform weight distribution for the nearest neighbors, we have
$w_i^b = \eins\{i\leq k\}/k$ for $i\in [s]$ and $b\in [B]$, where $k\in [s]$ is fixed.
Under this setting, condition $(i)$ is directly satisfied, and condition $(ii)$ holds with appropriately chosen parameters.
Regarding condition $(iii)$, we obtain $
\|w^b\|_2 = \bigl( k^b \bigr)^{-1/2} = 1/\sqrt{k}$ and $\sum_{i=1}^s i^{1/d} w_i^b = \sum_{i=1}^k i^{1/d}/k \asymp k^{1/d}$ for all $b\in [B]$.
If we set $C_{n,i}=\eins\{i\leq k\}/k$ for $i\in [s]$, it follows that $
\sum_{i=1}^s i^{1/d-1/2}C_{n,i} =\sum_{i=1}^k i^{1/d-1/2}/k \asymp k^{1/d-1/2}$,
implying that all conditions hold for bagged uniformly weighted $k$-distance density estimation.
With minor modifications, these arguments extend to the more general case $w_i = \int_{((i-1)/k,i/k]}\nu(dt)$, for $1\leq i\leq k$, with $w_i=0$ otherwise, where $\nu$ is the probability measure associated with $\mathrm{Beta}(\alpha,1)$ for $\alpha\geq 1$.

These conditions are primarily used to derive our surrogate risk, which leads to a new and effective algorithm without hyperparameter tuning. In fact, Proposition \ref{lem::order:k:w:lambda} in Section \ref{sec::analysissr} ensures that our proposed method satisfies these conditions with high probability, making it unnecessary to verify them in practical applications.

\paragraph{Surrogate risk minimization (SRM).} 
From the expression of the surrogate risk in \eqref{equ::bagerub}, minimizing the surrogate risk is equivalent to solving the following optimization problems:
\begin{align}\label{equ::bagwstar}
w^{b,*}:=\argmin_{w^b\in \mathcal{W}_s} \; \sqrt{(\log s)/B} \cdot \|w^b\|_2 + \sum_{i=1}^s w_i^b \overline{R}^b_{s,(i)}, 
\qquad 
b \in [B].
\end{align}

A closer examination of the optimization problems in \eqref{equ::bagwstar} reveals that each consists of two components. The first term is proportional to the $\ell_2$-norm of the weights $w^b$, while the second term represents a linear combination of these weights.

Without the first term, the optimization objective in \eqref{equ::bagwstar} reduces to the second term, $\sum_{i=1}^s w_i^b \overline{R}_{s,(i)}^b$, which attains its minimum when $w^b=(1,0,\ldots,0)$. In this case, the weighted $k$-distance simplifies to $R_s^{w,b}(x) = R_{s,(1)}^b(x)$, representing the distance from $x$ to its nearest neighbor. This often leads to overfitting in density estimation, as it fails to incorporate information from other nearest neighbors.

By introducing the $\|w^b\|_2$ term into the minimization problem \eqref{equ::bagwstar}, we mitigate the overfitting issue. 
The $\ell_2$-norm $\|w^b\|_2$ attains its maximum value of $1$ when all weight is assigned to a single nearest neighbor, i.e., when $w_i^b = 1$ for some $i \in [s]$ and $w_j^b = 0$ for all $j \neq i$. Conversely, it reaches its minimum value of $n^{-1/2}$ when the weights are uniformly distributed as $w^b = (1/n, \ldots, 1/n)$. Consequently, incorporating this term encourages the weights to be distributed across multiple nearest neighbors, thereby preventing overfitting. As a result, $\|w^b\|_2$ serves as a \textit{regularization} term in the minimization problem \eqref{equ::bagwstar}.

\paragraph{Solution to SRM.}
Notice that \eqref{equ::bagwstar} is a convex optimization problem solved efficiently from the data. 
For a fixed $b \in [B]$, considering the constraint Lagrangian, we have
\begin{align*}
\mathcal{L}(w^b,\mu^b,\nu^b):=\sqrt{(\log s)/B}\cdot \|w^b\|_2 + \sum_{i=1}^s w_i^b\overline{R}_{s,(i)}^b +\mu^b\biggl(1-\sum_{i=1}^s w_i^b\biggr)-\sum_{i=1}^s \nu_i^b w_i^b,
\end{align*}
where $\mu^b\in \mathbb{R}$ and $\nu_1^b,\ldots\nu_s^b\geq 0$ are the Lagrange multipliers.
Since \eqref{equ::bagwstar} is a convex optimization problem, the solution satisfying the KKT conditions is a global minimum.
Setting the partial derivative
of $\mathcal{L}(w^b,\mu^b,\nu^b)$ with respect to $w_i^b$ to zero gives:
\begin{align}\label{equ::logswi}
\sqrt{(\log s)/B}\cdot
{w_i^b}/{\|w^b\|_2}=\mu^b+\nu_i^b-\overline{R}_{s,(i)}^b.
\end{align}
Since $w^{b,*}$ is the optimal solution of \eqref{equ::bagwstar}, the KKT conditions imply that if $w_i^{b,*} > 0$, then $\nu_i^b = 0$. Conversely, if $w_i^{b,*} = 0$, then $\nu_i^b \geq 0$, which implies $\overline{R}_{s,(i)}^b\geq \mu^b$.
Therefore, $w_i^{b,*}$ is proportional to $\mu^b-\overline{R}_{s,(i)}^b$ for all nonzero entries.
This together with the equality constraint $\sum_{i=1}^s w_i^{b,*} = 1$ yields that $w_i^{b,*}$ has the form
\begin{align*}
w_i^{b,*} = \frac{\bigl(\mu^b-\overline{R}_{s,(i)}^b\bigr)\cdot \eins\{\overline{R}_{s,(i)}^b< \mu^b\}}{ \sum_{i=1}^s \bigl(\bigl(\mu^b-\overline{R}_{s,(i)}^b\bigr)\cdot \eins\{\overline{R}_{s,(i)}^b< \mu^b\}\bigr)}
\end{align*}
for $i\in [s]$.
Since $\overline{R}_{s,(i)}^b$ becomes larger as $i$ increases, the formulation above shows that $w_i^{b,*}$ becomes smaller as $i$ increases.
Moreover, the optimal weights have a cut-off effect that only nearest neighbors near $x$, i.e.~$\overline{R}_{s,(i)}^b< \mu^b$ are considered in the solution, while the weights for the remaining nearest neighbors are all set to zero.
This is consistent with our usual judgment: the closer the neighbor, the greater the impact on the density estimation.

There are many efficient methods to solve the convex optimization problem \eqref{equ::bagwstar}. Here, we follow the method developed in \cite{anava2016k,dong2020lookcom,sheng2023tnn}.
The key idea is to add nearest neighbors in a greedy manner based on their distance from $x$ until a stopping criterion is met. 
We present it in Algorithm \ref{alg::SRM}.

\begin{algorithm}[h]
\caption{Surrogate Risk Minimization (SRM)}
\label{alg::SRM}
\KwIn{Average $i$-distances $\overline{R}_{s,(i)}^b$, $1\leq i\leq s$.
} 
Let $r_i=\sqrt{B/\log s}\cdot \overline{R}_{s,(i)}^b $, $1\leq i\leq s$.\\
Set $\mu_0=r_1+1$ and $k=0$.\\
\While{$\mu_k>r_{k+1}$ and $k\leq s-1$}{
$k\gets k+1$,\\
$\mu_k=\left(\sum_{j=1}^k r_j +\sqrt{k+ (\sum_{j=1}^kr_j)^2-                         k\sum_{j=1}^kr_j^2}\right)/k$.
}
Compute $A=\sum_{i=1}^{s}((\mu_{k}-r_i)\cdot \mathbf{1}(r_i<\mu_{k}))$.\\
Compute $w_i^{b,*} = {(\mu_{k}-r_i)\cdot \mathbf{1}(r_i<\mu_{k})}/A$, $1\leq i\leq s$.\\
\KwOut{Weights $w^{b,*}$.}
\end{algorithm}

\paragraph{Density estimation.}
The discussions above indicate that the minimization problem \eqref{equ::bagwstar} provides a practical method for determining the weights of nearest neighbors in density estimation.

However, these weights depend on the data, introducing statistical dependence between the weights
$w_i^{b,*}$ and the $i$-distance $R_{s,(i)}^b(x)$ for $i\in [s]$ in each bag $b \in [B]$.
Such dependence complicates the theoretical analysis of weighted $k$-distances and density estimation within the framework of statistical learning theory.

To address this issue, we employ distinct subsets of the data for computing the $k$-distances, ensuring their independence from the optimized weights. For simplicity, we assume that $n = 2Bs$, which facilitates a more straightforward theoretical derivation without loss of generality. This assumption ensures that each subset is of equal size; if $n$ is not exactly divisible by $2B$, minor modifications to the partitioning scheme can be applied without affecting the theoretical conclusions.
We randomly partition $D_n$ into $B$ disjoint pairs of subsets $\{(D_s^b, \widetilde{D}_s^b)\}_{b=1}^B$, where each subset has size $s$ and $D_s^b\cap \widetilde{D}_s^b=\emptyset$ for each $b$. Let $w^{b,*}$ represent the weights obtained by solving the optimization problem \eqref{equ::bagwstar} using $\{D_s^b\}_{b=1}^B$, and let $\widetilde{R}_{s,(i)}^b(x)$ denote the $i$-distance of $x$ computed using $\{\widetilde{D}_s^b\}_{b=1}^B$ for $i\in [s]$.
Then, we define the \textit{regularized $k$-distances} as follows:
\begin{align}\label{equ::rbstars}
R^{b,*}_s(x) := R^{w^{b,*},b}_s(x) = \sum_{i=1}^s w^{b,*}_i \widetilde{R}_{s,(i)}^b(x).
\end{align}
We refer to the weighted average of these $k$-distances as the \textit{bagged regularized $k$-distance}
\begin{align}\label{equ::rlambdaBx}
R^{B,*}_n(x) := \frac{1}{B} \sum_{b=1}^B R^{b,*}_s(x).
\end{align}
By incorporating the $w^{b,*}$ and $R_n^{B,*}(x)$ into the BWDDE formula \eqref{equ::flambdabx}, 
we are able to obtain a new nearest-neighbor-based density estimator called \textit{bagged regularized $k$-distances for density estimation} (\textit{BRDDE}), expressed as 
\begin{align}\label{equ::fbstar}
f_n^{B,*}(x)
:= \frac{1}{V_d {R}_n^{B,*}(x)^d}
\biggl( \frac{1}{B} \sum_{b=1}^B \sum_{i=1}^s w_i^{b,*} \gamma_{s,i}\biggr)^d.
\end{align}
This minimization approach distinguishes our BRDDE from existing nearest-neighbor-based density estimators. 
Specifically, they suffer from the sensitivity to the choice of the hyperparameter $k$, since the selection of $k$ is inherently difficult due to the lack of supervised information. On the contrary, when the number of bagging rounds $B$ is fixed, SRM enables the calculation of the weights of nearest neighbors in each subset $D_s^b$ by solving the convex optimization problem based on the average $i$-distance $\overline{R}_{s,(i)}^b$ as in equation \eqref{equ::bagwstar}.
As a result, we successfully address the hyperparameter selection challenge without changing the unsupervised nature of the problem.

\paragraph{Anomaly Detection.}
By applying BRDDE to all samples, anomalies can be identified as instances with lower BRDDE values. However, explicit density estimation is not necessary for anomaly detection.
Since density estimates serve as anomaly scores, any monotone transformation preserves their ranking (possibly in reverse) and maintains the same family of level sets, leading to identical AUC values.
Thus, bagged regularized $k$-distances provide a sufficient and practical alternative for anomaly detection. They can be accurately computed in high-dimensional spaces, and their associated weights can be efficiently optimized from \eqref{equ::bagwstar}, making them an effective approach for density-based anomaly detection.

We now introduce our anomaly detection algorithm, \textit{bagged regularized $k$-distances for anomaly detection} (\textit{BRDAD}). 
The dataset $D_n$ is sorted into the sequence $\{X_1',\ldots,X_n'\}$ based on their bagged regularized $k$-distances in descending order, i.e.~$R_n^{B,*}(X_1') \geq \cdots \geq R_n^{B,*}(X_n')$.
Given the pre-specified number of anomalies $m$, the first $m$ instances $\{X_i'\}_{i=1}^m$, are considered as the $m$ anomalies. The complete procedure of our BRDAD algorithm is presented in Algorithm \ref{alg:BRDAD}.
As illustrated above, SRM mitigates the challenge of hyperparameter selection in density estimation. Consequently, BRDAD retains the advantages of BRDDE by using the same weights to address the sensitivity of hyperparameter selection in nearest-neighbor-based methods for unsupervised anomaly detection.

\begin{algorithm}[t]
\caption{Bagged Regularized $k$-Distances for Anomaly Detection (BRDAD)}
\label{alg:BRDAD}
\KwIn{
Data $D=\{X_1,\cdots,X_n\}$; Number of anomalies $m$; \\
\qquad\qquad Bagging rounds $B$; Subsampling size $s$.
} 
Randomly partition $D_n$ into $B$ disjoint pairs of subsets, denoted as $\{(D_s^b, \widetilde{D}_s^b)\}_{b=1}^B$, such that $D_s^b \cap \widetilde{D}_s^b = \emptyset$ for each $b$. \\
\For{$b \in [B]$}
{Compute weights $w^{b,*}$ using $D_s^b$ by \eqref{equ::bagwstar};\\
Compute the regularized $k$-distances $R^{b,*}_s(X_i)$ for $1\leq i\leq n$ using \eqref{equ::rbstars}, based on $w^{b,*}$ and $\widetilde{D}_s^b$.}
Compute the bagged regularized $k$-distances $R_n^{B,*}(X_i)$ by \eqref{equ::rlambdaBx} for $1\leq i\leq n$.\\
Sort the data $D_n = \{X_1, \ldots, X_n\}$ as $\{X_1', \ldots, X_n'\}$ in descending order according to their bagged regularized $k$-distances, i.e., $
R_n^{B,*}(X_1') \geq \cdots \geq R_n^{B,*}(X_n')$.\\
\KwOut{Anomalies $\{X_i'\}_{i=1}^m$.}
\end{algorithm}
 
Algorithm \ref{alg:BRDAD} is based on the assumption that $\mathrm{P}_1$ follows a uniform distribution. However, when prior knowledge about the density function of anomaly is available, anomalies can still be detected using the bagged regularized $k$-distances $R_n^{B,*}(x)$.  
Consider the Huber model in \eqref{equ::hubermodel}, where $\mathrm{P}, \mathrm{P}_0,$ and $\mathrm{P}_1$ have densities $f, f_0,$ and $f_1$, respectively. Defining $h := f_0 / f_1$ and $\rho := \Pi / (1 - \Pi)$, \citet[Corollary 3]{steinwart2005classification} establishes that instances in the set $\{X_i : h(X_i) \leq \rho\}$ can be recognized as anomalies with theoretical guarantees. This set can be equivalently written as $\{X_i : f(X_i) / f_1(X_i) \leq 2 \Pi\}$, based on the relationship between the densities.  
Since $R_n^{B,*}(x)^d$ is inversely proportional to the density estimate $f_n^{B,*}(x)$, anomalies can be identified by sorting the data according to:
\begin{align*}
f_1(X_1')^{1/d} R_n^{B,*}(X_1') \geq \cdots \geq f_1(X_n')^{1/d} R_n^{B,*}(X_n'),
\end{align*}
where the top $m$ instances $\{X_i'\}_{i=1}^m$ are considered anomalies.  
Furthermore, the theoretical results in Section \ref{sec::TheoreticalResults} can be extended to this setting with minor modifications, ensuring the generality of our approach.

\section{Theoretical Results} \label{sec::TheoreticalResults}

In this section, we present theoretical results related to our BRDAD algorithm. 
We first investigate the Huber contamination model in Section \ref{sec::hubermodel}, in which we can analyze the performance of the bagged regularized $k$-distances from a learning theory perspective.
Then, we present the convergence rates of BRDDE and BRDAD in Section \ref{sec::convergencebrdde}
and \ref{sec::ratebrdad}, respectively. Finally, we provide comments and discussions on our algorithms and theoretical results in Section \ref{sec::theoreticalcomments}.
We also compare our theoretical findings on the convergences of both BRDDE and BRDAD with other nearest-neighbor-based methods in this section.

\subsection{Huber Contamination Model}\label{sec::hubermodel}

In the Huber contamination model (HCM) in Assumption \ref{asp::huber}, for every instance $X$ from $\mathrm{P}$, we can use a latent variable $Y\in \{0,1\}$ that indicates which distribution it is from. 
More specifically, $Y=0$ and $Y=1$ indicate that the instance is from the normal and the anomalous distribution, respectively.
As a result, the anomaly detection problem can be converted into a bipartite ranking problem where instances are labeled positive or negative implicitly according to whether it is normal or not.
Let $\widetilde{\mathrm{P}}$ represent the joint probability distribution of $\mathcal{X}\times \mathcal{Y}$. 
In this case, our learning goal is to learn a score function that minimizes the probability of mis-ranking a pair of normal and anomalous instances, i.e.~that maximizes the area under the ROC curve (AUC).
Therefore, we can study regret bounds for the AUC of the bagged regularized $k$-distances to evaluate its performance from the learning theory perspective.
Let $r:\mathcal{X}\to \mathbb{R}$ be a score function, then the $\mathrm{AUC}$ of $r$ can be written as 
\begin{align*}
\mathrm{AUC}(r)=\mathbb{E}\bigl[\eins\{(Y-Y')(r(X)-r(X')>0)\}+\eins\{r(X)=r(X')\}/2|Y\neq Y'\bigr],
\end{align*}
where $(X,Y)$, $(X',Y')$ are assumed to be drawn $\mathrm{i.i.d.}$~from $\widetilde{\mathrm{P}}$.
In other words, the $\mathrm{AUC}$ of $r$ is the probability that a randomly drawn anomaly is ranked higher than a randomly drawn normal instance by the score function $r$. 
Given the HCM in \eqref{equ::hubermodel} and the assumption that $\mathrm{P}_1$ is uniformly distributed over $[0,1]^d$ in Assumption \ref{asp::huber}, the posterior probability function with respect to $\widetilde{\mathrm{P}}$ is given by  
\begin{align}\label{equ::etax}
\eta(x):=\widetilde{\mathrm{P}}(Y=1|X=x)=\frac{\Pi f_1(x)}{(1-\Pi)f_0(x)+\Pi f_1(x)} = \Pi f(x)^{-1}.
\end{align}
Then, the optimal $\mathrm{AUC}$ is defined as 
\begin{align*}
\mathrm{AUC}^*:=\sup_{r:\mathcal{X}\to \mathbb{R}} \mathrm{AUC}(r) = 1- \frac{1}{2\Pi(1-\Pi)} \mathbb{E}_{X,X'}\bigl[\min\bigl(\eta(X)(1-\eta(X')),\eta(X')(1-\eta(X)\bigr)\bigr].
\end{align*}
Finally, the $\mathrm{AUC}$ regret of a score function $r$ is defined as 
\begin{align*}
\mathrm{Reg}^{\mathrm{AUC}}(r):=\mathrm{AUC}^*-\mathrm{AUC}(r).
\end{align*}

As discussed in Section \ref{sec::weightdkad}, BRDAD is a density-based anomaly detection method. To establish its convergence rates under the Huber contamination model, we first derive the theoretical convergence rates of BRDDE in \eqref{equ::fbstar}, which are presented in the next subsection.

\subsection{Convergence Rates of BRDDE}\label{sec::convergencebrdde}

The convergence rates of BRDDE are presented in the following Theorem.

\begin{theorem}
\label{pro::rateBWDDE}
Let Assumption \ref{asp::holder} hold. 
Suppose that the dataset $D_n$ is randomly partitioned into $B$ disjoint pairs of subsets, denoted as $\{(D_s^b, \widetilde{D}_s^b)\}_{b=1}^B$, such that $D_s^b \cap \widetilde{D}_s^b = \emptyset$ for each $b$.   
Let $f$ be the true density function, and let $f_n^{B,*}$ be the BRDDE defined in \eqref{equ::fbstar}.  
If we choose
\begin{align}\label{equ::choicesb}
s\asymp(n/\log n)^{(d+1)/(d+2)} 
\quad
\text{and} 
\quad
B\asymp n^{1/(d+2)}(\log n)^{(d+1)/(d+2)},
\end{align}
then there exists $N_2^*\in \mathbb{N}$, which will be specified in the proof, such that for all $n>N_2^*$, with probability $\mathrm{P}^n$ at least $1-4/n^2$, we have
\begin{align*}
\int_{\mathcal{X}}|f_n^{B,*}(x) - f(x)| \, dx
\lesssim n^{-1/(2+d)}(\log n)^{(d+3)/(d+2)}.
\end{align*}
\end{theorem}

The convergence rate of the $L_1$-error of BRDDE in the above theorem matches the minimax lower bound established in \cite{zhao2020analysis} when the density function is Lipschitz continuous. 
Therefore, BRDDE attains the optimal convergence rates for density estimation.
As a result, the SRM procedure in Section \ref{sec::datadrivennnad} turns out to be a promising approach for determining the weights of nearest neighbors for BWDDE.

Moreover, notice that the number of iterations required in the optimization problem \eqref{equ::bagwstar} at each bagging round depends on the sub-sample size $s$.
In Theorem \ref{pro::rateBWDDE}, the choice of $s$ is significantly smaller than $n$, indicating that fewer iterations are required at each bagging round.
This explains the computational efficiency of incorporating the bagging technique if parallel computation is employed.
However, due to the dependence in $d$, this improvement becomes less and less significant in high dimension.
Further discussions on the complexity are presented in Section \ref{sec:complexity}.

\subsection{Convergence Rates of BRDAD}
\label{sec::ratebrdad}

The next theorem provides the convergence rates for BRDAD.

\begin{theorem}\label{thm::bagaverage}
Let Assumptions \ref{asp::huber} and  \ref{asp::holder} hold.  
Suppose the conditions in Theorem \ref{pro::rateBWDDE} hold, including the dataset partitioning and choice of parameters $s$ and $B$.
Let $R_n^{B,*}$ be the bagged regularized $k$-distances returned by Algorithm \ref{alg:BRDAD}.  
Then there exists $N_2^*\in \mathbb{N}$, as specified in Theorem \ref{pro::rateBWDDE}, such that for all $n>N_2^*$, with probability $\mathrm{P}^n$ at least $1 - 4/n^2$, we have  
\begin{align*}
\mathrm{Reg}^{\mathrm{AUC}}(R_n^{B,*})  \lesssim n^{-1/(2+d)}(\log n)^{(d+3)/(d+2)}.
\end{align*}
\end{theorem}

Theorem \ref{thm::bagaverage} establishes that, up to a logarithmic factor, the AUC regret of BRDAD converges at a rate of $\mathcal{O}(n^{-1/(d+2)})$, provided that the number of bagging rounds $B$ and the subsample size $s$ are chosen as in \eqref{equ::choicesb}.
Notably, the parameter choices and convergence rates in Theorem \ref{thm::bagaverage} align with those in Theorem \ref{pro::rateBWDDE} for BRDDE.
This follows from the fact that BRDAD is a density-based anomaly detection method built upon BRDDE.

Although surrogate risk minimization is formulated under Assumption \ref{asp::holder}, we note that these assumptions can be relaxed, and similar convergence rates for our BRDDE and BRDAD remain valid under more general conditions. In particular, Lipschitz continuity naturally extends to manifolds: if the data is supported on a $d'$-dimensional manifold $\mathcal{M}$ with density $f$ absolutely continuous with respect to the manifold's volume measure, then $f$ is Lipschitz continuous on $\mathcal{M}$ if there exists a constant $c_L>0$ such that $|f(x)-f(y)| \leq c_L d_{\mathcal{M}}(x,y)$ for all $x,y\in \mathcal{M}$, where $d_{\mathcal{M}}(x,y)$ denotes the geodesic distance. This assumption aligns with that in the prior literature \citep{berenfeld2021density}.
Additionally, the assumption of lower boundedness can be relaxed with suitable conditions on the tail behavior of the density, as discussed in \citet{zhao2022analysis}. Since geodesic distances on manifolds are well approximated by Euclidean distances within a small neighborhood \citep{niyogi2008finding, berenfeld2021density}, our theoretical results for SRM and the convergence rates of BRDDE and BRDAD can be extended to cases where both assumptions are relaxed.
However, our current focus remains on hyperparameter sensitivity in density estimation; therefore, a detailed exploration of these extensions is beyond the scope of this paper.

\subsection{Comments and Discussions}
\label{sec::theoreticalcomments}

By reformulating the analysis of bagged regularized $k$-distances in terms of BRDDE within a statistical learning framework \citep{vandervaart1996weak}, we establish convergence rates for the AUC regret of bagged regularized $k$-distances under the Huber contamination model, assuming mild regularity conditions on the density function (Theorem \ref{thm::bagaverage}). Notably, our findings reveal that the convergence rate of the AUC regret of BRDAD matches that for density estimation, indicating the effectiveness of BRDAD.

In contrast, previous theoretical studies on distance-based methods for unsupervised anomaly detection did not establish a connection between distance-based algorithms and density estimation, leaving the convergence rates unaddressed. For instance, \cite{sugiyama2013rapid} introduced a sampling-based outlier detection method and analyzed its effectiveness compared to traditional $k$-nearest neighbors approaches, but without a rigorous theoretical foundation. More recently, \cite{gu2019statistical} conducted a statistical analysis of distance-to-measure (DTM) for anomaly detection under the Huber contamination model, assuming specific regularity conditions on the distribution, and showed that anomalies can be identified with high probability. However, since these studies did not derive convergence rates for the AUC regret, their results are not directly comparable to ours.

\section{Error and Complexity Analysis} \label{sec::ErrorAnalysis}

In this section, we present the error analysis of the AUC regret and the complexity analysis of our algorithm. 
In detail, in Section \ref{sec::errorBWDDE}, we provide the error decomposition of the surrogate risk, which leads to the derivation of the surrogate risk in Proposition \ref{pro::I456sur} in Section \ref{sec::datadrivennnad}. 
Furthermore, in Section \ref{sec::LearningTheoryAnalysis}, we illustrate the three building blocks in learning the AUC regret,
which indicates the way to establish the convergence rates of both BRDDE and BRDAD in Theorem \ref{pro::rateBWDDE} and \ref{thm::bagaverage} in Section \ref{sec::ratebrdad}.
Finally, we analyze the time complexity of BRDAD and illustrate the computational efficiency of BRDAD compared to other distance-based methods for anomaly detection in Section \ref{sec:complexity}.

\subsection{Error Analysis for the Surrogate Risk}\label{sec::errorBWDDE}

In this section, we first provide the error decomposition for the BWDDE $f_n^B(x)$ in \eqref{equ::flambdabx}.
Then, we present the upper bounds for these error terms.

Let the term $(I)$ be defined as
\begin{align}\label{equ::i4}
(I) := \frac{1}{ V_d R_n^B (x)^d}
\sum_{j=0}^{d-1} \biggl( \frac{1}{B} \sum_{b=1}^B \sum_{i=1}^s w_i^b  \gamma_{s,i} \biggr)^j \bigl(V_d^{1/d}  f(x)^{1/d} R_n^B(x) \bigr)^{d-1-j}.
\end{align}
Using the triangle inequality and the equality 
\begin{align}\label{equ::dequality}
x^d - y^d
= (x - y) \cdot \sum_{i=0}^{d-1} x^i y^{d-1-i},
\end{align}
we obtain
\begin{align}\label{equ::flambdabdec}
\bigl| f_n^B(x) - f(x) \bigr|
& = \frac{1}{V_d R_n^B (x)^d} \cdot
\biggl| \biggl( \frac{1}{B} \sum_{b=1}^B \sum_{i=1}^s w_i^b  \gamma_{s,i} \biggr)^d - V_d f(x) R_n^B(x)^d \biggr|
\nonumber\\
& = (I) \cdot \biggl| \frac{1}{B} \sum_{b=1}^B \sum_{i=1}^s w_i^b  \gamma_{s,i} - V_d^{1/d} f(x)^{1/d} R_n^B(x) \biggr|
\nonumber\\
& =  (I) \cdot \biggl| \frac{1}{B} \sum_{b=1}^B \sum_{i=1}^s w_i^b  \gamma_{s,i} - \frac{1}{B}\sum_{b=1}^B \sum_{i=1}^s w_i^b V_d^{1/d} f(x)^{1/d} R_{s,(i)}^b(x)  \biggr|
\nonumber\\
& \leq (I) \cdot \sum_{i=1}^s \Biggl|\frac{1}{B}\sum_{b=1}^B w_i^b \bigl(\gamma_{s,i} - V_d^{1/d} f(x)^{1/d} R_{s,(i)}^b(x) \bigr) \Biggr|.
\end{align}
If the terms $(II)$ and $(III)$ are defined respectively as
\begin{align}
(II) 
& := \sum_{i=1}^s \Biggl|\frac{1}{B}\sum_{b=1}^B w_i^b  \big(\gamma_{s,i} - \mathrm{P}(B(x,R_{s,(i)}^b(x)))^{1/d}\big) \Biggr|, 
\label{equ::i5} 
\\
(III)
& := \sum_{i=1}^s \Biggl|\frac{1}{B}\sum_{b=1}^B  w_i^b \bigl( \mathrm{P}(B(x,R_{s,(i)}^b(x)))^{1/d} - V_d^{1/d} f(x)^{1/d} R_{s,(i)}^b(x) \bigr) \Biggr|,
\label{equ::i6}
\end{align}
then applying the triangle inequality to \eqref{equ::flambdabdec} yields
the error decomposition 
\begin{align}\label{equ::flambdabxerror}
\bigl| f^B_n(x) - f(x) \bigr|
\leq (I) \cdot (II) + (I) \cdot (III).
\end{align}

We emphasize that $\gamma_{s,i}=\mathbb{E}[\mathrm{P}(B(x,R_{s,(i)}^b(x)))^{1/d}]$ for a fixed $x\in \mathcal{X}$, $b\in [B]$, and $i\in [s]$.
Therefore, standard concentration inequalities can be applied to establish the convergence rates for term $(II)$.
The derivation of this equality is explained as follows:
Since $\mathrm{P}$ has a density over $[0,1]^d$ with respect to the Lebesgue measure, the random variable $\|X-x\|_2$ is continuous for a fixed $x\in \mathcal{X}$. 
By the probability integral transform, $\mathrm{P}(B(x,\|X-x\|_2))$ is uniformly distributed over $[0,1]$.
For any $b \in [B]$, note that $X_1^b,\ldots,X_s^b$ are $\mathrm{i.i.d.}$~with the same distribution $\mathrm{P}$.
Let $U_1^b, \ldots, U_s^b$ be $\mathrm{i.i.d.}$~uniform $[0,1]$ random variables. Then 
\begin{align*}
\bigl(\mathrm{P}(x,\|X_1^b-x\|_2),\ldots,\mathrm{P}(x,\|X_s^b-x\|_2)\bigr)\overset{\mathcal{D}}{=} \bigl(U_1^b,\ldots,U_s^b\bigr).
\end{align*}
Reordering the samples such that $\|X_{(1)}^b(x)-x\|_2\leq \cdots\leq \|X_{(s)}^b(x)-x\|_2$, we obtain
\begin{align}\label{equ::pbxribx}
\bigl(\mathrm{P}(B(x,R_{s,(1)}^b(x))),\ldots,\mathrm{P}(B(x,R_{s,(s)}^b(x)))\bigr)\overset{\mathcal{D}}{=}\big(U_{(1)}^b,\ldots,U_{(s)}^b\big),
\end{align}
where $U_{(i)}^b$ is the $i$-th order statistic of $U_1^b,\ldots,U_s^b$. 
This reduces the study of $\mathrm{P}(B(x,R_{s,(i)}^b(x)))$ to the study of $U_{(i)}^b$. 
By Corollary 1.2 in \cite{biaulecture}, $U_{(i)}^b\sim \mathrm{Beta}(i,s+1-i)$.
Hence, $\mathbb{E}[\mathrm{P}(B(x,R_{s,(i)}^b(x)))^{1/d}] = \mathbb{E}[(U_{(i)}^b)^{1/d}] = \gamma_{s,i}$. 

In the literature on weighted nearest-neighbor density estimation \citep{biau2011weighted, biaulecture}, the expression $(i/s)^{1/d}$ is often used in place of $\gamma_{s,i}$, introducing an additional error term $|\gamma_{s,i} - (i/s)^{1/d}|$. While this error is negligible in the absence of bagging, it slows the convergence rate of term $(II)$ when bagging is employed. Therefore, using $\gamma_{s,i}$ enables a clearer demonstration of the benefits of bagging in density estimation.

The following proposition provides the upper bounds for the error terms $(I)$, $(II)$, and $(III)$, respectively.

\begin{proposition}\label{pro::I456}
Let Assumption \ref{asp::holder} hold. 
Furthermore, let $(I)$, $(II)$, and $(III)$ be defined as in \eqref{equ::i4}, \eqref{equ::i5}, and \eqref{equ::i6}, respectively.
Let $k^b$, $\underline{k}$, and $\overline{k}$ be defined as in Proposition \ref{pro::I456sur}. Suppose that the conditions $(i)-(iv)$ in Proposition \ref{pro::I456sur} hold.
Then there exists $N_1\in \mathbb{N}$, which will be specified in the proof, such that for all $n> N_1$ and $x$ satisfying $B(x,R_{s,(k^b)}^b(x))\subset [0,1]^d$ for all $b\in [B]$, the following statements hold with probability $\mathrm{P}^{Bs}$ at least $1-2/n^2$:
\begin{enumerate}
\item[$(I)$] 
$\lesssim \bigl( \overline{k} / s \bigr)^{-1/d}$;
\item[$(II)$] 
$\lesssim 
\bigl(\overline{k}/s\bigr)^{1/d}\bigl((\log n)/ (\overline{k} B)\bigr)^{1/2}$;
\item[$(III)$] 
$\lesssim (\log n)^{1+1/d}/(s^{1/d}\overline{k})+ 
(\overline{k}/s)^{2/d}$.
\end{enumerate}
\end{proposition}

\subsection{Learning the AUC Regret: Three Building Blocks} \label{sec::LearningTheoryAnalysis}

Recalling that the central concern in statistical learning theory is the convergence rates of learning algorithms under various settings. 
In Section \ref{sec::hubermodel}, we show that when the probability distribution $\mathrm{P}$ follows the Huber contamination model (HCM) in Assumption \ref{asp::huber}, 
we can use a latent variable $Y$ to indicate whether it is from the anomalous distribution. 
Moreover, the posterior probability in \eqref{equ::etax} implies that in HCM, anomalies can be identified by using the Bayes classifier with respect to the classification loss, resulting in the set of anomalies as 
\begin{align*} 
\mathcal{S}
:= \{ x \in \mathbb{R}^d : \eta(x) > 1/2 \}
= \{ x \in \mathbb{R}^d : \Pi f(x)^{-1} > 1/2 \}
= \{ x \in \mathbb{R}^d : f(x) < 2 \Pi \}.
\end{align*}
This set can be estimated by the lower-level set estimation of BRDDE at the threshold $2 \Pi $ as in \eqref{equ::fbstar}, i.e., $ \widehat{\mathcal{S}} := \{x \in \mathbb{R}^d : f_n^{B,*}(x) < 2 \Pi \}$ with $f_n^{B,*}(x)$ as defined in \eqref{equ::fbstar}. If we choose 
\begin{align*}
\theta =  \frac{1}{(2V_d\Pi )^{1/d}B}\sum_{b=1}^B \sum_{i=1}^s w_i^{b,*} \gamma_{s,i} ,
\end{align*}
then we have 
\begin{align*} 
\{ x \in \mathbb{R}^d : R_n^{B,*}(x) \geq \theta\}
= \{ x \in \mathbb{R}^d : f_n^{B,*}(x) < 2 \Pi \}
= \widehat{\mathcal{S}}.
\end{align*}
This implies that the upper-level set of bagged regularized $k$-distances, i.e., $\{ x \in \mathbb{R}^d : R_n^{B,*}(x) \geq \theta\}$, equals the estimation $\widehat{\mathcal{S}}$ with the properly chosen threshold.
As a result, the unsupervised anomaly detection problem is converted to an implicit binary classification problem.
Therefore, we are able to analyze the performance of $R_n^{B,*}(x)$ in anomaly detection by applying the analytical tools for classification.
Since the posterior probability estimation is inversely proportional to the BRDDE as shown in \eqref{equ::etax} in Section \ref{sec::hubermodel},
the problem of analyzing the posterior probability estimation can be further converted to analyzing the BRDDE.
Therefore, it is natural and necessary to investigate the following three problems:
\begin{enumerate}
\item[(i)] 
The finite sample bounds of the optimized weights $w^{b,*}$ by solving SRM problems.
\item[(ii)] 
The convergence of the BRDDE as stated in Theorem \ref{pro::rateBWDDE}, that is, whether $f_n^{B,*}$ converges to $f$ in terms of $L_1$-norm.
\item[(iii)] 
The convergence of AUC regret for $R_n^{B,*}$, i.e.,~whether the convergences of BRDDE $f_n^{B,*}$ imply the convergences of the AUC regret of $R_n^{B,*}$. 
\end{enumerate}

\begin{center}
\begin{tikzpicture}[node distance = 5cm, auto]\label{fig::learning}
\tikzstyle{block} = [rectangle, draw, text width=9.5em, text centered, rounded corners, minimum height=4em, font=\large]
\tikzstyle{line} = [draw, line width=2pt, -{latex}]  

\node [block] (init1) {
$\mathrm{Reg}^{\mathrm{AUC}}(R_n^{B,*})\to 0$};
\node [block, right of=init1] (identify1) {
$\|f_n^{B,*} - f\|_{L_1(\mathcal{X})}\to 0$\\
\vspace{+1mm}
(Theorem \ref{pro::rateBWDDE})
};
\node [block, right of=identify1] (evaluate1) {SRM};

\path [line] (identify1) -- (init1);
\path [line] (evaluate1) -- (identify1);
\end{tikzpicture}
\captionof{figure}{An illustration of the three fundamental components of AUC regret. The left block represents the consistency of AUC regret, the middle block signifies the consistency of BRDDE, and the right block corresponds to the statistical analysis of SRM, aligning with Problem (iii), (ii), and (i), respectively.}
\end{center}

The above three problems form the foundations for conducting a learning theory analysis on bagged regularized $k$-distances and serve as three main building blocks.
Notice that Problem (ii) is already provided in Theorem \ref{pro::rateBWDDE} in Section \ref{sec::convergencebrdde}. 
Detailed explorations of the other two Problems (i) and (iii), will be expanded in the following subsections.

\subsubsection{Analysis for the Surrogate Risk Minimization}
\label{sec::analysissr}

The following Proposition ensures that the solution to the surrogate risk minimization satisfies the conditions in Proposition \ref{pro::I456sur} with high probability, thereby addressing Problem (i) and validating the effectiveness of our algorithm.

\begin{proposition}\label{lem::order:k:w:lambda}
Let Assumption \ref{asp::holder} hold.
Let $\{ D_s^b\}_{b=1}^B$ be $B$ disjoint subsets of size $s$ randomly drawn from the data set $D_n$. Moreover, let $w^{b,*}$ be defined as in \eqref{equ::bagwstar}, and $k^{b,*} := k(w^{b,*}) := \sup \{ i \in [n] : w_i^{b,*} \neq 0 \}$.
If we choose $s$ and $B$ as in \eqref{equ::choicesb} in Theorem \ref{pro::rateBWDDE}, then there exists $N_2\in \mathbb{N}$, which will be specified in the proof, such that for all $n>N_2$, the following two statements hold  with probability $\mathrm{P}^{Bs}$ at least $1-1/n^2$:
\vspace{-0.5mm}
\begin{enumerate}
\item[1.] The conditions $(i)-(iv)$ in Proposition \ref{pro::I456sur} hold for $w^{b,*}$ and $k^{b,*}$.
\item[2.] Furthermore, we have $k^{b,*}\asymp (n/\log n)^{1/(d+2)}$ for $b\in [B]$.
\end{enumerate}
\end{proposition}

\subsubsection{Analysis for the AUC Regret}
\label{sec::analysisaucregret}

Problem (iii) in the left block of Figure \ref{fig::learning} is solved by the next proposition, which shows that the problem of bounding the AUC regret of the bagged regularized $k$-distances can be converted to the problem of bounding the $L_1$-error of the BRDDE.

\begin{proposition}\label{lem:lemfrometatoauc}
Let Assumptions \ref{asp::huber} and \ref{asp::holder} hold, and let $f$ be the true density function.
Let $R_n^{B,*}$ be the bagged regularized $k$-distances as in \eqref{equ::rlambdaBx} and $f_n^{B,*}(x)$ be the BRDDE as in \eqref{equ::fbstar}.
Suppose that there exists a constant $c>0$ such that $\|f_n^{B,*}\|_{\infty}\geq c$. 
Then, we have 
\begin{align*}
\mathrm{Reg}^{\mathrm{AUC}}(R_n^{B,*}) 
\displaystyle \lesssim \int_{\mathcal{X}} \bigl|f_n^{B,*}(x)-f(x)\bigr|dx.
\end{align*}
\end{proposition}

The results in Proposition \ref{lem:lemfrometatoauc} apply broadly to any density estimator that is lower bounded and whose weights satisfy the conditions outlined in Propositions \ref{pro::I456} and \ref{lem::order:k:w:lambda}. Indeed, this means any sufficiently good (weighted) $k$-NN density estimator meeting these conditions would achieve similar theoretical guarantees.
Therefore, our proposed estimator does not claim a faster convergence rate than existing methods. Instead, its primary advantage lies in reducing sensitivity to hyperparameter selection in practice, as thoroughly discussed in Section \ref{sec::datadrivennnad}. The theoretical guarantees established in Proposition \ref{lem:lemfrometatoauc} play a crucial role in validating the consistency of our estimator in comparison to standard approaches.

\subsection{Complexity Analysis} \label{sec:complexity}

To deal with the efficiency issue in distance-based methods for anomaly detection when dealing with large-scale datasets, \cite{wu2006outlier} proposed the iterative subsampling, i.e., for each test sample, they first randomly select a portion of data and then compute the $k$-distance over the subsamples.
They provided a probabilistic analysis of the quality of the subsampled distance compared to the $k$-distance over the whole dataset.
Furthermore, \cite{sugiyama2013rapid} proposed the one-time sampling for the computation of the $k$-distances over the dataset for all test samples,
which is shown to be more efficient than the iterative sampling.
Although these sub-sampling methods improve computational efficiency, these distance-based methods fail to comprehensively utilize the information in the dataset since a large portion of samples are dropped out.
By contrast, the bagging technique incorporated in our BRDAD not only addresses the efficiency issues when dealing with large-scale datasets but also maintains the ability to make full use of the data.
In the following, we conduct a complexity analysis for BRDAD in detail to show the computational efficiency of BRDAD.

As a widely used algorithm, the $k$-d tree \citep{friedman1977algorithm} is commonly employed in NN-based methods to search for nearest neighbors. Given $n$ data points in $d$ dimensions, \citet{friedman1977algorithm} showed that constructing a $k$-d tree requires $\mathcal{O}(nd \log n)$ time, while searching for $k$ nearest neighbors takes $\mathcal{O}(k \log n)$ time. Below, we analyze the time complexities of the construction and search stages in BRDAD to demonstrate how bagging reduces computational complexity.
\begin{enumerate}
\item[$(i)$] 
In each bagging round, BRDAD constructs a $k$-d tree using $s$ data points. Taking $s$ as the order in Theorem \ref{pro::rateBWDDE}, the construction time per $k$-d tree is $\mathcal{O}(ds\log s) = \mathcal{O}(dn^{(1+d)/(d+2)} (\log n)^{1/(d+2)})$ when parallelism is applied. In contrast, without bagging, constructing a single $k$-d tree requires $\mathcal{O}(dn \log n)$ time. Thus, bagging reduces the construction time complexity.
\item[$(ii)$] 
The time complexity of regularized $k$-distances at each bagging round consists of two main components: $(i)$ computing the average $k$-distances, and $(ii)$ solving the SRM problem. The first part, querying $k^{b,*}$ neighbors, takes $\mathcal{O}(k^{b,*}\log s)$ time. For the second part, \citet[Theorem 3.3]{anava2016k} shows that Algorithm \ref{alg::SRM} finds the solution in $\mathcal{O}(k^{b,*})$ time. Consequently, the search stage requires at most $\mathcal{O}(k^{b,*} \log s)$ time. From the order of $k^{b,*}$ and $s$ in Proposition \ref{lem::order:k:w:lambda} and Theorem \ref{thm::bagaverage}, the search complexity with parallel computation across subsets $\{D_s^b\}_{b=1}^B$ is $\mathcal{O}(n^{1/(d+2)} (\log n)^{(d+1)/(d+2)})$. In contrast, without bagging, the search complexity is $\mathcal{O}(n^{2/(d+2)} (\log n)^{(2d+2)/(d+2)})$ from the order of $k^*$ in Lemma \ref{lemm::kbstarnobag} in Section \ref{sec::proofsbrdde},
This result shows that bagging also improves the search efficiency.
\end{enumerate}
\vspace{-1mm}
In summary, with proper parallelization, the overall complexity is dominated by the construction stage at $\mathcal{O}(dn^{(1+d)/(d+2)} (\log n)^{1/(d+2)})$, compared to $\mathcal{O}(dn \log n)$ without bagging. Thus, bagging enhances computational efficiency when fully leveraging parallel computation. However, due to its dependence on $d$, the computational improvement becomes less significant in higher dimensions.

For popular distance-based anomaly detection methods like standard $k$-NN and DTM \citep{gu2019statistical}, the primary computational cost comes from constructing a $k$-d tree and searching for $k$ nearest neighbors. If $k$ is set to the optimal order $\mathcal{O}(n^{2/(d+2)}(\log n)^{d/(d+2)})$ for standard $k$-NN density estimation, the construction stage takes $\mathcal{O}(n d \log n)$ time, while the search stage requires $\mathcal{O}(n^{2/(d+2)} (\log n)^{(2d+2)/(d+2)})$ time.
For another distance-based method, LOF, in addition to constructing a $k$-d tree and searching for $k$ nearest neighbors, there is an extra step of computing scores for all samples, which adds a time complexity of $\mathcal{O}(n)$ \citep{breunig2000lof}.
A straightforward comparison shows that these methods are significantly more computationally intensive than BRDAD.
In contrast, the construction and search stages of BRDAD have much lower complexities of $\mathcal{O}(dn^{(1+d)/(d+2)} (\log n)^{1/(d+2)})$ and $\mathcal{O}(n^{1/(d+2)}(\log n)^{(d+1)/(d+2)})$, respectively.

\section{Experiments} \label{sec::Experiments}

This section presents numerical experiments. In Section \ref{subsec::synexp}, we perform synthetic data experiments on density estimation to illustrate the convergence of the surrogate risk and mean absolute error of BRDDE as the sample size grows. Section \ref{subsec::realexp} focuses on real-world anomaly detection benchmarks, where we compare BRDAD with various methods and explain its advantages. Our empirical results demonstrate that bagging enhances the algorithm's performance.

\subsection{Synthetic Data Experiments on Density Estimation} 
\label{subsec::synexp}

In this section, we empirically demonstrate the convergence of both the surrogate risk (SR) and the mean absolute error (MAE) of BRDDE as the sample size $n$ increases. The results are presented in Figures \ref{fig::SR_converge} and \ref{fig::MAE_converge}. As established in Theorem \ref{pro::I456sur}, the surrogate risk is expected to exhibit a convergence behavior similar to that of the MAE for BRDDE. To investigate this, we sample $n$ data points from the $\mathcal{N}(0, 1)$ distribution, with $n$ set to $\{300, 1000, 3000, 5000, 10000\}$ for training. The surrogate risk for each $n$ is computed using Algorithm \ref{alg::SRM} with $B=1$.  
Additionally, we randomly sample $10,000$ instances to calculate the MAE to assess BRDDE’s performance. Each experiment is repeated 20 times for every sample size $n$. The results in Figure \ref{fig::SR_converge} show that the surrogate risk decreases monotonically as $n$ increases, while Figure \ref{fig::MAE_converge} exhibits a similar convergence pattern for the MAE. Furthermore, we plot the ratio of SR to MAE for each sample size $n$ in Figure \ref{fig::Ratio_converge}, which indicates that the ratio stabilizes when $n$ exceeds 3000, further confirming the similarity in their convergence behaviors.
To extend this analysis to more complex distributions and higher dimensions, we conduct additional experiments using a mixture distribution,  
$0.4 \times \mathcal{N}(0.3 \mathbf{1}_d,0.01\mathbf{I}_d) + 0.6 \times \mathcal{N}(0.7 \mathbf{1}_d,0.0025\mathbf{I}_d)$,
across dimensions $d \in \{1,3,5,7,9\}$. The ratio of SR to MAE for each sample size $n$ is plotted in Figure \ref{fig::Ratio_converge-2}, revealing convergence toward a dimension-dependent constant. This result reinforces the generalizability of our optimized weights across diverse settings.

\begin{figure}[!h]
\centering
\subfigure[Convergence of SR]{
\begin{minipage}{0.47\linewidth}
\centering
\centerline{\includegraphics[width = \linewidth]{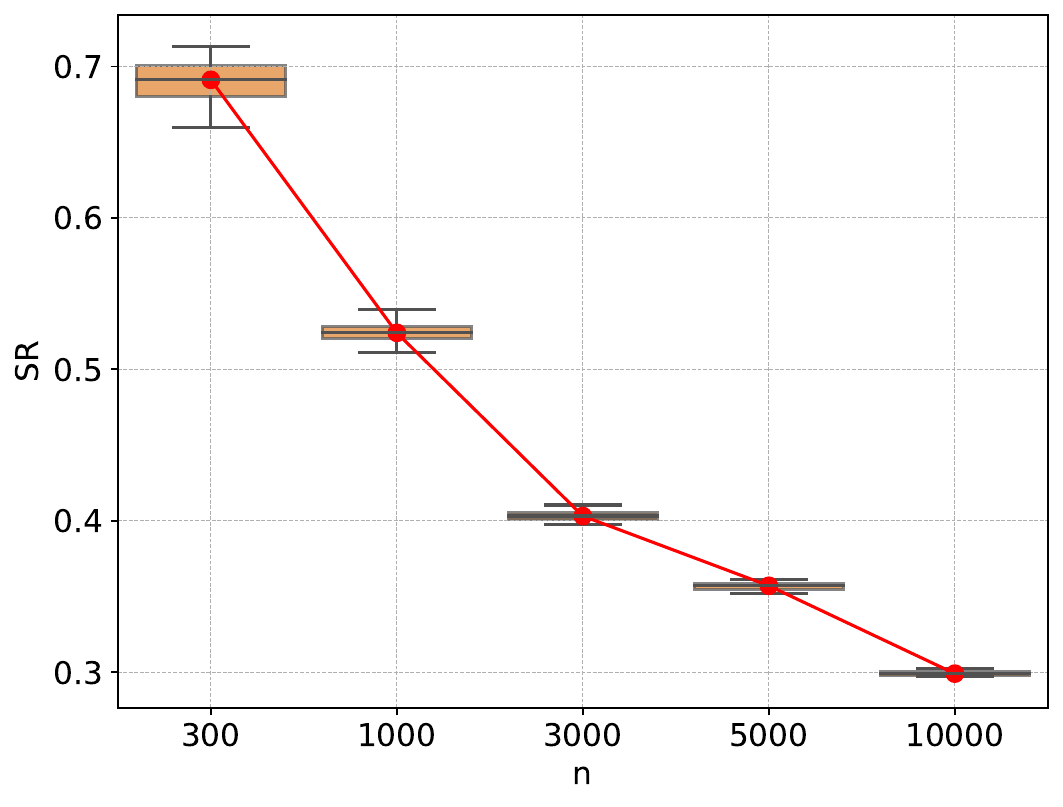}}
\label{fig::SR_converge}
\end{minipage}} 
\hfill
\subfigure[Convergence of MAE]{
\begin{minipage}{0.475\linewidth}
\centering
\centerline{\includegraphics[width = \linewidth]{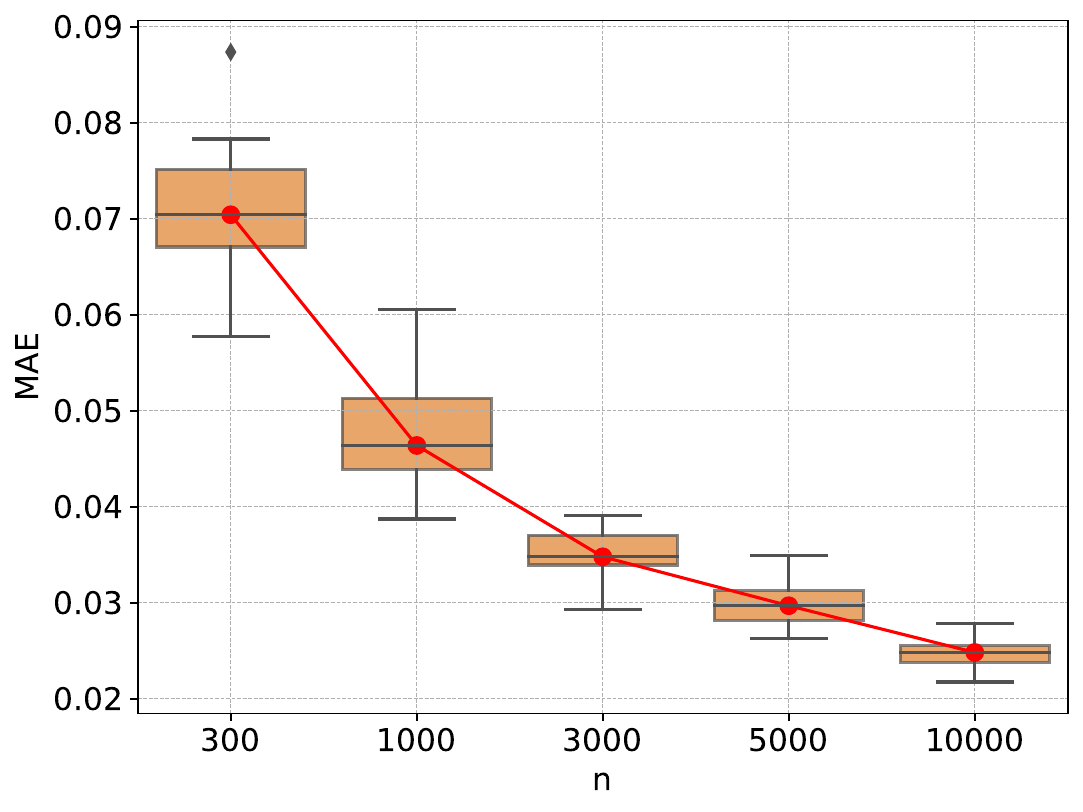}}
\label{fig::MAE_converge}
\end{minipage}} 
\\
\subfigure[Ratio of SR to MAE]{
\begin{minipage}{0.485\linewidth}
\vspace{+1.5mm}
\centering
\centerline{\includegraphics[width = \linewidth]{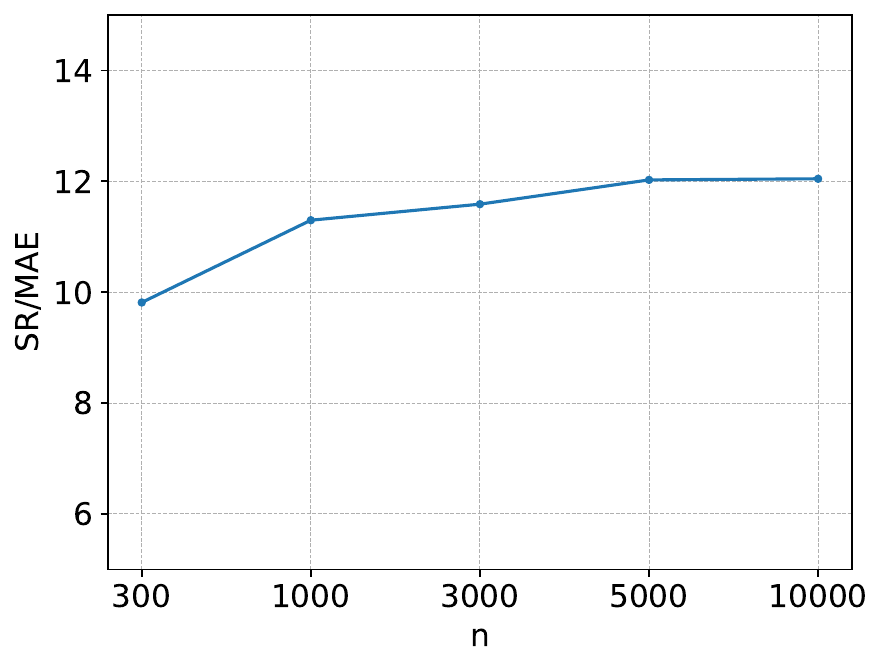}}
\label{fig::Ratio_converge}
\end{minipage}}
\subfigure[Log ratio of SR to MAE]{
\begin{minipage}{0.485\linewidth}
\vspace{+1.5mm}
\centering
\centerline{\includegraphics[width = \linewidth]{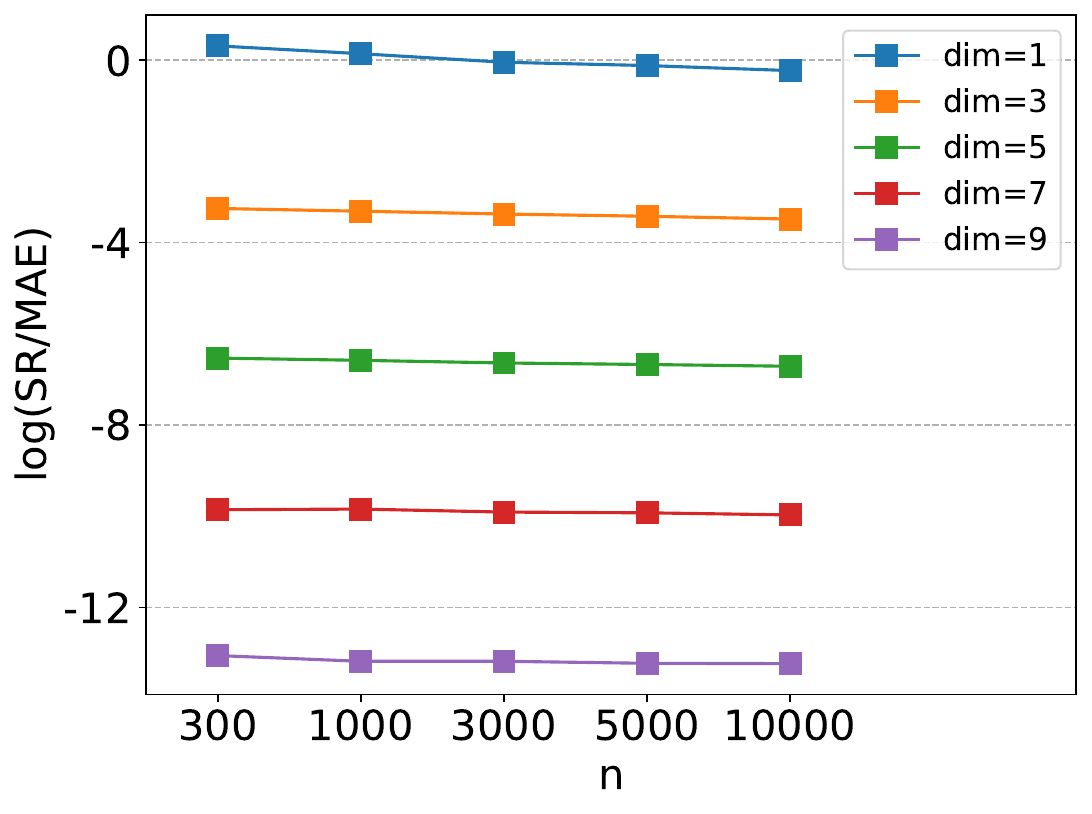}}
\label{fig::Ratio_converge-2}
\end{minipage}}
\caption{
(a)(b) show that SRM leads to the convergence of both surrogate risk (SR) and mean absolute error (MAE).
Furthermore, (c) shows that as the sample size $n$ increases, the ratio of SR to MAE becomes stable, indicating similar convergence behaviors for both SR and MAE by applying Algorithm \ref{alg::SRM}.
}
\vspace{-3.0mm}
\end{figure}

We provide an illustrative example on a synthetic dataset to demonstrate the sensitivity of choosing the hyperparameter $k$ in other nearest-neighbor-based density estimation methods, including the $k$-NN density estimation ($k$-NN) and the weighted $k$-NN density estimation (W$k$NN) \citep{biau2011weighted}. In accordance with \cite{biau2011weighted}, we take $\nu$ as the measure of $U^{\alpha}$, where $U$ is uniform on $[0,1]$ and $\alpha>0$ is a parameter. Given an integer $k$, the weights of W$k$NN are defined by $
w_i = \int_{((i-1)/k,i/k]}\nu(dt) = (i/k)^{1/\alpha}-((i-1)/k)^{1/\alpha}$, for $i\in [k]$, with $w_i=0$ otherwise.
We generate 1000 data points to train the density estimators and an additional 10,000 points to compute the MAE from a Gaussian mixture model with the density function $0.5 \times \mathcal{N}(0.3,0.01) + 0.5 \times \mathcal{N}(0.7,0.0025)$. We vary the hyperparameter $k$ from 3 to 500 to observe its effect on the MAE for both $k$-NN and W$k$NN. The results in Figure \ref{fig:kinsensitivity} illustrate that the performance of these density estimators is significantly influenced by the choice of $k$, regardless of $\alpha$. Only a narrow range of $k$ values leads to optimal results.
In contrast, our proposed estimator, BRDDE, mitigates this sensitivity. The black dashed line in Figure \ref{fig:kinsensitivity} represents the MAE performance of BRDDE, demonstrating that it achieves near-optimal results comparable to the best-tuned nearest-neighbor-based methods—without requiring fine-tuning of $k$.

\begin{figure}
\centering
\includegraphics[width=0.45\textwidth]{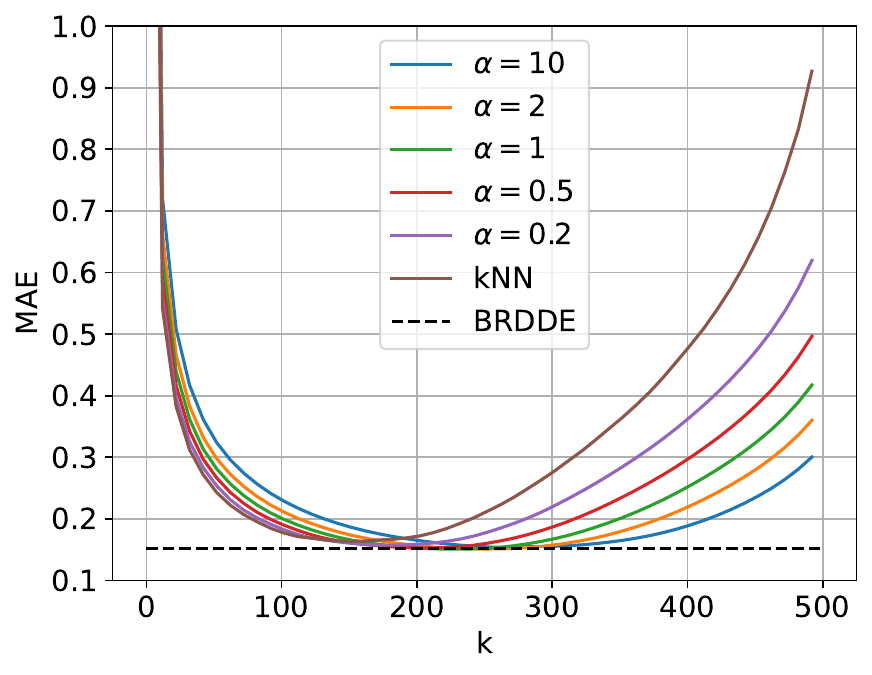}
\vspace{-5mm}
\caption{Illustration of parameter $k$'s sensitivity.}
\label{fig:kinsensitivity}
\vspace{-4mm}
\end{figure}

\subsection{Real-world Data Experiments on Anomaly Detection} \label{subsec::realexp}

\subsubsection{Dataset Descriptions}

To provide an extensive experimental evaluation, we use the latest anomaly detection benchmark repository named ADBench established by \citet{han2022adbench}.
The repository includes 47 tabular datasets, ranging from 80 to 619326 instances and from 3 to 1555 features.
We provide the descriptions of these datasets in the Table \ref{tab::detail_adbench}.

\begin{table}
\centering
\captionsetup{justification=centering}
\caption{\normalsize{Descriptions of ADBench Datasets}}
\label{tab::detail_adbench}
\resizebox{1\linewidth}{!}{
\begin{tabular}{|c|l|c|c|c|c|l|}
\toprule
\textbf{Number} & \textbf{Data} & \textbf{\# Samples} & \textbf{\# Features} & \textbf{\# Anomaly} & \textbf{\% Anomaly} & \textbf{Category} \\
\midrule
1 & ALOI & 49534 & 27 & 1508 & 3.04 & Image \\
2 & annthyroid & 7200 & 6 & 534 & 7.42 & Healthcare \\
3 & backdoor & 95329 & 196 & 2329 & 2.44 & Network \\
4 & breastw & 683 & 9 & 239 & 34.99 & Healthcare \\
5 & campaign & 41188 & 62 & 4640 & 11.27 & Finance \\
6 & cardio & 1831 & 21 & 176 & 9.61 & Healthcare \\
7 & Cardiotocography & 2114 & 21 & 466 & 22.04 & Healthcare \\
8 & celeba & 202599 & 39 & 4547 & 2.24 & Image \\
9 & census & 299285 & 500 & 18568 & 6.20 & Sociology \\
10 & cover & 286048 & 10 & 2747 & 0.96 & Botany \\
11 & donors & 619326 & 10 & 36710 & 5.93 & Sociology \\
12 & fault & 1941 & 27 & 673 & 34.67 & Physical \\
13 & fraud & 284807 & 29 & 492 & 0.17 & Finance \\
14 & glass & 214 & 7 & 9 & 4.21 & Forensic \\
15 & Hepatitis & 80 & 19 & 13 & 16.25 & Healthcare \\
16 & http & 567498 & 3 & 2211 & 0.39 & Web \\
17 & InternetAds & 1966 & 1555 & 368 & 18.72 & Image \\
18 & Ionosphere & 351 & 32 & 126 & 35.90 & Oryctognosy \\
19 & landsat & 6435 & 36 & 1333 & 20.71 & Astronautics \\
20 & letter & 1600 & 32 & 100 & 6.25 & Image \\
21 & Lymphography & 148 & 18 & 6 & 4.05 & Healthcare \\
22 & magic.gamma & 19020 & 10 & 6688 & 35.16 & Physical \\
23 & mammography & 11183 & 6 & 260 & 2.32 & Healthcare \\
24 & mnist & 7603 & 100 & 700 & 9.21 & Image \\
25 & musk & 3062 & 166 & 97 & 3.17 & Chemistry \\
26 & optdigits & 5216 & 64 & 150 & 2.88 & Image \\
27 & PageBlocks & 5393 & 10 & 510 & 9.46 & Document \\
28 & pendigits & 6870 & 16 & 156 & 2.27 & Image \\
29 & Pima & 768 & 8 & 268 & 34.90 & Healthcare \\
30 & satellite & 6435 & 36 & 2036 & 31.64 & Astronautics \\
31 & satimage-2 & 5803 & 36 & 71 & 1.22 & Astronautics \\
32 & shuttle & 49097 & 9 & 3511 & 7.15 & Astronautics \\
33 & skin & 245057 & 3 & 50859 & 20.75 & Image \\
34 & smtp & 95156 & 3 & 30 & 0.03 & Web \\
35 & SpamBase & 4207 & 57 & 1679 & 39.91 & Document \\
36 & speech & 3686 & 400 & 61 & 1.65 & Linguistics \\
37 & Stamps & 340 & 9 & 31 & 9.12 & Document \\
38 & thyroid & 3772 & 6 & 93 & 2.47 & Healthcare \\
39 & vertebral & 240 & 6 & 30 & 12.50 & Biology \\
40 & vowels & 1456 & 12 & 50 & 3.43 & Linguistics \\
41 & Waveform & 3443 & 21 & 100 & 2.90 & Physics \\
42 & WBC & 223 & 9 & 10 & 4.48 & Healthcare \\
43 & WDBC & 367 & 30 & 10 & 2.72 & Healthcare \\
44 & Wilt & 4819 & 5 & 257 & 5.33 & Botany \\
45 & wine & 129 & 13 & 10 & 7.75 & Chemistry \\
46 & WPBC & 198 & 33 & 47 & 23.74 & Healthcare \\
47 & yeast & 1484 & 8 & 507 & 34.16 & Biology \\
\bottomrule
\end{tabular}
}
\end{table}

\subsubsection{Methods for Comparison}
\label{sec::CompaMethods}

We conduct experiments on the following anomaly detection algorithms.
\begin{itemize}
\item[\textit{(i)}] BRDAD is our proposed algorithm, with details provided in Algorithm \ref{alg:BRDAD}.  
The choice of $B$ depends on the sample size: for $n \in (0, 10,000]$, $(10,000, 50,000]$, and $(50,000, +\infty)$, we set $B = 1$, $5$, and $10$, respectively.
In practice, when $B$ is fixed, we randomly divide the data into $B$ subsets, each containing either $\lfloor n/B \rfloor$ or $\lfloor n/B \rfloor +1$ samples. Each subset is then further split into two parts such that their sizes are equal or differ by at most 1. This process ensures that the full dataset is partitioned as evenly as possible.
\item[\textit{(ii)}] Distance-To-Measure (DTM) \citep{gu2019statistical} is a distance-based algorithm which employs a generalization of the $k$ nearest neighbors named ``distance-to-measure''. We use the author's implementation.
As suggested by the authors, the number of neighbors $k$ is fixed to be $k=0.03 \times \text{sample size}$.
\item[\textit{(iii)}] $k$-Nearest Neighbors ($k$-NN) \citep{ramaswamy2000efficient} is a distance-based algorithm that uses the distance of a point from its $k$-th nearest neighbor to distinguish anomalies. We use the implementation of the Python package {\tt PyOD} with its default parameters.
\item[\textit{(iv)}] Local Outlier Factor (LOF) \citep{breunig2000lof} is a distance-based algorithm that measures the local deviation of the density of a given data point with respect to its neighbors. We also use {\tt PyOD} with its default parameters.
\item[\textit{(v)}] Partial Identification Forest (PIDForest) \citep{gopalan2019pidforest} is a forest-based algorithm that computes the anomaly score of a point by determining the minimum density of data points across all subcubes partitioned by decision trees. We use the authors' implementation with the number of trees $T=50$, the number of buckets $B=5$, and the depth of trees $p=10$ suggested by the authors.
\item[\textit{(vi)}] Isolation Forest (iForest) \citep{liu2008isolation} is a forest-based algorithm that works by randomly partitioning features of the data into smaller subsets and distinguishing between normal and anomalous points based on the number of ``splits'' required to isolate them, with anomalies requiring fewer splits. We use the implementation of the Python package {\tt PyOD} with its default parameters.
\item[\textit{(vii)}] One-class SVM (OCSVM) \citep{scholkopf1999support} is a kernel-based algorithm which tries to separate data from the origin in the transformed high-dimensional predictor space.  We also use {\tt PyOD} with its default parameters.
\end{itemize}

Note that as BRDAD, iForest, and PIDForest are randomized algorithms, we repeat these three algorithms for 10 runs and report the averaged AUC performance. DTM, $k$-NN, LOF, and OCSVM are deterministic, and hence we report a single AUC number for them.

\subsubsection{Experimental Results}

\begin{table}
\centering
\captionsetup{justification=centering}
\caption{\normalsize{Experimental Comparisons on ADBench Datasets}}
\label{tab::res_adbench}
\resizebox{0.9\linewidth}{!}{
\begin{tabular}{llllllll}
\toprule
&      BRDAD &        DTM &        $k$-NN &        LOF &  PIDForest &    iForest &      OCSVM \\
\midrule
ALOI             &           0.5427 &           0.5440 &           0.6942 &  \textbf{0.7681} &           0.5061 &           0.5411 &           0.5326 \\
annthyroid       &           0.6516 &           0.6772 &           0.7343 &           0.7076 &  \textbf{0.8781} &           0.8138 &           0.5842 \\
backdoor         &           0.8490 &  \textbf{0.9216} &           0.6682 &           0.7135 &           0.6965 &           0.7238 &           0.8465 \\
breastw          &  \textbf{0.9883} &           0.9799 &           0.9765 &           0.3907 &           0.9750 &           0.9871 &           0.8052 \\
campaign         &           0.6711 &           0.6908 &           0.7202 &           0.5366 &  \textbf{0.7945} &           0.7182 &           0.6630 \\
cardio           &           0.9142 &           0.8879 &           0.7330 &           0.6372 &           0.8258 &           0.9271 &  \textbf{0.9286} \\
Cardiotocography &           0.6302 &           0.6043 &           0.5449 &           0.5705 &           0.5587 &           0.6973 &  \textbf{0.7872} \\
celeba           &           0.5896 &           0.6929 &           0.5666 &           0.4332 &           0.6732 &           0.6955 &  \textbf{0.6962} \\
census           &           0.6394 &           0.6435 &  \textbf{0.6465} &           0.5501 &           0.5543 &           0.6116 &           0.5336 \\
cover            &  \textbf{0.9301} &           0.9277 &           0.7961 &           0.5262 &           0.8065 &           0.8784 &           0.9141 \\
donors           &           0.7858 &  \textbf{0.8000} &           0.6117 &           0.5977 &           0.6945 &           0.7810 &           0.7323 \\
fault            &  \textbf{0.7591} &           0.7587 &           0.7286 &           0.5827 &           0.5437 &           0.5714 &           0.5074 \\
fraud            &           0.9552 &  \textbf{0.9583} &           0.9342 &           0.4750 &           0.9489 &           0.9493 &           0.9477 \\
glass            &           0.7993 &  \textbf{0.8688} &           0.8640 &           0.8114 &           0.7913 &           0.7933 &           0.4407 \\
Hepatitis        &           0.6954 &           0.6303 &           0.6745 &           0.6429 &  \textbf{0.7186} &           0.6944 &           0.6418 \\
http             &           0.9943 &           0.0507 &           0.2311 &           0.3550 &           0.9870 &  \textbf{0.9999} &           0.9949 \\
InternetAds      &  \textbf{0.7274} &           0.7063 &           0.7110 &           0.6485 &           0.6754 &           0.6913 &           0.6890 \\
Ionosphere       &           0.9113 &           0.9237 &  \textbf{0.9259} &           0.8609 &           0.6820 &           0.8493 &           0.7395 \\
landsat          &           0.6176 &  \textbf{0.6184} &           0.5773 &           0.5497 &           0.5245 &           0.4833 &           0.3660 \\
letter           &           0.8426 &           0.8417 &  \textbf{0.8950} &           0.8872 &           0.6636 &           0.6318 &           0.4843 \\
Lymphography     &           0.9988 &           0.9965 &           0.9988 &           0.9953 &           0.9656 &  \textbf{0.9993} &           0.9977 \\
magic.gamma      &           0.8205 &           0.8214 &  \textbf{0.8323} &           0.6712 &           0.7252 &           0.7316 &           0.5947 \\
mammography      &           0.8132 &           0.8301 &           0.8424 &           0.7398 &           0.8453 &  \textbf{0.8592} &           0.8412 \\
mnist            &           0.8335 &  \textbf{0.8630} &           0.8041 &           0.6498 &           0.5366 &           0.7997 &           0.8204 \\
musk             &           0.7583 &           0.9987 &           0.6604 &           0.4271 &  \textbf{0.9997} &           0.9995 &           0.8094 \\
optdigits        &           0.3912 &           0.5474 &           0.4189 &           0.5831 &  \textbf{0.8248} &           0.6970 &           0.5336 \\
PageBlocks       &           0.8889 &           0.8859 &           0.7813 &           0.7345 &           0.8154 &  \textbf{0.8980} &           0.8903 \\
pendigits        &           0.9174 &  \textbf{0.9581} &           0.7127 &           0.4821 &           0.9214 &           0.9515 &           0.9354 \\
Pima             &  \textbf{0.7291} &           0.7224 &           0.7137 &           0.5978 &           0.6842 &           0.6803 &           0.6022 \\
satellite        &  \textbf{0.7449} &           0.7375 &           0.6489 &           0.5436 &           0.7122 &           0.7043 &           0.5972 \\
satimage-2       &           0.9991 &  \textbf{0.9991} &           0.9164 &           0.5514 &           0.9919 &           0.9935 &           0.9747 \\
shuttle          &           0.9898 &           0.9442 &           0.6317 &           0.5239 &           0.9885 &  \textbf{0.9968} &           0.9823 \\
skin             &  \textbf{0.7570} &           0.7177 &           0.5881 &           0.5756 &           0.7071 &           0.6664 &           0.4857 \\
smtp             &           0.8506 &           0.8854 &           0.8953 &           0.9023 &  \textbf{0.9203} &           0.9077 &           0.7674 \\
SpamBase         &           0.5687 &           0.5663 &           0.4977 &           0.4581 &  \textbf{0.6941} &           0.6212 &           0.5251 \\
speech           &           0.4834 &           0.4810 &           0.4832 &  \textbf{0.5067} &           0.4739 &           0.4648 &           0.4639 \\
Stamps           &  \textbf{0.8980} &           0.8594 &           0.8362 &           0.7269 &           0.8883 &           0.8911 &           0.8179 \\
thyroid          &           0.9353 &           0.9470 &           0.9508 &           0.8075 &           0.9687 &  \textbf{0.9771} &           0.8437 \\
vertebral        &           0.3236 &           0.3663 &           0.3768 &  \textbf{0.4208} &           0.2857 &           0.3515 &           0.3852 \\
vowels           &           0.9489 &           0.9667 &  \textbf{0.9797} &           0.9443 &           0.7817 &           0.7590 &           0.5507 \\
Waveform         &  \textbf{0.7783} &           0.7685 &           0.7457 &           0.7133 &           0.7263 &           0.7144 &           0.5393 \\
WBC              &  \textbf{0.9972} &           0.9930 &           0.9925 &           0.8399 &           0.9904 &           0.9959 &           0.9967 \\
WDBC             &           0.9841 &           0.9773 &           0.9782 &           0.9796 &  \textbf{0.9916} &           0.9850 &           0.9877 \\
Wilt             &           0.3138 &           0.3545 &           0.4917 &  \textbf{0.5394} &           0.5012 &           0.4477 &           0.3491 \\
wine             &  \textbf{0.8788} &           0.4277 &           0.4992 &           0.8756 &           0.8221 &           0.7987 &           0.6941 \\
WPBC             &           0.5188 &           0.5101 &           0.5208 &           0.5184 &  \textbf{0.5283} &           0.4942 &           0.4743 \\
yeast            &           0.3717 &           0.3876 &           0.3936 &  \textbf{0.4571} &           0.4019 &           0.3964 &           0.4141 \\
\midrule
\textbf{Rank Sum} & \textbf{145} &  160 & 192 & 243 & 187 &  159  &  228 \\
\textbf{Num. No. 1} & \textbf{11} & 8 & 5 & 5 & 9 & 6 & 3 \\
\bottomrule
\end{tabular}
}
\end{table}

Table \ref{tab::res_adbench} presents the performance of seven anomaly detection methods on the ADBench benchmark, evaluated using the AUC metric. The last two rows summarize each algorithm’s rank sum and the number of top-ranking performances. A lower rank sum and a higher number of first-place rankings indicate better performance.
BRDAD demonstrates exceptional results across both evaluation metrics. Specifically, it achieves the lowest rank sum of 145, significantly outperforming other methods, with DTM and iForest scoring 160 and 159, respectively.
In terms of first-place rankings across datasets, BRDAD ranks first in 11 out of 47 tabular datasets, outperforming PIDForest (9/47) and DTM (8/47).
Overall, BRDAD exhibits outstanding performance, excelling over previous distance-based methods while competing effectively with forest-based approaches.
\begin{itemize}
\item On the one hand, BRDAD outperforms distance-based methods such as DTM and $k$-NN in some datasets. For example, on the Satellite dataset, while DTM achieves a high AUC of 0.7375, BRDAD further improves it to 0.7449. Similarly, on the InternetAds dataset, BRDAD achieves an AUC of 0.7274, outperforming $k$-NN’s score of 0.7110.
\item On the other hand, BRDAD remains competitive even in datasets where distance-based methods perform poorly while forest-based methods excel, such as the Stamps and Wine datasets. On Stamps, DTM and $k$-NN achieve AUC scores of 0.8594 and 0.8362, respectively, whereas forest-based methods like PIDForest and iForest attain 0.8883 and 0.8911. Surprisingly, despite being a distance-based method, BRDAD surpasses them all with an AUC of 0.8980.  
Similarly, on the Wine dataset, PIDForest and iForest achieve AUC scores of 0.8221 and 0.7987, respectively, while distance-based methods DTM and $k$-NN perform significantly worse, with scores of 0.4277 and 0.4992. Remarkably, BRDAD not only outperforms other distance-based methods but also surpasses forest-based methods, achieving the highest AUC of 0.8788.
\end{itemize}
These results empirically demonstrate BRDAD’s superiority over both distance-based and forest-based anomaly detection algorithms.

\subsubsection{Parameter Analysis}\label{sec:sensitivity}

In this section, we analyze the impact of the bagging rounds $B$ on BRDAD’s performance using the ADBench datasets.  
Following the sample size categorization in Section \ref{sec::CompaMethods}, we classify the datasets into three groups: small datasets with $n \in (0, 10,000]$, medium datasets with $n \in (10,000, 50,000]$, and large datasets with $n \in (50,000, +\infty)$. We evaluate BRDAD with different values of $B \in \{1, 5, 10, 20\}$ and record the rank sum for each setting.

\begin{table}[htbp]
\centering
\captionsetup{justification=centering}
\caption{\normalsize{Experimental Comparisons for BRDAD with different $B$ on ADBench Datasets}}
\label{tab::res_paramB}
\begin{tabular}{lcccc}
\toprule
&  $B=1$ &       $B=5$  &   $B=10$   &   $B=20$  \\
\midrule
\textbf{Small} &  \textbf{95}&  112 & 106 & 102 \\
\textbf{Medium} & 25  & \textbf{25} & 26 & 26 \\
\textbf{Large} & 26  & 25 & \textbf{25} & 26 \\
\bottomrule
\end{tabular}
\end{table}

Table \ref{tab::res_paramB} indicates that the effect of $B$ on the rank sum metric varies with sample size.  
For small datasets, $B=1$ significantly outperforms other choices, as bagging reduces the already limited sample size.  
In contrast, for medium and large datasets, a slightly larger $B$ yields marginally better performance.  
This aligns with our theoretical findings in Theorem \ref{thm::bagaverage}, which suggest that $B$ should increase with sample size.  
These results also validate the effectiveness of our chosen values of $B$ in real-world datasets, as discussed in the previous subsection.

\section{Proofs}\label{sec::proofs}

In this section, we present proofs of the theoretical results in this paper.
More precisely, we first provide proofs related to the surrogate risk in Section \ref{sec::proofssurrogate}.
The proofs related to the convergence rates of BRDDE and BRDAD are provided in Sections \ref{sec::proofsbrdde} and \ref{sec::proofsbrdadauc}, respectively.

\subsection{Proofs Related to the Surrogate Risk}\label{sec::proofssurrogate}

In this section, we first provide proofs related to the error analysis of BWDDE in Section \ref{sec::prooferrorBWDDE}.
Then in Section \ref{sec::proofssurrogaterisk}, we present the proof of Proposition \ref{sec::datadrivennnad} concerning the surrogate risk.

\subsubsection{Proofs Related to Section \ref{sec::errorBWDDE}} \label{sec::prooferrorBWDDE}

Before we proceed, we present Bernstein's inequality \cite{bernstein1946theory} that will be frequently applied within the subsequent proofs.  
This concentration inequality is extensively featured in numerous statistical learning literature, such as \cite{massart2007concentration, cucker2007learning, steinwart2008support, cai2023extrapolated}.

\begin{lemma}[Bernstein's inequality] \label{lem::bernstein}
	Let $B>0$ and $\sigma>0$ be real numbers, and $n\geq 1$ be an integer. Furthermore, let $\xi_1,\ldots,\xi_n$ be independent random variables satisfying $\mathbb{E}_{\mathrm{P}}\xi_i=0$, $\|\xi_i\|_\infty\leq B$, and $\mathbb{E}_{\mathrm{P}}\xi_{i}^2\leq \sigma^2$ for all $i=1,\ldots,n$. Then for all $\tau>0$, we have
	\begin{align*}
		\mathrm{P}\biggl(\frac{1}{n}\sum_{i=1}^n\xi_i\geq \sqrt{\frac{2\sigma^2\tau}{n}}+\frac{2B\tau}{3n}\biggr)\leq e^{-\tau}.
	\end{align*}
\end{lemma}

Although Bernstein's inequality is a powerful tool for establishing finite-sample bounds for random variables, the resulting bounds may not be tight enough for random variables with vanishing variance. To address this limitation, we employ the sub-exponential tail bound, which provides sharper concentration results for such random variables.

We begin by introducing the definition of sub-exponential random variables as stated in \citet[Definition 2.7]{wainwright2019high} as follows. 

\begin{definition}[Sub-exponential variables]\label{def::subexponential}
A random variable $X$ with mean $\mu = \mathbb{E}X$ is \textit{sub-exponential} if there exist non-negative parameters $(\zeta,\alpha)$ such that 
\begin{align*}
\mathbb{E}[e^{\lambda(X-\mu)}]\leq e^{{\zeta^2\lambda^2}/{2}}, \qquad \text{ for all } |\lambda|\leq 1/\alpha.
\end{align*}
\end{definition}

This condition on the moment-generating function, together with the Chernoff technique, leads to concentration inequalities that bound the deviations of sub-exponential random variables. The following lemma gives this result \citep[Proposition 2.9]{wainwright2019high}.

\begin{lemma}[Sub-exponential tail bound]\label{lem::subtailbound}
Suppose that $X$ is sub-exponential with parameters $(\zeta,\alpha)$. Then for all $0\leq \tau \leq \zeta^2/(2\alpha^2)$, we have
\begin{align*}
\mathrm{P}(|X-\mu|\geq \zeta\sqrt{2\tau}) \leq 2e^{-\tau}.
\end{align*}
\end{lemma}

To measure the complexity of the functional space, we first recall the definition of the covering number in \cite{vandervaart1996weak}.

\begin{definition}[Covering Number]
Let $(\mathcal{X}, d)$ be a metric space and $A \subset \mathcal{X}$. For $\varepsilon>0$, the $\varepsilon$-covering number of $A$ is denoted as 
\begin{align*}
\mathcal{N}(A, d, \varepsilon) 
:= \min \biggl\{ n \geq 1 : \exists x_1, \ldots, x_n \in \mathcal{X} \text{ such that } A \subset \bigcup^n_{i=1} B(x_i, \varepsilon) \biggr\},
\end{align*}
where $B(x, \varepsilon) := \{ x' \in \mathcal{X} : d(x, x') \leq \varepsilon \}$.
\end{definition}

The following Lemma, which is taken from \cite{hang2022under} and needed in the proof of Lemma \ref{lem::Rrho}, provides the covering number of the indicator functions on the collection of balls in $\mathbb{R}^d$.

\begin{lemma}\label{lem::CoveringNumber}
Let $\mathcal{B} := \{ B(x, r) : x \in \mathbb{R}^d, r > 0 \}$ and $\eins_{\mathcal{B}} := \{ \eins_B : B \in \mathcal{B} \}$. Then for any $\varepsilon \in (0, 1)$, there exists a universal constant $C$ such that 
\begin{align*}
\mathcal{N}(\eins_{\mathcal{B}}, \|\cdot\|_{L_1(\mathrm{Q})}, \varepsilon)
\leq C (d+2) (4e)^{d+2} \varepsilon^{-(d+1)}
\end{align*}
holds for any probability measure $\mathrm{Q}$.
\end{lemma}

The following lemma, which will be used several times in the sequel,  provides the uniform bound on the distance between any point and its $k$-th nearest neighbor with a high probability when the distribution has bounded support.

\begin{lemma}\label{lem::Rrho} 
Let Assumption \ref{asp::holder} hold. Furthermore, let $\{ D_s^b\}_{b=1}^B$ be $B$ disjoint subsets of size $s$ randomly drawn from the data set $D_n$, where $D_s^b = \{X_1^b, \ldots, X_s^b\}$ and $R_{s,(i)}^b(x)$ denotes the $i$-distance of $x$ in the subset $D_s^b$ for $i\in [s]$. Additionally, let the sequence $c_n := \lceil 48(2d+9) \log n + 48\log 2 + 144 \rceil$.
Suppose that $s \geq c_1':=\max\{ 4e, d + 2, C\}$, where $C$ is the constant specified in Lemma \ref{lem::CoveringNumber}. Then, there exist constants $0 < c_2' < c_3'$ such that, with probability $ \mathrm{P}^{Bs}$ at least $1 - 1/n^2$, for all $x \in \mathcal{X}$, $b \in [B]$, and $c_n \leq i \leq s$,   the following inequalities hold:
\begin{align}\label{equ::infrksuprk} 
c_2' (i/s)^{1/d}\leq  R_{s,(i)}^b(x)\leq c_3' (i/s)^{1/d}
\end{align}
and
\begin{align}\label{equ::pbxrkxkn} 
\mathrm{P}(B(x,R_{s,(i)}^b(x)))\asymp i/s.
\end{align} 
\end{lemma}

\begin{proof}[of Lemma \ref{lem::Rrho}]
For $x \in \mathcal{X}$ and $q \in [0,1]$, under the continuity of the density function, as stated in Assumption \ref{asp::holder}, the intermediate value theorem guarantees the existence of a unique $\rho_x(q) \geq 0$ such that $\mathrm{P}(B(x, \rho_x(q))) = q$.

Let $\tau:=(2d+9)\log n + \log 2 + 3$ and $c_n: = \lceil 48\tau \rceil =  \lceil 48(2d+9) \log n + 48\log 2 + 144 \rceil$. Let $i$ be an integer fixed in the sequel with $i \geq c_n$, which ensures that $i>3\tau$. Accordingly, we consider the set $\mathcal{B}_i^- :=  \big\{ B \bigl( x, \rho_x \bigl( (i - \sqrt{3\tau i})/s \bigr) \bigr) : x \in \mathcal{X} \bigr\} \subset \mathcal{B}$. 
Lemma \ref{lem::CoveringNumber} implies that for any probability measure $\mathrm{Q}$ and $\varepsilon\in (0,1)$, there holds
\begin{align} \label{Bk-CoveringNumber}
\mathcal{N}(\eins_{\mathcal{B}_{i}^-}, \|\cdot\|_{L_1(\mathrm{Q})}, \varepsilon)
\leq \mathcal{N}(\eins_{\mathcal{B}}, \|\cdot\|_{L_1(\mathrm{Q})}, \varepsilon)
\leq C (d+2) (4e)^{d+2} \varepsilon^{-(d+1)}.
\end{align}
By the definition of the covering number, there exists $\varepsilon$-net $\{A_j^-\}_{j=1}^J \subset  \mathcal{B}_{i}^-$ with $J:= \lfloor C (d+2) (4e)^{d+2} \varepsilon^{-(d+1)} \rfloor$, and for any $x \in \mathcal{X}$, there exists some $j \in \{ 1, \ldots, J \}$ such that 
\begin{align} \label{eq::approxAj-1}
\bigl\| \eins \bigl\{ B \bigl( x, \rho_x \bigl( (i -\sqrt{3\tau i})/s \bigr) \bigr) \bigr\} - \eins_{A_j^-} \bigr\|_{L_1(D_s^b)} 
\leq \varepsilon.
\end{align}
Let $j\in [J]$ and $b\in [B]$ be fixed for now.
For any $\ell \in [s]$, define the random variables $\xi_{\ell,b} = \eins_{A_j^-}(X_\ell^b) -  (i - \sqrt{3\tau i})/s $. 
We have $\mathbb{E}_{\mathrm{P}}\xi_{\ell,b} = 0$, $\|\xi_{\ell,b}\|_{\infty} \leq 1$, and $\mathbb{E}_{\mathrm{P}}\xi_{\ell,b}^2 \leq (i - \sqrt{3\tau i})/s$. Applying Bernstein's inequality in Lemma \ref{lem::bernstein}, we obtain
\begin{align}\label{equ::1ssuml1s}
\frac{1}{s} \sum_{\ell=1}^s \eins_{A_j^-}(X_\ell^b) - (i - \sqrt{3\tau i)})/s
\leq \sqrt{2\tau (i - \sqrt{3\tau i})} / s + 2\tau/ (3s)
\end{align}
for $A_j^-$
with probability $\mathrm{P}^s$ at least $1-e^{-\tau}$. 
The union bound then implies that this inequality  holds for all $A_j^-$, $j = 1, \cdots, J$ with probability $\mathrm{P}^s$ at least $1-Je^{-\tau}$.
Combining this result with \eqref{eq::approxAj-1}, we obtain the following bound:
\begin{align}\label{equ::uniformbound}
&\frac{1}{s} \sum_{\ell=1}^s \eins \bigl\{X_\ell^b\in B \bigl( x, \rho_x \bigl( (i - \sqrt{3\tau i})/s \bigr) \bigr) \bigr\} - (i - \sqrt{3\tau i)})/s \nonumber\\
&\leq \sqrt{2\tau (i - \sqrt{3\tau i})} / s + 2\tau/ (3s) + \varepsilon
\end{align}
for all $x\in \mathcal{X}$ with probability $\mathrm{P}^s$ at least $1-Je^{-\tau}$.
 
Let $\varepsilon = 1/s$. Since $\lfloor x\rfloor \leq  x$ and $s\geq \max\{4e, d+2, C\}$, it follows that $\log J \leq \log C + \log(d+2) +  (d+2)\log(4e) + (d+1) \log s \leq (2d+5)\log s$.
Since $\tau =(2d+9)\log n + \log 2 + 3$ and $s\leq n$, we have the following lower bound on the probability: $1-Je^{-\tau}\geq 1-s^{2d+5}/(2e^3 n^{2d+9})> 1-1/(2n^4)$.
Therefore, for fixed $c_n\leq i\leq s$ and $b\in [B]$, inequality \eqref{equ::uniformbound} holds
for all $x\in \mathcal{X}$
with probability $\mathrm{P}^s$ at least $1-1/(2n^4)$.

Next, note that since $\tau>3$ and $i\geq \lceil 48\tau \rceil \geq 48\tau$,
a straightforward calculation gives
\begin{align}\label{equ::2taui}
&\sqrt{2\tau (i - \sqrt{3\tau i})} / s + 2\tau/ (3s) + \varepsilon \leq \sqrt{2\tau i}/s + (2\tau/3+1)/s \nonumber \\
& \leq \sqrt{2\tau i}/s +\tau/s\leq \sqrt{2\tau i}/s +\sqrt{\tau i/48}/s =(\sqrt{2}+\sqrt{1/48})\sqrt{\tau i}/s\leq \sqrt{3\tau i}/s.
\end{align}
The first inequality holds trivially without requiring $i\geq \lceil 48\tau \rceil$.
However, when considering the set
$\mathcal{B}_i^+ := \big\{ B \bigl( x, \rho_x \bigl( (i + \sqrt{3\tau i})/s \bigr) \bigr) : x \in \mathcal{X} \bigr\}$,  
the condition $i \geq  \lceil 48\tau \rceil$  is necessary to establish $i + \sqrt{3\tau i} \leq 5i/4$, allowing us to proceed with a similar argument.
Combining \eqref{equ::2taui} with \eqref{equ::uniformbound}, we get 
$\sum_{\ell=1}^s \eins \bigl\{ X_\ell^b\in B \bigl( x, \rho_x \bigl( (i - \sqrt{3\tau i})/s \bigr) \bigr) \bigr\} /s \leq i/s$ for all $x\in \mathcal{X}$ with probability $\mathrm{P}^s$ at least $1-1/(2n^4)$. 
By the definition of $R_{s,(i)}^b(x)$, there holds 
\begin{align}\label{equ::Rrho1-}
R_{s,(i)}^b(x) \geq \rho_x \bigl( (i - \sqrt{3\tau i})/s \bigr).
\end{align}
Now, let $V_d$ denote the volume of the $d$-dimensional closed unit ball and $\mu$ denote the Lesbegue measure. By the definition of $\rho_x$ and the condition $f(x)\leq \overline{c}$ for all $x\in \mathcal{X}$, as stated in Assumption \ref{asp::holder}, we have
\begin{align}\label{equ::rxlower}
& \bigl(i - \sqrt{3\tau i}\bigr)/s  = \mathrm{P} \bigl( B ( x, \rho_x \bigl( (i - \sqrt{3\tau i})/s \bigr) \bigr) \bigr) 
=\int_{B ( x, \rho_x ( (i - \sqrt{3\tau i})/s ) )} f(u) \, du \nonumber
\\
&\leq \overline{c} \cdot
\mu\bigl(B \bigl( x, \rho_x \bigl( (i - \sqrt{3\tau i})/s \bigr)\bigr)=\overline{c} V_d \rho_x^d\bigl( (i - \sqrt{3\tau i})/s\bigr).
\end{align}
Since $i\geq \lceil 48\tau \rceil \geq 48\tau$, we have $
\bigl(i - \sqrt{3\tau i}\bigr)/s \geq \bigl(i - \sqrt{3 i^2/48}\bigr)/s = 3i/(4s)$.
Combining this with inequality \eqref{equ::rxlower}, we get
$ \overline{c} V_d \rho_x^d\bigl( (i - \sqrt{3\tau i})/s\bigr) \geq 3i/(4s)$.
Therefore, we obtain 
\begin{align*}
\rho_x \bigl( (i - \sqrt{3\tau i})/s \bigr) 
\geq \bigl( 3 /(4V_d\overline{c} ) \bigr)^{1/d}(i/s)^{1/d}.
\end{align*}
Combining this with \eqref{equ::Rrho1-}, we have $R_{s,(i)}^b(x) \geq \rho_x\big((i-\sqrt{3\tau i})/s)  \bigr) \gtrsim (i/s)^{1/d}$ for all $x \in \mathcal{X}$ and fixed $i$ and $b$ with probability $\mathrm{P}^s$ at least $1-1/(2n^4)$. Therefore, by a union bound argument over $i$ and $b$, for all $i \geq c_n$, $x\in \mathcal{X}$, and $b\in [B]$, we have
\begin{align}\label{equ::rkxlower}
R_{s,(i)}^b(x)	
\geq \rho_x\big((i-\sqrt{3\tau i})/s)  \bigr)\gtrsim (i/s)^{1/d}
\end{align}
with probability $ \mathrm{P}^{Bs}$ at least $1 - 1/(2n^2)$. 

On the other hand, to obtain the upper bound for $R_{s,(i)}^b(x)$, we consider the set $\mathcal{B}_i^+ =  \big\{ B \bigl( x, \rho_x \bigl( (i + \sqrt{3\tau i})/s \bigr) \bigr) : x \in \mathcal{X} \bigr\}$. Using similar arguments, we can show that for all $c_n\leq i\leq s$, $x\in \mathcal{X}$, and $b\in [B]$, it holds that
\begin{align}\label{equ::rkxupper}
R_{s,(i)}^b(x) \leq \rho_x\bigl((i+\sqrt{3\tau i})/s\bigr)
\lesssim (i/s)^{1/d}
\end{align}
with probability $ \mathrm{P}^{Bs}$ at least $1 - 1/(2n^2)$.
Combining \eqref{equ::rkxlower} and   \eqref{equ::rkxupper}, we obtain that \eqref{equ::infrksuprk} holds for all $c_n\leq i\leq s$, $x\in \mathcal{X}$, and $b\in [B]$
with probability $ \mathrm{P}^{Bs}$ at least $1-1/n^2$.
Furthermore, \eqref{equ::rkxlower} and  \eqref{equ::rkxupper} implies 
\begin{align*}
(i-\sqrt{3\tau i})/s=\mathrm{P}\big(B(x,\rho_x((i-\sqrt{3\tau i})/s))\big)\leq \mathrm{P}(B(x,R_{s,(i)}^b(x)))\\
\leq \mathrm{P}\big(B(x,\rho_x((i+\sqrt{3\tau i})/s))\big)=(i+\sqrt{3\tau i})/s.
\end{align*}
Since $\sqrt{3\tau i}\lesssim \sqrt{i\log n}\lesssim i$ and $i-\sqrt{3\tau i}\geq i-\sqrt{3i^2/48}=  3i/4$ for $c_n\leq i\leq s$, we get
$\mathrm{P}(B(x,R_{s,(i)}^b(x)))\asymp i/s$.
This completes the proof of Lemma \ref{lem::Rrho}.
\end{proof}

The following lemma, which is needed in the proof of Lemma \ref{lem::i5lem}, shows that the probability $\mathrm{P}(B(x,R_{s,(i)}^b(x)))$ is a Lipschitz continuous function for fixed $i\in [s]$ and $b\in [B]$.
This Lemma is necessary to establish the uniform concentration inequality in Lemma \ref{lem::i5lem}. 

\begin{lemma}\label{lem:kdistlipschitz}
Suppose that Assumption \ref{asp::holder} holds. For $i \in [s]$ and $b \in [B]$, let $R_{s,(i)}^b(x)$ denote the $i$-distance of $x$ in the subset $D_s^b = \{X_1^b, \ldots, X_s^b\}$. Then for any $x, x' \in \mathcal{X}$ and all $b\in[B]$, we have $\bigl|\mathrm{P}(B(x,R_{s,(i)}^b(x))) - \mathrm{P}(B(x',R_{s,(i)}^b(x')))\bigr|\leq 2\overline{c}V_d \cdot 3^d d^{(d+1)/2} \|x-x'\|_2$, where $V_d$ denotes the volume of the $d$-dimensional closed unit ball.
\end{lemma}

\begin{proof}[of Lemma \ref{lem:kdistlipschitz}]
We first show that $R_{s,(i)}^b(x)$ is a Lipschitz continuous function of $x$.
Let $t=R_{s,(i)}^b(x)$. By the triangle inequality the $i$ nearest neighbors of $x$ are at distance at most $d(x, x')+t$ of $x'$. Therefore there are at least $i$ points at at most this distance from $x'$, hence the $i$-th nearest neighbor of $x'$ is also at distance less than $d(x, x')+t$, in other words,
\begin{align*}
R_{s,(i)}^b(x') \leq R_{s,(i)}^b(x)+d(x,x'). 
\end{align*}
By symmetry, we also have $R_{s,(i)}^b(x) \leq R_{s,(i)}^b(x')+d(x,x')$, implying that
$R_{s,(i)}^b(x)$ is a Lipschitz continuous function.

Next, we aim to show the continuity of $\mathrm{P}(B(x, R_{s,(i)}^b(x)))$. 
For any $x,x'\in\mathcal{X}$, given $\|f\|_{\infty} \leq \overline{c}$ from Assumption \ref{asp::holder}, we have: 
\begin{align}\label{equ::biglpbxrsibx}
\bigl|\mathrm{P}(B(x,R_{s,(i)}^b(x))) -\mathrm{P}(B(x',R_{s,(i)}^b(x')))\bigr|\leq \overline{c} \bigl|\mu(B(x,R_{s,(i)}^b(x)) \triangle B(x',R_{s,(i)}^b(x')))\bigr|,
\end{align}
where the notation $\triangle$ represents the symmetric difference between the two sets, i.e.~$A\triangle B:=(A\setminus B)\cup (B\setminus A)$, and $\mu$ denote the Lesbegue measure.
Using geometric considerations, the Hausdorff distance between $B(x, R_{s,(i)}^b(x))$ and $B(x', R_{s,(i)}^b(x'))$ can be bounded by $\delta := \|x - x'\|_2 + |R_{s,(i)}^b(x) - R_{s,(i)}^b(x')|$. 
Therefore, we have 
\begin{align*}
\mu(B(x,R_{s,(i)}^b(x)) \setminus B(x',R_{s,(i)}^b(x')))\leq V_d \bigl(\bigl(R_{s,(i)}^b(x)+\delta\bigr)^d - R_{s,(i)}^b(x)^d\bigr).
\end{align*}
Applying the Lagrange mean value theorem, there exists $\xi \in [R_{s,(i)}^b(x), R_{s,(i)}^b(x)+\delta]$ such that $
\bigl(R_{s,(i)}^b(x)+\delta\bigr)^d - R_{s,(i)}^b(x)^d\leq  d  \xi^{d-1}\delta$.
From the Lipschitz continuity of $R_{s,(i)}^b(x)$ and the assumption $\mathcal{X}=[0,1]^d$, we know that $\delta \leq 2 \|x - x'\|_2\leq 2d^{1/2}$ 
and $R_{s,(i)}^b(x) \leq d^{1/2}$. This implies $\xi\leq R_{s,(i)}^b(x)+\delta \leq 3d^{1/2}$.
Substituting this into the above inequalities, we get $
\mu(B(x,R_{s,(i)}^b(x)) \setminus B(x',R_{s,(i)}^b(x')))\leq V_d\cdot d\cdot (3d^{1/2})^{d-1} \cdot \delta$.
Since $\delta\leq 2\|x-x'\|_2$, this yields 
\begin{align*}
\mu(B(x,R_{s,(i)}^b(x)) \setminus B(x',R_{s,(i)}^b(x')))\leq V_d \cdot 3^d d^{(d+1)/2} \|x-x'\|_2.
\end{align*}
By symmetry, we can similarly bound $\mu(B(x', R_{s,(i)}^b(x')) \setminus B(x, R_{s,(i)}^b(x)))$. Combining these bounds, we obtain:
\begin{align*}
\bigl|\mu(B(x,R_{s,(i)}^b(x)) \triangle B(x',R_{s,(i)}^b(x')))\bigr|\leq 
2V_d \cdot 3^d d^{(d+1)/2} \|x-x'\|_2.
\end{align*}
Finally, substituting this into  \eqref{equ::biglpbxrsibx}, we obtain the desired assertion.
\end{proof}

By \eqref{equ::pbxribx}, the study of $\mathrm{P}(B(x,R_{s,(i)}^b(x)))$ reduces to analyzing $U_{(i)}^b$, where $U_{(i)}^b \sim \mathrm{Beta}(i, s+1-i)$. Consequently, term $(II)$ in \eqref{equ::i5} represents a sum of sample mean deviations, where the ``sample'' refers to
$\mathrm{P}(B(x,R_{s,(i)}^b(x)))^{1/d}$ across $b\in [B]$. and our objective is to establish finite-sample bounds for each $i$.

To achieve this, we leverage the sub-exponential property of the $d$-th root of the beta distribution. However, this problem is challenging due to the complexity of the integral in its moment-generating function. To the best of our knowledge, no existing results provide such bounds for this distribution.

To address this gap, we introduce a new approach to derive an upper bound for the moment-generating function. Specifically, we establish upper bounds on the $p$-th absolute central moment of the beta distribution for $p\geq 1$ in Lemma \ref{lem::betasub}, following \cite{skorski2023bernstein}. We then extend these results to its $d$-th root in Lemma \ref{lem::momentbetaroot}, leading to the sub-exponential property established in Lemma \ref{lem::rootbetasub}. This property serves as the foundation for our finite-sample bounds in Lemma \ref{lem::i5lem}.

\begin{lemma}\label{lem::betasub}
Let $s$ and $i$ be integers with $s\geq 2i$. Let $X\sim \mathrm{Beta}(i,s+1-i)$ and $\mu = \mathbb{E}X = i/(s+1)$. Define
\begin{align}\label{equ::vu}
u=\frac{2(s+1-2i)}{(s+1)(s+3)}, \quad v = \frac{i(s+1-i)}{(s+1)^2(s+2)}.
\end{align}
Then $X$ is sub-exponential with parameters $(2\sqrt{v},2u)$.
Furthermore, for all $p\geq 1$, we have $\mathbb{E}\bigl[|X-\mu|^p\bigr] \leq 4^{p+1}(p\sqrt{v})^p$.
\end{lemma}

\begin{proof}[of Lemma \ref{lem::betasub}]
We define the moment generating function   $\phi_X(\lambda):=\mathbb{E}[\exp(\lambda(X-\mathbb{E}X))]$ and its logarithm $\psi(\lambda):=\log \phi_X(\lambda)$. 
Since $s\geq 2i$, we have $s+1-i\geq i$. 
Therefore, for the Beta distribution $\mathrm{Beta}(i,s+1-i)$ with $v$ and $u$ as defined above, the first case in 
\citet{skorski2023bernstein}[Claim (43)] provides the bound:
\begin{align}\label{equ::psilog}
\psi(\lambda)\leq 
-\frac{v}{u^2} \bigl(u\lambda+\log(1-u\lambda)\bigr),\quad \ 0\leq \lambda<1/u.
\end{align}
Define $g(x)=\log(1-x)+x+2x^2$ for $0\leq x\leq 1/2$. Then we have $g'(x)=x(3-4x)/(1-x)>0$. Therefore, $g(x)$ is increasing, and thus $g(x)\geq g(0)=0$ for $0\leq x\leq 1/2$.
Hence, for $0\leq \lambda\leq 1/(2u)$, we have
$\log(1-u\lambda)\geq -u\lambda-2u^2\lambda^2$. Substituting this into \eqref{equ::psilog}, we get
\begin{align}\label{equ::psitbound}
\psi(\lambda)\leq 2v\lambda^2, \quad 0\leq \lambda\leq 1/(2u).
\end{align}

Next, we extend the bound to $\lambda<0$. Since $X\sim \mathrm{Beta}(i,s+1-i)$, it follows that $1-X \sim \mathrm{Beta}(s+1-i,i)$.
Applying the second case in \citet{skorski2023bernstein}[Claim (43)], 
we have 
\begin{align*}
\psi(\lambda)=\log 
\mathbb{E}[\exp\bigl((-\lambda)\cdot (1-X-\mathbb{E}[1-X])\bigr)]\leq v\lambda^2/2, \quad -1/(2u)\leq  \lambda<0. 
\end{align*}
Combining this with \eqref{equ::psitbound}, we conclude
\begin{align*}
\psi(\lambda) \leq 2v\lambda^2,\quad |\lambda|\leq 1/(2u).
\end{align*}
Therefore, $X$ is sub-exponential with parameters $(2\sqrt{v},2u)$. 

Next, we turn to bound the moments.
It is clear to see that for $x \geq 0$ and $a > 0$, we have $e^{ax} \geq ax$. For $x \geq \mu$ and $p \geq 1$, this implies
\begin{align*}
\frac{x-\mu}{4p\sqrt{v}} \leq e^{(x-\mu)/(4p\sqrt{v})} \leq \bigl(e^{(x-\mu)/(4\sqrt{v})} + e^{-(x-\mu)/(4\sqrt{v})}\bigr)^{1/p}.
\end{align*}
By symmetry, this yields
\begin{align*}
|x-\mu|^p \leq (4p\sqrt{v})^p \bigl(e^{(x-\mu)/(4\sqrt{v})} + e^{-(x-\mu)/(4\sqrt{v})}\bigr).
\end{align*}
Given $s \geq 2i$, we have
\begin{align*}
\frac{1}{4\sqrt{v}} = \frac{(s+1)\sqrt{s+2}}{4\sqrt{i(s+1-i)}} \leq \frac{s+1}{4} \cdot \sqrt{\frac{s+2}{s+1-i}} \leq \frac{(s+1)(s+2)}{4(s+1-i)} < \frac{(s+1)(s+3)}{4(s+1-2i)} = \frac{1}{2u}.
\end{align*}
Taking expectation with respect to $X$ and using the sub-exponential property, we obtain
\begin{align*}
\mathbb{E}[|X-\mu|^p] &\leq (4p\sqrt{v})^p \bigl(\mathbb{E}[e^{(X-\mu)/(4\sqrt{v})}] + \mathbb{E}[e^{-(X-\mu)/(4\sqrt{v})}]\bigr) \\
&\leq 2e^{1/8} (4p\sqrt{v})^p \leq 4^{p+1}(p\sqrt{v})^p.
\end{align*}
This completes the proof.
\end{proof}

\begin{lemma}\label{lem::momentbetaroot}
Let $i$, $s$, and $d$ be integers with $s \geq 2i$. Let $X \sim \mathrm{Beta}(i, s+1-i)$ and $\gamma_{s,i}$ be defined as in \eqref{equ::gammai}. Then for all $p \geq 1$, we have  
\begin{align*}
\mathbb{E}\bigl[|X^{1/d}-\gamma_{s,i}|^p\bigr] \leq 3\cdot(8p)^p \bigl(i^{-1/2+1/d}s^{-1/d}\bigr)^p. 
\end{align*}
\end{lemma}

\begin{proof}[Proof of Lemma \ref{lem::momentbetaroot}]
For all $x, y \in \mathbb{R}$ and $p \geq 1$, by Jensen's inequality, we have $(|x|+|y|)^p \leq 2^{p-1}(|x|^p + |y|^p)$. Therefore, 
\begin{align}\label{equ::rootmoment}
\mathbb{E}\bigl[|X^{1/d}-\gamma_{s,i}|^p\bigr]\leq 2^{p-1} \mathbb{E}\bigl[|X^{1/d}-(i/(s+1))^{1/d}|^p\bigr] + 
2^{p-1}|(i/(s+1))^{1/d}-\gamma_{s,i}|^p.
\end{align}
\textit{Bounding the First Term:}
By equality \eqref{equ::dequality}, we have
\begin{align*}
|X-i/(s+1)| & = |X^{1/d}-(i/(s+1))^{1/d}| \cdot \sum_{j=0}^{d-1} \Bigl(X^{j/d} (i/(s+1))^{(d-1-j)/d}\Bigr) \\
& \geq (i/(s+1))^{(d-1)/d} \cdot |X^{1/d}-(i/(s+1))^{1/d}|.
\end{align*}
Therefore, we get
\begin{align}\label{equ::upperx1dfirst1}
\mathbb{E}\bigl[|X^{1/d}-(i/(s+1))^{1/d}|^p\bigr] \leq (i/(s+1))^{-p(d-1)/d} \cdot \mathbb{E}\bigl[|X-(i/(s+1))|^p\bigr].
\end{align}
Lemma \ref{lem::betasub} implies that $\mathbb{E}\bigl[|X-i/(s+1)|^p\bigr] \leq 4^{p+1}(p\sqrt{v})^p$,
where $v$ is specified in \eqref{equ::vu}. Since $v < i/(s+1)^2$, combining this with \eqref{equ::upperx1dfirst1} gives
\begin{align*}
2^{p-1}\mathbb{E}\bigl[|X^{1/d}-(i/(s+1))^{1/d}|^p\bigr] & \leq  2^{p-1}(i/(s+1))^{-p(d-1)/d} \cdot  4^{p+1}p^p\cdot (i^{1/2}/(s+1))^p\\
& = 2\cdot (8p)^p \cdot\bigl(i^{-1/2+1/d}(s+1)^{-1/d}\bigr)^p.
\end{align*}
\textit{Bounding the Second Term:}
Using Gautschi's inequality \citep{gautschi1959some} for Gamma functions, we have
\begin{align*}
\biggl(i+\frac{1}{d}-1\biggr)^{1/d}<\frac{\Gamma(i+1/d)}{\Gamma(i)}<\biggl(i+\frac{1}{d}\biggr)^{1/d}
\end{align*}
and 
\begin{align*}
\biggl(s+ \frac{1}{d}\biggr)^{1/d} <\frac{\Gamma(s + 1+  1/d  )}{\Gamma(s+1)}< \biggl(s+1+ \frac{1}{d}\biggr)^{1/d}.
\end{align*}
By the definition of $\gamma_{s,i}$ in \eqref{equ::gammai}, we conclude that
\begin{align}\label{equ::gammasiupper}
\biggl(\frac{i+1/d-1}{s+1+1/d}\biggr)^{1/d} < \gamma_{s,i} < \biggl(\frac{i+1/d}{s+1/d}\biggr)^{1/d}.
\end{align}
Rewriting the left bound:
\begin{align}\label{equ::i1ds1}
\biggl(\frac{i+1/d-1}{s+1+1/d}\biggr)^{1/d} = \biggl(\frac{i}{s+1}\biggr)^{1/d} \cdot \biggl(1 + \frac{(1/d-1)/i-(1/d)/(s+1)}{1+(1/d)/(s+1)}\biggr)^{1/d}.
\end{align}
Since $s \geq 2i$, we have
\begin{align*}
\biggl|\frac{(1/d-1)/i-(1/d)/(s+1)}{1+(1/d)/(s+1)}\biggr|& \leq \biggl|\frac{1/d}{s+1}-\frac{1/d-1}{i}\biggr|\leq \frac{1/d}{2i}+\frac{1-1/d}{i}\nonumber \\
& = \frac{2-1/d}{2i}\leq 1-\frac{1}{2d}.
\end{align*}
Define $g(x) = (1+x)^{1/d}$ for $x \in [-1+1/(2d),1-1/(2d)]$. Then, its derivative satisfies $g'(x)= (1+x)^{1/d-1}/d \leq 2^{1-1/d}/d^{1/d}< 2$. By the mean value theorem, we have $|(1+x)^{1/d}-1| \leq 2|x|$ for $x \in [-1+1/(2d), 1-1/(2d)]$. This implies
\begin{align}\label{equ::bound1}
\biggl|\biggl(\frac{i+1/d-1}{s+1+1/d}\biggr)^{1/d} - \biggl(\frac{i}{s+1}\biggr)^{1/d}\biggr| \leq \frac{2(2-1/d)}{2i} \biggl(\frac{i}{s+1}\biggr)^{1/d}\leq \frac{2}{i} \biggl(\frac{i}{s+1}\biggr)^{1/d}.
\end{align}
Next, we rewrite the right bound in \eqref{equ::gammasiupper} as
\begin{align*}
\biggl(\frac{i+1/d}{s+1/d}\biggr)^{1/d} = \biggl(\frac{i}{s+1}\biggr)^{1/d}\biggl(1+\frac{(1/d)/i+(1-1/d)/(s+1)}{1-(1-1/d)/(s+1)}\biggr)^{1/d}.
\end{align*}
For $s \geq 2i \geq 2$, we derive $1-(1-1/d)/(s+1)\geq 1-1/(s+1)=s/(s+1)\geq 2/3$ and $(1/d)/i+(1-1/d)/(s+1)\geq 0$. Using the inequality $(1+x)^{1/d}\leq 1+x/d$ for $x\geq 0$, we obtain
\begin{align*}
\biggl(1+\frac{(1/d)/i+(1-1/d)/(s+1)}{1-(1-1/d)/(s+1)}\biggr)^{1/d}& \leq 1+ \frac{1}{d}\cdot\frac{(1/d)/i+(1-1/d)/(s+1)}{1-(1-1/d)/(s+1)}\\
& \leq 1 + \frac{3}{2d}\cdot \biggl(\frac{1/d}{i}+\frac{1-1/d}{2i}\biggr)\\
& = 1+\frac{3(1/d+1)\cdot (1/d)}{4i}\leq 1+\frac{3}{2i},
\end{align*}
where the second inequality follows from $1-(1-1/d)/(s+1)\geq 2/3$ and $s\geq 2i$ and the last inequality follows from $x(1+x)\leq 2$ for $x\in [0,1]$.
Therefore, we obtain
\begin{align*}
\biggl|\biggl(\frac{i+1/d}{s+1/d}\biggr)^{1/d} - \biggl(\frac{i}{s+1}\biggr)^{1/d}\biggr| \leq \frac{2}{i} \biggl(\frac{i}{s+1}\biggr)^{1/d}.
\end{align*}
Combining this with \eqref{equ::bound1}, we conclude that
\begin{align*}
|(i/(s+1))^{1/d}-\gamma_{s,i}| \leq 2i^{-1+1/d}(s+1)^{-1/d}\leq  2i^{-1/2+1/d}(s+1)^{-1/d}.
\end{align*}
Therefore, we have
\begin{align*}
2^{p-1}|(i/(s+1))^{1/d}-\gamma_{s,i}|^p \leq (8p)^p \bigl(i^{-1/2+1/d}(s+1)^{-1/d}\bigr)^p.
\end{align*}
Finally, combining the bounds for both terms and using \eqref{equ::rootmoment}, we obtain
\begin{align*}
\mathbb{E}\bigl[|X^{1/d}-\gamma_{s,i}|^p\bigr] \leq 3\cdot(8p)^p \bigl(i^{-1/2+1/d}s^{-1/d}\bigr)^p.
\end{align*}
This concludes the proof.
\end{proof}

\begin{lemma}\label{lem::rootbetasub}
Let $s$ and $i$ be integers with $s\geq 2i$. Let $X\sim \mathrm{Beta}(i,s+1-i)$. Then $X^{1/d}$ is sub-exponential with parameters $(16e\sqrt{3}i^{-1/2+1/d}s^{-1/d},16ei^{-1/2+1/d}s^{-1/d})$.
\end{lemma}

\begin{proof}[of Lemma \ref{lem::rootbetasub}]
Define the moment generating function (MGF) as 
\begin{align*}\phi(\lambda):=\mathbb{E}[\exp(\lambda(X^{1/d}-\gamma_{s,i}))]
\end{align*}
and its logarithm $\psi(\lambda):=\log \phi(\lambda)$, where $\gamma_{s,i}$ is defined in \eqref{equ::gammai}. 
To verify the sub-exponential property, we evaluate the MGF.

Using the Taylor expansion of the exponential function, we write
\begin{align*}
\mathbb{E}\big[\exp\big(\lambda (X^{1/d}-\gamma_{s,i})\big)\big] = \mathbb{E}\Bigg[1 + \lambda  (X^{1/d}-\gamma_{s,i})+ \sum_{p=2}^\infty \frac{\lambda^p  (X^{1/d}-\gamma_{s,i})^p}{p!}\Bigg].
\end{align*}
Since $\mathbb{E}[X^{1/d}-\gamma_{s,i}]=0$, the linear term vanishes, and we bound the higher-order terms using Lemma \ref{lem::momentbetaroot} as follows 
\begin{align*}
\mathbb{E}\big[\exp\big(\lambda (X^{1/d}-\gamma_{s,i})\big)\big] \leq  1+  \sum_{p=2}^\infty \frac{|\lambda|^p  \mathbb{E}\bigl[\bigl|X^{1/d}-\gamma_{s,i}\bigr|^p\bigr]}{p!}\leq  1+  \sum_{p=2}^\infty \frac{3(8i^{-1/2+1/d}s^{-1/d}|\lambda| )^p  p^p}{p!}.
\end{align*}
Using  Stirling’s approximation $p!\geq (p/e)^p$ for $p\geq 1$, we obtain
\begin{align*}
\mathbb{E}\big[\exp\big(\lambda (X^{1/d}-\gamma_{s,i})\big)\big]\leq  1+  3\sum_{p=2}^\infty (8e i^{-1/2+1/d}s^{-1/d}|\lambda|)^p 
\end{align*}
The geometric series sums to
\begin{align*}
\sum_{p=2}^\infty (8e i^{-1/2+1/d}s^{-1/d}|\lambda|)^p = \frac{64e^2i^{-1+2/d}s^{-2/d}\lambda^2 }{1 -8ei^{-1/2+1/d}s^{-1/d}|\lambda|}.
\end{align*}
For $|\lambda| \leq i^{1/2-1/d}s^{1/d}/(16e)$, we have $8ei^{-1/2+1/d}s^{-1/d}|\lambda| \leq 1/2$, ensuring convergence. Substituting this bound, we get
\begin{align*}
\mathbb{E}\big[\exp\big(\lambda (X^{1/d}-\gamma_{s,i})\big)\big] \leq 1+384e^2i^{-1+2/d}s^{-2/d}\lambda^2 .
\end{align*}
Since $1+x\leq e^x$ for $x\in \mathbb{R}$, we get
\begin{align*}
1+384e^2 i^{-1+2/d}s^{-2/d}\lambda^2\leq \exp(384e^2 i^{-1+2/d}s^{-2/d}\lambda^2).
\end{align*}
Therefore, we have 
\begin{align*}
\mathbb{E}\big[\exp\big(\lambda (X^{1/d}-\gamma_{s,i})\big)\big] \leq \exp(384e^2 i^{-1+2/d}s^{-2/d}\lambda^2), \quad |\lambda| \leq i^{1/2-1/d}s^{1/d}/(16e).
\end{align*}
This establishes that $X^{1/d}$ is sub-exponential with parameters $\zeta=16e\sqrt{3}i^{-1/2+1/d}s^{-1/d}$ and $\alpha = 16ei^{-1/2+1/d}s^{-1/d}$.
This completes the proof. 
\end{proof}

\begin{lemma}\label{lem::i5lem}
Let Assumption \ref{asp::holder} hold. Additionally, let $\{D_s^b\}_{b=1}^B$ be $B$ disjoint subsets of size $s$ randomly drawn from data $D_n$.
Let $R_{s,(i)}^b(x)$ be the $i$-distance of $x$ in the subset $D_s^b$. Define $\overline{k}$ as in Proposition \ref{pro::I456sur}.
Suppose that
$B \geq 2(d^2+4)(\log n)/3$, $s\geq 2\overline{k}$.
Furthermore, assume that there exists constants $C_{n,i}'$ such that $w_i^b\leq C_{n,i}'$ for all $b\in [B]$ and $i\in [s]$.
Then, there exists $n_1:=2d^d+1$, such that for all $x \in \mathcal{X}$, $1\leq i\leq \overline{k}$, and all $n > n_1$, the following statement holds with probability $ \mathrm{P}^{Bs}$ at least $1 - 1/n^2$:
\begin{align}\label{equ::lemma152}
\Biggl|\frac{1}{B} \sum_{b=1}^B w^b_i \bigl(\mathrm{P}(B(x,R_{s,(i)}^b(x)))^{1/d}-\gamma_{s,i}\bigr)\Biggr|\lesssim C_{n,i}' \biggl(\frac{i}{s}\biggr)^{1/d} \sqrt{\frac{\log n}{i B}}.
\end{align}
\end{lemma}

\begin{proof}[of Lemma \ref{lem::i5lem}]
We first prove \eqref{equ::lemma152} holds for a  fixed $x\in \mathcal{X}$ and a fixed $1\leq i\leq\overline{k}$.
For any $b\in [B]$, by \eqref{equ::pbxribx}, we have
\begin{align*}
\bigl(\mathrm{P}(B(x,R_{s,(1)}^b(x))),\ldots,\mathrm{P}(B(x,R_{s,(s)}^b(x)))\bigr)\overset{\mathcal{D}}{=}\big(U_{(1)}^b,\ldots,U_{(s)}^b\big),
\end{align*}
where $U_{(i)}^b$ is the $i$-th order statistic of $\mathrm{i.i.d.}$~uniform $[0,1]$ random variables. By \citet[Corollary 1.2]{biaulecture},
$U_{(i)}^b \sim \mathrm{Beta}(i,s+1-i)$.
Let $\xi_b:=\mathrm{P}(B(x,R_{s,(i)}^b(x)))$.
Since $\{D_s^b\}_{b=1}^B$ are independent, 
$\{\xi_b\}_{b=1}^B$ are independent random variables following $\mathrm{Beta}(i,s+1-i)$. The desired inequality \eqref{equ::lemma152} then reduces to the concentration inequality for $w^b_i\xi_b^{1/d}$.
To prove this, we begin by showing that $\frac{1}{B}\sum_{b=1}^B w^b_i \bigl(\xi_b^{1/d}-\gamma_{s,i}\bigr)$ is a sub-exponential random variables. Since $\xi_b$ are independent, we have
\begin{align*}
\mathbb{E} \bigg[\exp\biggl(\lambda\bigg(\frac{1}{B}\sum_{b=1}^B w_i^b \big(\xi_b^{1/d}-\gamma_{s,i}\big)\biggr)\biggr)\bigg] = \prod_{b=1}^B \mathbb{E}\big[\exp\bigl(\lambda w_i^b \bigl(\xi_b^{1/d}-\gamma_{s,i}\bigr)/B\bigr)\bigr]. 
\end{align*}
Since $s\geq 2\overline{k}$, it follows that $s-i+1\geq i$ for $1\leq i\leq \overline{k}$.
Given the condition $w_i^b\leq C_{n,i}'$, we have $|\lambda w_i^b/B|\leq i^{1/2-1/d}s^{1/d}/(16e)$ for all $|\lambda |\leq i^{1/2-1/d}s^{1/d}B/(16eC_{n,i}')$.
By leveraging the sub-exponential property of $\xi_b^{1/d}$ stated in Lemma \ref{lem::rootbetasub}, we obtain
\begin{align*}
\mathbb{E} \bigg[\exp\biggl(\lambda\bigg(\frac{1}{B}\sum_{b=1}^B w_i^b \big(\xi_b^{1/d}-\gamma_{s,i}\big)\bigg)\biggr)\bigg] \leq \exp \biggl(\frac{384 e^2 C_{n,i}'^2 i^{-1+2/d}\lambda^2}{Bs^{2/d}}\biggr), \quad  |\lambda|\leq \frac{i^{1/2-1/d}s^{1/d}B}{16eC_{n,i}'}.
\end{align*} 
This shows $\frac{1}{B}\sum_{b=1}^B w_i^b \big(\xi_b^{1/d}-\gamma_{s,i}\big)$ is sub-exponential. 
Applying the sub-exponential tail bound in Lemma \ref{lem::subtailbound}, we have
\begin{align*}
 \mathrm{P}^{Bs}\biggl(\Biggl|\frac{1}{B} \sum_{b=1}^B w^b_i \bigl(\xi_b^{1/d}-\gamma_{s,i}\bigr)\Biggr|\geq 16e\sqrt{6}C_{n,i}' \biggl(\frac{i}{s}\biggr)^{1/d} \sqrt{\frac{\tau}{i B}} \biggr)\leq 2e^{-\tau}.
\end{align*}
for all $0\leq \tau\leq 3B/2$. Since $B\geq 2(d^2+4)(\log n)/3$, it follows that $3B/2\geq (d^2+4)\log n$. Taking $\tau:=(d^2+4)\log n$ and replacing $
\xi_b$ with $\mathrm{P}(B(x,R_{s,(i)}^b(x)))$, for a fixed $x\in \mathcal{X}$ and a fixed $1\leq i\leq \overline{k}$, we have 
\begin{align}\label{equ::xibupper}
 \mathrm{P}^{Bs}\biggl(\Biggl|\frac{1}{B} \sum_{b=1}^B w^b_i \bigl(\mathrm{P}(B(x,R_{s,(i)}^b(x)))^{1/d}-\gamma_{s,i}\bigr)\Biggr|\lesssim C_{n,i}' \biggl(\frac{i}{s}\biggr)^{1/d} \sqrt{\frac{\log n}{i B}}\biggr)\geq 1-\frac{2}{n^{d^2+4}}.
\end{align}

To extend the upper bound to all $x\in \mathcal{X}$, consider a $1/n^d$-net $\{z_j\}_{j=1}^J$ of $[0,1]^d$.
A $1/n^d$-net is a finite subset of $\mathcal{X}$ such that for any $x \in \mathcal{X}$, there exists $z_j$ in the net with $\|x - z_j\|_2 \leq 1/n^d$.  
The construction of such a net can be done by placing grid points spaced $1/(dn^d)$ apart in each of the $d$ dimensions. This results in at most $dn^d$ grid points per dimension, and thus the total number of grid points satisfies $J \leq (dn^d)^d\leq d^d n^{d^2}$.  
By \eqref{equ::xibupper}, the bound holds for each $z_j$ with $j\in [J]$,
and using the union bound, it holds for all $j\in [J]$ with probability $ \mathrm{P}^{Bs}$ at least $1-2d^d/n^4$. 

Since $\{z_j\}_{j=1}^J$ is a $1/n^d$-net, for any $x\in \mathcal{X}$, there exists $z_j$ such that $\|x-z_j\|_2\leq 1/n^d$. By Lemma \ref{lem:kdistlipschitz}, we have $\bigl|\mathrm{P}(B(x,R_{s,(i)}^b(x))) - \mathrm{P}(B(z_j,R_{s,(i)}^b(z_j)))\bigr|\lesssim 1/n^d$ for all $b\in [B]$.
Using the Minkowski inequality, this implies 
\begin{align*}
&\bigl|\mathrm{P}(B(x,R_{s,(i)}^b(x)))^{1/d} - \mathrm{P}(B(z_j,R_{s,(i)}^b(z_j)))^{1/d}\bigr|\\
&\leq \bigl|\mathrm{P}(B(x,R_{s,(i)}^b(x))) - \mathrm{P}(B(z_j,R_{s,(i)}^b(z_j)))\bigr|^{1/d}\lesssim 1/n.
\end{align*}
Combining this with the triangle inequality and the error bound for $z_j$, we obtain
\begin{align*}
& \Biggl|\frac{1}{B}\sum_{b=1}^B w^b_i \bigl(\mathrm{P}(B(x,R_{s,(i)}^b(x)))^{1/d}-\gamma_{s,i}\bigr)\Biggr|\\
& \leq \Biggl|\frac{1}{B}\sum_{b=1}^B w^b_i \bigl(\mathrm{P}(B(x,R_{s,(i)}^b(x)))^{1/d}-\bigl(\mathrm{P}(B(z_j,R_{s,(i)}^b(z_j)))^{1/d}\bigr)\Biggr|
\\
& \phantom{=}+ \Biggl|\frac{1}{B}\sum_{b=1}^B w^b_i \bigl(\mathrm{P}(B(z_j,R_{s,(i)}^b(z_j)))^{1/d}-\gamma_{s,i}\bigr)\Biggr|\\
& \lesssim C_{n,i}' \biggl(\frac{i}{s}\biggr)^{1/d} \sqrt{\frac{\log n}{i B}}+\frac{C_{n,i}'}{n}
\end{align*}
for all $x\in \mathcal{X}$ and a fixed $i$
with probability $ \mathrm{P}^{Bs}$ at least $1-2d^d/n^4$.
By applying the union bound over $i$, the same holds for all $x\in \mathcal{X}$ and $1\leq i\leq \overline{k}$ with probability $ \mathrm{P}^{Bs}$ at least $1-2d^d/n^3$.
For $n>n_1:=2d^d+1$, we note that $i\log n\geq 1$ for all $1\leq i\leq \overline{k}$. This implies
\begin{align*}
\biggl(\frac{i}{s}\biggr)^{1/d} \sqrt{\frac{\log n}{i B}} = \biggl(\frac{s}{i}\biggr)^{1-1/d}\sqrt{\frac{i\log n}{Bs^2}}\geq \sqrt{\frac{i\log n}{Bs^2}}\geq \sqrt{\frac{1}{ns}}\geq \frac{1}{n}.
\end{align*}
Combining this with the previous inequality, we conclude
\begin{align*}
\Biggl|\frac{1}{B}\sum_{b=1}^B w^b_i \bigl(\mathrm{P}(B(x,R_{s,(i)}^b(x)))^{1/d}-\gamma_{s,i}\bigr)\Biggr|\lesssim C_{n,i}' \biggl(\frac{i}{s}\biggr)^{1/d} \sqrt{\frac{\log n}{i B}}
\end{align*}
for all $x\in \mathcal{X}$, $1\leq i\leq \overline{k}$, and $n>n_1$,
with probability $ \mathrm{P}^{Bs}$ at least $1-2d^d/n^3\geq 1-1/n^2$. This completes the proof. 
\end{proof}

\begin{proof}[of Proposition \ref{pro::I456}]
Let $\Omega_1$ denote the event defined by \eqref{equ::infrksuprk} and \eqref{equ::pbxrkxkn} in Lemma \ref{lem::Rrho}, and let $\Omega_2$ denote the event defined by \eqref{equ::lemma152} in Lemma \ref{lem::i5lem}. By applying the union bound, the event $\Omega_1 \cap \Omega_2$ holds with probability $ \mathrm{P}^{Bs}$ at least $1 - 2/n^2$ for all $n > n_1$, where $n_1 = 2d^d + 1$.

The condition
$\sum_{i=1}^s w_i^b i ^{1/d} \asymp \bigl( k^b \bigr)^{1/d}$ in $(iii)$ implies the existence of a constant $c_4'>0$ such that $\sum_{i=1}^s w_i^b (i/s)^{1/d} \geq c_4' (k^b/s)^{1/d}$ for all $b\in [B]$.
On the other hand, the condition $\sum_{i=1}^{c_n} w_i^b\lesssim (\log n)/k^b$ in $(i)$ and the bound $\underline{k} \gtrsim (\log n)^2$ together imply the existence of $n_2 \in \mathbb{N}$ such that for all $n\geq n_2$, we have $\sum_{i=1}^{c_n} w_i^b\leq c_4'/2$ and $k^b\geq c_n$ for all $b\in [B]$, where $c_n$ is the sequence specified in Lemma \ref{lem::Rrho}.
Hence, we consider the subsequent arguments under the assumptions that the event $\Omega_1 \cap \Omega_2$ holds and that $n > N_1 := \max\{n_1, n_2\}$.

\textit{Proof of Bounding $(I)$}.
Let $(IV)$ and $(V)$ be defined as follows: 
\begin{align*}
(IV): = \sum_{j=0}^{d-1} \biggl( \frac{1}{B} \sum_{b=1}^B \sum_{i=1}^s w_i^b (i/s)^{1/d} \biggr)^j \bigl(V_d^{1/d} f(x)^{1/d} R_n^B(x) \bigr)^{d-1-j} 
\,
\text{ and } 
\,
(V): = V_d R_n^B(x)^d.
\end{align*}
By the equality \eqref{equ::i4}, in order to derive the upper bound of $(I)$, it suffices to derive the upper bound of $(IV)$ and the lower bound of $(V)$.

Let us first consider $(IV)$. 
By condition \textit{(iii)}, we have $
\sum_{i=1}^s w_i^b (i/s)^{1/d} 
\asymp \bigl( k^b / s \bigr)^{1/d}
\lesssim \bigl( \overline{k} / s \bigr)^{1/d}$ for $b\in [B]$.
Consequently we get
\begin{align}\label{equ::bagsumwiin1d}
\biggl( \frac{1}{B} \sum_{b=1}^B \sum_{i=1}^s w_i^b  (i/s)^{1/d} \biggr)^j
\lesssim \bigl( \overline{k} / s \bigr)^{j/d}.
\end{align}
On the event $\Omega_1$, we have
\begin{align*}
R_s^{w,b}(x) 
& = \sum_{i=1}^s w_i^b R_{s,(i)}^b(x) 
= \sum_{i=1}^{c_n} w_i^b R_{s,(i)}^b(x) + \sum_{i=c_n+1}^s w_i^b R_{s,(i)}^b(x)
\\
& \leq R_{s,(c_n)}^b(x) + \sum_{i=c_n+1}^s w_i^b R_{s,(i)}^b(x)  
\lesssim (c_n /s)^{1/d} + \sum_{i=1}^{s} w_i^b (i/s)^{1/d} 
\lesssim \bigl( \overline{k} / s \bigr)^{1/d}
\end{align*}
for all $x\in \mathcal{X}$,
where the last inequality follows from $\sum_{i=1}^s w_i^b i^{1/d}\lesssim \overline{k}^{1/d}$ in condition $(iii)$ and $\underline{k}\gtrsim (\log n)^2$ in condition $(i)$.
Consequently we obtain 
\begin{align} \label{equ::rlambdabx}
R_n^B(x)
= \frac{1}{B} \sum_{b=1}^B R_s^{w,b}(x)
\lesssim \bigl( \overline{k} / s \bigr)^{1/d}
\end{align}
for all $x\in \mathcal{X}$.
Combining this with \eqref{equ::bagsumwiin1d} and $\|f\|_{\infty} \leq \overline{c}$ from Assumption \ref{asp::holder}, we derive
\begin{align}\label{equ::bagivupper}
(IV) 
\lesssim \sum_{j=0}^{d-1} \overline{c}^{(d-1-j)/d}\cdot \bigl( \overline{k} / s \bigr)^{j/d} \cdot \bigl( \overline{k} / s \bigr)^{(d-1-j)/d}
\lesssim \bigl( \overline{k} / s \bigr)^{(d-1)/d}.
\end{align}

Next, let us consider $(V)$. 
On the event $\Omega_1$, for any $b\in [B]$, we have
\begin{align*}
R_s^{w,b}(x)
\geq \sum_{i=c_n+1}^s w_i^b R_{s,(i)}^b(x)
& \gtrsim \sum_{i=c_n+1}^s w_i^b (i/s)^{1/d}
= \sum_{i=1}^s w_i^b (i/s)^{1/d} - \sum_{i=1}^{c_n} w_i^b  (i/s)^{1/d}
\end{align*}
for all $x\in \mathcal{X}$.
From earlier arguments in the second paragraph of this proof, we have $\sum_{i=1}^s w_i^b (i/s)^{1/d}\geq c_4' (k^b/s)^{1/d}$ and $\sum_{i=1}^{c_n} w_i^b(i/s)^{1/d}\leq (c_n/s)^{1/d} \sum_{i=1}^{c_n} w_i^b\leq c_4'(k^b/s)^{1/d}/2$ for all $n>N_1$.
Therefore, we have 
\begin{align}\label{equ::rsblower}
R_s^{w,b}(x)\gtrsim \sum_{i=1}^s w_i^b (i/s)^{1/d} - \sum_{i=1}^{c_n} w_i^b  (i/s)^{1/d} \geq c_4'\bigl( k^b / s \bigr)^{1/d}/2
\end{align}
for all $x\in \mathcal{X}$.
This implies  
\begin{align}\label{equ::rnbxlower}
R_n^B(x)
= \frac{1}{B} \sum_{b=1}^B R_s^{w,b}(x)
\gtrsim (\underline{k} / s)^{1/d}
\end{align}
for all $x\in \mathcal{X}$.
Therefore, we get
$(V) = V_d R_n^B(x)^d 
\gtrsim \underline{k}/s$.
Combining with \eqref{equ::bagivupper} and $\overline{k}\asymp \underline{k}$ in condition \textit{(ii)}, we conclude that
$(I) = (IV) / (V)
\lesssim \bigl( \overline{k} / s \bigr)^{-1/d}$.
This completes the proof of bounding $(I)$.

\textit{Proof of Bounding $(II)$}. 
Since $w_i^b=0$ for all $i>\overline{k}$ and $b\in [B]$, it follows that
\begin{align*}
\sum_{i=1}^{s}\Biggl|\frac{1}{B} \sum_{b=1}^B w^b_i \bigl(\mathrm{P}(B(x,R_{s,(i)}^b(x))^{1/d}-\gamma_{s,i}\bigr)\Biggr|  = \sum_{i=1}^{\overline{k}}\Biggl|\frac{1}{B} \sum_{b=1}^B w^b_i \bigl(\mathrm{P}(B(x,R_{s,(i)}^b(x))^{1/d}-\gamma_{s,i}\bigr)\Biggr| .
\end{align*}
By applying \eqref{equ::lemma152} on the event $\Omega_2$, for all $x \in \mathcal{X}$, we have
\begin{align*}
\sum_{i=1}^{\overline{k}}\Biggl|\frac{1}{B} \sum_{b=1}^B w^b_i \bigl(\mathrm{P}(B(x,R_{s,(i)}^b(x)))^{1/d}-\gamma_{s,i}\bigr)\Biggr|
\lesssim \sum_{i=1}^{\overline{k}} C_{n,i} \biggl(\frac{i}{s}\biggr)^{1/d} \sqrt{\frac{\log n}{i B}}.
\end{align*}
Given $\sum_{i=1}^s C_{n,i} i^{1/d-1/2} \lesssim \bigl(\overline{k}\bigr)^{1/d-1/2}$ in condition \textit{(iv)}, we obtain 
\begin{align*}
\sum_{i=1}^{s} \Biggl|\frac{1}{B} \sum_{b=1}^B w^b_i \bigl( \mathrm{P}(B(x,R_{s,(i)}^b(x)))^{1/d} -\gamma_{s,i} \bigr)\Biggr|
& \lesssim \bigl( \overline{k} / s \bigr)^{1/d} \bigl( (\log n) /  (\overline{k} B) \bigr)^{1/2}.
\end{align*}
This completes the proof of bounding $(II)$.

\textit{Proof of Bounding $(III)$}.
Let $b\in [B]$ be fixed for now. We analyze two cases separately: 
$1\leq i<c_n$ and $c_n\leq i\leq k^b$.

We begin with the case $1\leq i<c_n$. On the event $\Omega_1$, we have $R_{s,(i)}^b(x)\leq R_{s,(c_n)}^b(x)\lesssim (c_n/s)^{1/d}$ and $\mathrm{P}(B(x,R_{s,(i)}^b(x)))\leq \mathrm{P}(B(x,R_{s,(c_n)}^b(x)))\lesssim c_n/s$ for all $x\in \mathcal{X}$ by using \eqref{equ::infrksuprk} and \eqref{equ::pbxrkxkn}.
These bounds imply
\begin{align*}
\bigl| \mathrm{P}(B(x,R_{s,(i)}^b(x)))^{1/d} -  V_d^{1/d} f(x)^{1/d} R_{s,(i)}^b(x) \bigr|\lesssim (c_n/s)^{1/d}\lesssim ((\log n)/s)^{1/d}.
\end{align*}
Therefore, we have 
\begin{align*}
\sum_{i=1}^{c_n-1} w_i^b \bigl| \mathrm{P}(B(x,R_{s,(i)}^b(x)))^{1/d} -  V_d^{1/d} f(x)^{1/d} R_{s,(i)}^b(x) \bigr|
\lesssim\biggl(\frac{\log n}{s}\biggr)^{1/d} \sum_{i=1}^{c_n} w_i^b .
\end{align*}
Using $\sum_{i=1}^{c_n} w_i^b \lesssim (\log n)/k^b$ from condition \textit{(i)} and $\underline{k}\asymp \overline{k}$ from condition $(ii)$, we obtain
\begin{align}\label{equ::sumicswbi}
\sum_{i=1}^{c_n-1} w_i^b \bigl|\mathrm{P}(B(x,R_{s,(i)}^b(x)))^{1/d}- V_d^{1/d} f(x)^{1/d} R_{s,(i)}^b(x)\bigr|
\lesssim \frac{(\log n)^{1+1/d}}{s^{1/d}\overline{k}}
\end{align}
for all $x\in \mathcal{X}$.

Now, we consider the case $c_n\leq i\leq k^b$.
Using \eqref{equ::dequality} and the condition $\|f\|_{\infty}\geq \underline{c}$ from Assumption \ref{asp::holder}, we get 
\begin{align}\label{equ::iii1d}
&\bigl| \mathrm{P}(B(x,R_{s,(i)}^b(x)))^{1/d} - V_d^{1/d} f(x)^{1/d} R_{s,(i)}^b(x) \bigr|\nonumber\\
&\lesssim \frac{\bigl| \mathrm{P}(B(x,R_{s,(i)}^b(x))) - V_d f(x) R_{s,(i)}^b(x)^d \bigr|}{\sum_{j=0}^{d-1} \bigl(\mathrm{P}(B(x,R_{s,(i)}^b(x)))^{j/d} R_{s,(i)}^b(x)^{d-1-j}\bigr)},
\end{align}
for all $x\in \mathcal{X}$. 
Next, consider $x$ such that $B(x,R_{s,(k^b)}^b(x))\subset [0,1]^d$ for all $b\in [B]$.
Using the Lipschitz smoothness from Assumption \ref{asp::holder}, we obtain
\begin{align*}
& \bigl| \mathrm{P}(B(x,R_{s,(i)}^b(x))) - V_d f(x) R_{s,(i)}^b(x)^d \bigr| \\
& = \biggl| \int_{B(x,R_{s,(i)}^b(x))} f(y) \, dy  - \int_{B(x,R_{s,(i)}^b(x))} f(x) \, dy \biggr|
\leq \int_{B(x,R_{s,(i)}^b(x))} |f(y) - f(x)| \, dy\\
& \leq c_L \int_{B(x,R_{s,(i)}^b(x))} \|y - x\|_2 \, dy
\lesssim R_{s,(i)}^b(x)^{d+1}.
\end{align*}
On the event $\Omega_1$, we have $R_{s,(i)}^b(x)\lesssim (i/s)^{1/d}$ for $c_n\leq i\leq k^b$.
This implies
\begin{align}\label{equ::pbxrsix}
\bigl| \mathrm{P}(B(x,R_{s,(i)}^b(x))) - V_d f(x) R_{s,(i)}^b(x)^d \bigr| \lesssim (i/s)^{1+1/d}.
\end{align}
Moreover, using $\|f\|_{\infty}\geq \underline{c}$ in Assumption \ref{asp::holder} and $R_{s,(i)}^b(x)\asymp (i/s)^{1/d}$ on the event $\Omega_1$, we have $\mathrm{P}(B(x,R_{s,(i)}^b(x)))\gtrsim i/s$ for $c_n\leq i\leq k^b$. 
Consequently, we obtain 
\begin{align*}
\sum_{j=0}^{d-1} (i/s)^{j/d} \bigl(R_{s,(i)}^b(x)\bigr)^{d-1-j}
\gtrsim \sum_{j=0}^{d-1} (i/s)^{j/d}\cdot(i/s)^{(d-1-j)/d}
\gtrsim (i/s)^{(d-1)/d}.
\end{align*}
This together with \eqref{equ::iii1d} and \eqref{equ::pbxrsix} implies 
\begin{align}\label{equ::pbxrsix1d}
\bigl|\mathrm{P}(B(x,R_{s,(i)}^b(x)))^{1/d}-V_d^{1/d} f(x)^{1/d} R_{s,(i)}^b(x)\bigr|
\lesssim (i/s)^{2/d}
\end{align}
for $c_n\leq i\leq k^b$.
Therefore, using $\sum_{i=1}^s w_i^b i^{1/d} \asymp \bigl(k^b \bigr)^{1/d}$ in condition  $(iii)$, we have 
\begin{align*}
& \sum_{i=c_n}^{k^b} w_i^b \bigl| (\mathrm{P}(B(x,R_{s,(i)}^b(x))))^{1/d}-V_d^{1/d} f(x)^{1/d} R_{s,(i)}^b(x) \bigr|\\
& \lesssim \sum_{i=c_n}^{k^b} w_i^b (i/s)^{2/d}\lesssim (\overline{k}/s)^{1/d} \sum_{i=1}^{s} w_i^b (i/s)^{1/d}\lesssim (\overline{k}/s)^{1/d} \cdot (k^b/s)^{1/d} \leq (\overline{k}/s)^{2/d}.
\end{align*}
Combining this with \eqref{equ::sumicswbi}, we obtain
\begin{align*}
\sum_{i=1}^{k^b} w_i^b \bigl| \mathrm{P}(B(x,R_{s,(i)}^b(x)))^{1/d}-V_d^{1/d} f(x)^{1/d} R_{s,(i)}^b(x) \bigr|\lesssim \frac{(\log n)^{1+1/d}}{s^{1/d}\overline{k}}+(\overline{k}/s)^{2/d}.
\end{align*}
for all $x\in \mathcal{X}$ and a fixed $b\in [B]$. Averaging over $b\in [B]$, we have 
\begin{align*}
\frac{1}{B} \sum_{b=1}^B \sum_{i=1}^{k^b} w_i^b \bigl| \mathrm{P}(B(x,R_{s,(i)}^b(x)))^{1/d} - V_d^{1/d} f(x)^{1/d} R_{s,(i)}^b(x) \bigr|
\lesssim \frac{(\log n)^{1+1/d}}{s^{1/d}\overline{k}}+(\overline{k}/s)^{2/d}.
\end{align*}
This completes the proof of bounding $(III)$, and hence the proof of Proposition \ref{pro::I456}.
\end{proof}

\subsubsection{Proofs Related to Section \ref{sec::datadrivennnad}}
\label{sec::proofssurrogaterisk}

To prove Proposition \ref{pro::I456sur}, we require the following lemma, which establishes an upper bound on the number of instances near the boundary.

\begin{lemma}\label{lem::samplenear}
let $\{D_s^b\}_{b=1}^B$ be $B$ disjoint subsets of size $s$ randomly drawn from the data $D_n$, with $D_s^b=\{X_1^b,\ldots,X_s^b\}$ for $b\in [B]$.
Define $\Delta_n := [c_3'(\overline{k}/s)^{1/d}, 1-c_3'(\overline{k}/s)^{1/d}]^d$, where $c_3'$ is the constant from Lemma \ref{lem::Rrho} and $\overline{k}$ is as defined in Proposition \ref{pro::I456sur}.
Assume that the condition $\underline{k}\gtrsim (\log n)^2$ holds.
Define $\mathcal{I}^b:=\{i\in [s]: X_i^b\in \Delta_n\}$ and $n^b:=|\mathcal{I}^b|$.
Then, for all $n>N_1$ with $N_1$ specified in Proposition \ref{pro::I456}, there holds $1 - n^b/s \lesssim \bigl(\overline{k}/s\bigr)^{1/d}$ for all $b\in [B]$ with probability $ \mathrm{P}^{Bs}$ at least $1-1/n^2$.
\end{lemma}

\begin{proof}[of Lemma \ref{lem::samplenear}]
Let $\Delta_n^c := [0,1]^d \setminus \Delta_n$.
Let $b\in [B]$ be fixed.
For $\ell \in [s]$, we define $\xi_{\ell,b} := \eins_{\Delta_n^c}(X_{\ell}^b)-\mathrm{P}(X\in \Delta_n^c)$.
Then we have $\mathbb{E}_{\mathrm{P}}\xi_{\ell,b}=0$ and $\mathbb{E}_{\mathrm{P}} [\xi_{\ell,b}]^2\leq \mathrm{P}(X\in \Delta_n^c)$. Given $\|f\|_{\infty}\leq \overline{c}$ in Assumption \ref{asp::holder}, we have $\mathrm{P}(X\in \Delta_n^c)\leq  \overline{c}\mu(\Delta_n^c)\lesssim (\overline{k}/s)^{1/d}$. Therefore, we have $\mathbb{E}_{\mathrm{P}} [\xi_{\ell,b}]^2\lesssim (\overline{k}/s)^{1/d}$. 
Applying Bernstein's inequality in Lemma \ref{lem::bernstein}, we obtain 
\begin{align*}
\frac{1}{s} \sum_{\ell=1}^s \eins_{\Delta_n^c}(X_\ell^b) - \mathrm{P}(X\in \Delta_n^c)
\lesssim \sqrt{2(\overline{k}/s)^{1/d}\tau/s}+2\tau/(3s)
\end{align*}
with probability $\mathrm{P}^s$ at least $1-e^{-\tau}$. Setting $\tau := 3 \log n$ and using $\mathrm{P}(X\in \Delta_n^c)\lesssim (\overline{k}/s)^{1/d}$, we obtain 
\begin{align*}
1-n^b/s=\frac{1}{s} \sum_{\ell=1}^s \eins_{\Delta_n^c}(X_\ell^b)\lesssim \bigl(\overline{k}/s\bigr)^{1/d}+ \sqrt{(\overline{k}/s)^{1/d}(\log n)/s}+(\log n)/s.
\end{align*}
with probability $\mathrm{P}^s$ at least $1-1/n^3$.
By the definition of $N_1$ in Proposition \ref{pro::I456}, we have
$\underline{k}\geq c_n$ for all $n>N_1$, where $c_n$ is specified in Lemma \ref{lem::Rrho}.
Therefore, we obtain $
(\overline{k}/s)^{1/d}\geq \overline{k}/s \geq \underline{k}/s \geq c_n/s$.
This yields 
\begin{align*}
1-n^b/s\lesssim 
\bigl(\overline{k}/s\bigr)^{1/d}+ \sqrt{(\overline{k}/s)^{1/d}(\log n)/s}+(\log n)/s\lesssim (\overline{k}/s)^{1/d}.
\end{align*}
with probability $\mathrm{P}^s$ at least $1-1/n^3$.
Using the union bound, this inequality holds for all $b \in [B]$ with probability  $ \mathrm{P}^{Bs}$ at least $1-1/n^2$. 
This completes the proof.
\end{proof}

\begin{proof}[of Proposition \ref{pro::I456sur}]
Let $\Omega_1$ denote the event defined by \eqref{equ::infrksuprk} and \eqref{equ::pbxrkxkn} in Lemma \ref{lem::Rrho}.
Furthermore, let 
$\Omega_3$ be the event defined by the inequality for the upper bound of $(I)$, $(II)$, and $(III)$ in Proposition \ref{pro::I456}, and let $\Omega_4$ be the event defined by the inequality for the upper bound of $1 - n^b/s$ in Lemma \ref{lem::samplenear}.
By applying the union bound argument on Lemma \ref{lem::Rrho}, \ref{lem::samplenear}, and Proposition \ref{pro::I456}, the event $\Omega_1\cap \Omega_3\cap\Omega_4$ holds with probability $\mathrm{P}^{Bs}$ at least $1-1/n^2-1/n^2-2/n^2\geq 1-4/n^2$ for all $n>N_1^*:=N_1$, where $N_1$ is specified in Proposition \ref{pro::I456}.
The subsequent arguments assume that $\Omega_1\cap \Omega_3\cap \Omega_4$ holds and $n>N_1^*$.

For $X_i^b$ satisfying $B(X_i^b,R_{s,(k^{b'})}^{b'}(X_i^b))\subset [0,1]^d$ for all $b'\in [B]$, using the bound of $(I)$, $(II)$, and $(III)$ on the event $\Omega_3$ and \eqref{equ::flambdabxerror}, we get
\begin{align}\label{equ::lxibfnb}
L(X_i^b, f_n^B) =
\bigl|f_n^B(X_i^b)-f(X_i^b)\bigr|
& \lesssim (I) \cdot (II) + (I) \cdot (III)\nonumber \\
& \lesssim ((\log n)/\overline{k})^{1+1/d} + \bigl((\log n)/ (\overline{k} B)\bigr)^{1/2} + (\overline{k}/s)^{1/d}.
\end{align}
The condition $B \lesssim (\overline{k}/(\log n))^{1+2/d}$ implies that 
\begin{align}\label{equ::lognk11d}
((\log n)/\overline{k})^{1+1/d}\lesssim   \bigl((\log n)/ (\overline{k} B)\bigr)^{1/2}.
\end{align}
The conditions $\|w^b\|_2\gtrsim \bigl(k^b\bigr)^{-1/2}$ for $b\in [B]$ and $\underline{k}\asymp \overline{k}$ yield 
that $\|w^b\|_2\gtrsim (\underline{k})^{-1/2}\gtrsim (\overline{k})^{-1/2}$.
Combining this with the condition $\log s \asymp \log n$ in $(ii)$, we obtain 
\begin{align}\label{equ::logskb}
\bigl((\log n)/ (\overline{k} B)\bigr)^{1/2}\lesssim \sqrt{(\log s)/B}\cdot \|w^b\|_2.
\end{align}
By \eqref{equ::rsblower} in the proof of Proposition \ref{pro::I456} and the condition $\underline{k}\asymp \overline{k}$ in $(ii)$, we obtain that for all $n>N_1$, 
\begin{align*}
(\overline{k}/s)^{1/d}\lesssim \bigl( \underline{k} / s \bigr)^{1/d}\lesssim R_s^{w,b}(X_i^b).
\end{align*}
on the event $\Omega_1$.
Combining this with \eqref{equ::lxibfnb}, \eqref{equ::lognk11d}, and \eqref{equ::logskb}, it follows that
\begin{align*}
\bigl|f_n^B(X_i^b)-f(X_i^b)\bigr|\lesssim \sqrt{(\log s)/B} \cdot \|w^b\|_2 + R^{w,b}_s(X_i^b).
\end{align*}
This completes the proof of \eqref{equ:bagupperbound}.

Next, let us turn to the proof of \eqref{equ::bagerub}.
Let $\Delta_n$, $\mathcal{I}^b$, and $n^b$ be defined by Lemma \ref{lem::samplenear}.
Then, it is clear to see that 
\begin{align}\label{mathcalrldsfws}
\mathcal{R}_{L,D_s^B}(f_n^B)
& = \frac{1}{B} \sum_{b=1}^B \frac{1}{s} \sum_{i=1}^s \bigl| f_n^B(X_i^b) - f(X_i^b) \bigr|
\nonumber\\
& = \frac{1}{B} \sum_{b=1}^B \frac{1}{s} \biggl( \sum_{i \in \mathcal{I}^b} \bigl| f_n^B(X_i^b) - f(X_i^b) \bigr| + \sum_{i \in [s] \setminus \mathcal{I}^b} \bigl| f_n^B(X_i^b) - f(X_i^b) \bigr| \biggr).
\end{align}
Let $b\in [B]$ be fixed for now. 
We consider the first term on the right-hand side of \eqref{mathcalrldsfws}. 
For any $i\in \mathcal{I}^b$, we have $X_i^b\in \Delta_n$.
Consequently, for any $b'\in [B]$ and $y\in B(X_i^b,R_{s,(k^{b'})}^{b'}(X_i^b))$, we obtain
$d(y,\mathbb{R}^d\setminus[0,1]^d)\geq c_3'(\overline{k}/s)^{1/d}-R_{s,(k^{b'})}^{b'}(X_i^b) \geq 0$ on the event $\Omega_1$.
This implies that $y\in [0,1]^d$ and thus
$B(X_i^b,R_{s,(k^{b'})}^{b'}(X_i^b))\subset [0,1]^d$ for all $i\in \mathcal{I}^b$ and $b'\in [B]$.
Therefore, by \eqref{equ:bagupperbound}, we have 
\begin{align}\label{equ::1bb11s}
\frac{1}{s}\sum_{i\in \mathcal{I}^b}\bigl| f_n^B(X_i^b) - f(X_i^b) \bigr|
& \lesssim \frac{1}{s}\sum_{i\in \mathcal{I}^b} \biggl(\sqrt{(\log s)/B} \cdot \|w^b\|_2 + R^{w,b}_s(X_i^b)\biggr)\nonumber \\
& \lesssim \sqrt{(\log s)/B} \cdot \|w^b\|_2+\frac{1}{s}\sum_{i=1}^s  R^{w,b}_s(X_i^b).
\end{align}
Now, we consider the second term on the right-hand side of \eqref{mathcalrldsfws}.
The condition $\sum_{i=1}^s i^{1/d}w_i^b \asymp \bigl(k^b \bigr)^{1/d}$, $b \in [B]$, together with \eqref{equ::rnbxlower}
in the proof of Proposition \ref{pro::I456} implies that for all $x\in [0,1]^d$, there holds
\begin{align*}
f_n^B(x)
= \frac{1}{V_d R_n^B(x)^d}
\biggl( \frac{1}{B} \sum_{b=1}^B \sum_{i=1}^s w_i^b  (i/s)^{1/d} \biggr)^d
\lesssim \frac{\overline{k} / s}{R_n^B(x)^d} 
\lesssim 1
\end{align*}
on the event $\Omega_1$.
Combining this with $\|f\|_{\infty}\leq \overline{c}$ in Assumption \ref{asp::holder}, on the event $\Omega_3$, we get
\begin{align*}
\sum_{i\in [s]\setminus\mathcal{I}^b} \bigl| f_n^B(X_i^b) - f(X_i^b) \bigr|\lesssim \#\{i\in [s]: X_i^b\in \Delta_n^c\}= s - n^b \lesssim s(k^b/s)^{1/d}\lesssim s(\overline{k}/s)^{1/d}.
\end{align*} 
By \eqref{equ::rsblower} in the proof of Proposition \ref{pro::I456}, we have $(\overline{k}/s)^{1/d}
\lesssim  R_s^{w,b}(X_i^b)$ for all 
$i \in [s]$ on the event $\Omega_1$.
Consequently, we obtain
\begin{align*}
\frac{1}{s}\sum_{i\in [s]\setminus\mathcal{I}^b} \bigl| f_n^B(X_i^b) - f(X_i^b) \bigr|\lesssim (\overline{k}/s)^{1/d} \lesssim  \frac{1}{s}\sum_{i=1}^s R_s^{w,b}(X_i^b).
\end{align*}
Combining this with \eqref{mathcalrldsfws} and \eqref{equ::1bb11s}, we obtain
\begin{align*}
\mathcal{R}_{L,D_s^B}(f_n^B)
\lesssim \mathcal{R}_{L,D_s^B}^{\mathrm{sur}}(f_n^B)&:=
\frac{1}{B}\sum_{b=1}^B\biggl(\sqrt{(\log s)/B} \cdot \|w^b\|_2 + \frac{1}{s}\sum_{i=1}^s R^{w,b}_s(X_i^b)\biggr).
\end{align*}
Since $\frac{1}{s}\sum_{i=1}^s R^{w,b}_s(X_i^b) = \sum_{i=1}^s w_i^b \overline{R}_{s,(i)}^b$, we obtain the desired assertion. 
\end{proof}

\subsection{Proofs Related to the Convergence Rates of BRDDE}\label{sec::proofsbrdde}

We present the proofs related to the results concerning the surrogate risk minimization in Section \ref{sec::proofauc}.
Additionally, the proof of Theorem \ref{pro::rateBWDDE} are provided in Section \ref{sec::proofstheorem4}.

\subsubsection{Proofs Related to Section \ref{sec::analysissr}}\label{sec::proofauc}

The following lemma, which will be used several times in the sequel, supplies the key to the proof of Proposition \ref{lem::order:k:w:lambda}.

\begin{lemma}\label{lem:existenceofsolution}
Let $\overline{R}_{s,(i)}^b$ be defined as in Proposition \ref{pro::I456sur}, and let $w^{b,*}$ and $k^{b,*}$ be defined as in Proposition \ref{lem::order:k:w:lambda}.
Then, for each $b\in [B]$, there exists a constant $\mu^b>0$ such that $ \overline{R}_{s,( k^{b,*})}^b <\mu^b\leq  \overline{R}_{s,( k^{b,*}+1)}^b$. 
(For simplicity, we set $\overline{R}_{s,(i)}^b = \infty$ for all $i > s$.) Moreover, the optimal weights satisfy
\begin{align} \label{equ:structureofweights}
w_i^{b,*}=\frac{\mu^b- \overline{R}_{s,(i)}^b }{\sum_{i=1}^{k^{b,*}}\big(\mu^b- \overline{R}_{s,(i)}^b\big)}, \qquad1\leq i\leq  k^{b,*}.
\end{align}
Moreover, the weights $w_i^{b,*}$ are bounded as follows:
\begin{align}\label{equ::wixbound}
\frac{\overline{R}_{s,( k^{b,*})}^b-\overline{R}_{s,(i)}^b}{\sum_{i=1}^{k^{b,*}}\big(\overline{R}_{s,( k^{b,*}+1)}^b-\overline{R}_{s,(i)}^b\big)}
\leq w_i^{b,*} 
\leq \frac{\overline{R}_{s,( k^{b,*}+1)}^b-\overline{R}_{s,(i)}^b}{\sum_{i=1}^{k^{b,*}}\big(\overline{R}_{s,( k^{b,*})}^b-\overline{R}_{s,(i)}^b\big)}, \quad 1\leq i\leq  k^{b,*}.
\end{align}
Additionally, the following inequality holds:
\begin{align}\label{equ::sumlambdax}
\sum^{k^{b,*}}_{i=1}\big(\overline{R}_{s,( k^{b,*})}^b-\overline{R}_{s,(i)}^b\big)^2< \frac{\log s}{B}\leq \sum^{k^{b,*}}_{i=1}\big(\overline{R}_{s,( k^{b,*}+1)}^b-\overline{R}_{s,(i)}^b\big)^2.
\end{align}
\end{lemma}

\begin{proof}[of Lemma \ref{lem:existenceofsolution}]
By Theorem 3.1 in \cite{anava2016k}, for each $b\in [B]$, there exists a constant $\mu^b > 0$ such that the weights satisfy the formula in \eqref{equ:structureofweights}. 
This inequality, together with \eqref{equ:structureofweights}, implies that the bound in \eqref{equ::wixbound} holds.
Moreover, by \eqref{equ::logswi}, we have 
\begin{align*}
\sqrt{(\log s)/B}\cdot
{w_i^{b,*}}/{\|w^{b,*}\|_2}=\mu^b+\nu_i^b-\overline{R}_{s,(i)}^b.
\end{align*}
For $1\leq i\leq k^{b,*}$, we have $w_i^{b,*}>0$, this implies that $\nu_i^b=0$ by the KKT condition. Therefore, we have
\begin{align*}
\sqrt{(\log s)/B} \cdot w_i^{b,*}/\|w^{b,*}\|_2 = \mu^b-\overline{R}_{s,(i)}^b, \quad 1\leq i\leq k^{b,*}.
\end{align*}
Now, summing over $i$ from $1$ to $k^{b,*}$, we get
\begin{align}\label{equ::mubequality}
\sum_{i=1}^{k^{b,*}}\big(\mu^b- \overline{R}_{s,(i)}^b\big)^2 = \frac{\log s}{B} \sum_{i=1}^{k^{b,*}} (w_i^{b,*})^2/\|w^{b,*}\|_2^2 =\frac{\log s}{B}.
\end{align}
This, along with the inequality $\overline{R}_{s,( k^{b,*})}^b< \mu^b\leq  \overline{R}_{s,( k^{b,*}+1)}^b$, establishes \eqref{equ::sumlambdax}.
\end{proof}

\begin{proof}[of Proposition \ref{lem::order:k:w:lambda}]
Let $c_2'$ and $c_3'$ be the constants defined in \eqref{equ::infrksuprk}, and let $\{c_n\}$ be the sequence from Lemma \ref{lem::Rrho}.
Define $c_4':=(c_2'/c_3')^d<1$.
Since $s\asymp(n/\log n)^{(d+1)/(d+2)}$ and $B \asymp n^{1/(d+2)}(\log n)^{(d+1)/(d+2)}$, there exists constants $c_5'$, $c_6'$, and $c_7'$ such that 
\begin{align}\label{equ::choicesb1}
c_5'\leq \frac{s}{ (n/\log n)^{(d+1)/(d+2)}}\leq c_6'\text{ and } B\geq c_7'n^{1/(d+2)}(\log n)^{(d+1)/(d+2)}.
\end{align}
The choice of $s$ implies that $\log s\gtrsim \log(n)-\log(\log n)\gtrsim \log n$. Consequently, we have 
\begin{align*}
s^{2/d}c_n^{-1-2/d}\log s\gtrsim n^{\frac{2(d+1)}{d(d+2)}}(\log n)^{-\frac{4d+6}{d(d+2)}}.
\end{align*}
Using the order of $B$, we know that there exists $n_3\in \mathbb{N}$ such that for all $n> n_3$, there holds 
\begin{align}\label{equ::BS}
s^{2/d}c_n^{-1-2/d}(\log s)/(3^{2/d+1}(c_3')^2c_4'^{-1-2/d})>B.
\end{align}
Furthermore, from the order of $s$ and $B$, we see that there exists $n_4\in \mathbb{N}$ such that for all $n>n_4$, the following three inequalities hold:
\begin{align}
&c_5'(n/\log n)^{(d+1)/(d+2)} \geq c_1',\label{equ::nchoice1}\\
& c_7'n^{1/(d+2)}(\log n)^{(d+1)/(d+2)} \geq 2(d^2+4)(\log n)/3, \label{equ::nchoice2}\\
& c_5'(n/\log n)^{d/(d+2)} \geq c_8':=\frac{4(c_6')^{\frac{2}{d+2}}}{c_4'((c_3')^2 c_7')^{\frac{d}{2+d}}}\biggl(\int_{1/2}^1 (1-t^{1/d})^2 \, dt \biggr)^{-\frac{d}{d+2}},\label{equ::nchoice3}
\end{align}
where $c_1'$ is the constant specified in Lemma \ref{lem::Rrho}.
As we will demonstrate in the second part of this argument, the inequalities \eqref{equ::nchoice1} and \eqref{equ::nchoice3}, together with the condition $n>\lceil (c_6')^{2+d} \rceil$, guarantee that the lower bound for $s$ is satisfied. Similarly, inequality \eqref{equ::nchoice2} ensures that the lower bound for $B$ holds.

Additionally, by the divergence of the sequence $c_n$, there exists $n_5\in \mathbb{N}$ such that for all $n\geq n_5$, we have 
\begin{align}\label{equ::n4}
\frac{1}{c_n}\leq \frac{1}{2}\int_{1/2}^1 t^{1/d}(1-t^{1/d})\, dt.
\end{align}
Let $\Omega_1$ denote the event defined by \eqref{equ::infrksuprk} and \eqref{equ::pbxrkxkn} in Lemma \ref{lem::Rrho}.
By Lemma \ref{lem::Rrho}, the event $\Omega_1$ holds with probability $\mathrm{P}^{Bs}$ at least $1-1/n^2$. 
For the remainder of the proof, we assume that  $\Omega_1$ holds and that $n>N_2:=\max\{n_3,n_4,n_5, \lceil (c_6')^{2+d}\rceil\}$.
We proceed with the proof of statement 1.

\vspace{+2mm}
\noindent
\textit{Verification of Condition $(i)$:}
We first show that $k^{b,*}\geq \lceil 2c_n/c_4'\rceil$ for all $b\in [B]$ by contradiction.
Suppose that $k^{b,*}< \lceil 2c_n/c_4'\rceil$  for some $b\in [B]$.
Since $c_4'<1$, it follows that $c_n/c_4'>c_n>2$. Therefore, we have $k^{b,*}+1<\lceil 2c_n/c_4'\rceil+1\leq 2c_n/c_4'+2\leq 3c_n/c_4'$. This leads to the bound
$
\overline{R}_{s,(k^{b,*}+1)}^b
\leq c_3'((k^{b,*}+1)/s)^{1/d}\leq 3^{1/d}c_3'c_4'^{-1/d}(c_n/s)^{1/d}$ on the event $\Omega_1$.
Therefore, we obtain
\begin{align*}
\sum^{k^{b,*}}_{i=1}\big(\overline{R}_{s,(k^{b,*}+1)}^b-\overline{R}_{s, (i)}^b\big)^2 
\leq k^{b,*} \bigl(\overline{R}_{s,(k^{b,*}+1)}^b\bigr)^2
\leq 3^{2/d+1}(c_3')^2c_4'^{-1-2/d} c_n^{1+2/d} s^{-2/d}.
\end{align*}
Combining this with \eqref{equ::BS}, we conclude that 
$\sum^{k^{b,*}}_{i=1}\big(\overline{R}_{s,(k^{b,*}+1)}^b-\overline{R}_{s, (i)}^b\big)^2 <(\log s)/B$.
However, by \eqref{equ::sumlambdax} in Lemma \ref{lem:existenceofsolution}, we know that
$(\log s)/B 
\leq \sum^{k^{b,*}}_{i=1}\big(\overline{R}_{s,(k^{b,*}+1)}^b-\overline{R}_{s, (i)}^b\big)^2$.
This leads to a contradiction, implying that 
\begin{align}\label{equ::kbcn}
k^{b,*}\geq \lceil 2c_n/c_4'\rceil > c_n, \quad \forall b\in [B].
\end{align}
Let $b\in [B]$ be fixed.
By Lemma \ref{lem:existenceofsolution}, we have
\begin{align}\label{equ::wibstarin}
\sum_{i=1}^{c_n} w_i^{b,*}\leq
\frac{\sum_{i=1}^{c_n}\bigl(\overline{R}_{s,(k^{b,*}+1)}^b-\overline{R}_{s,(i)}^b\bigr)}{\sum_{i=1}^{k^{b,*}}\big(\overline{R}^b_{s,(k^{b,*})}-\overline{R}_{s,(i)}^b\big)}.
\end{align}
On the event $\Omega_1$, we have the following upper bound for the numerator:
\begin{align}\label{equ::upper}
\sum_{i=1}^{c_n}\bigl(\overline{R}_{s,(k^{b,*}+1)}^b-\overline{R}_{s,(i)}^b\bigr)\leq c_n \overline{R}_{s,(k^{b,*}+1)}^b\lesssim c_n(k^{b,*}/s)^{1/d}.
\end{align}
Next, we establish the lower bound for $\sum_{i=1}^{k^{b,*}}\big(\overline{R}^b_{s,(k^{b,*})}-\overline{R}_{s,(i)}^b\big)$. 
Since $k^{b,*}>  c_n$ by \eqref{equ::kbcn}, we have $\overline{R}_{s,(k^{b,*})}^b\geq c_2' \big(k^{b,*}/s\big)^{1/d} = c_3'\bigl(c_4'k^{b,*}/s\bigr)^{1/d}\geq c_3'\bigl(\lfloor c_4'k^{b,*}\rfloor/s\bigr)^{1/d} $ and $\overline{R}_{s,(i)}^b\leq c_3'(i/s)^{1/d}$ for $c_n\leq i\leq s$ on the event $\Omega_1$.
Therefore, we obtain
\begin{align}\label{equ::rkxminusrix}
\overline{R}_{s,(k^{b,*})}^b-\overline{R}_{s,(i)}^b\geq c_3'\big(\lfloor c_4'k^{b,*}\rfloor/s\big)^{1/d}-c_3'(i/s)^{1/d}.
\end{align}
for $c_n\leq i\leq s$. Since $\lfloor c_4'k^{b,*}\rfloor\geq 2c_n$ by \eqref{equ::kbcn}, we obtain
\begin{align*}
&\sum_{i=1}^{k^{b,*}} \big(\overline{R}_{s,(k^{b,*})}^b-\overline{R}_{s,(i)}^b\big)
\geq 	\sum_{i=c_n}^{\lfloor c_4'k^{b,*}\rfloor } \big(\overline{R}_{s,(k^{b,*})}^b-\overline{R}_{s,(i)}^b\big) \gtrsim	\sum_{i=c_n}^{\lfloor c_4'k^{b,*}\rfloor } \biggl(\biggl(\frac{\lfloor c_4'k^{b,*}\rfloor}{s}\biggr)^{1/d}-\biggl(\frac{i}{s}\biggr)^{1/d}\bigg) \nonumber \\
& \gtrsim
\frac{(k^{b,*})^{1/d+1}}{s^{1/d}}  \cdot \frac{1}{\lfloor c_4'k^{b,*}\rfloor}\sum_{i=c_n}^{\lfloor c_4'k^{b,*}\rfloor }\bigg(1-\bigg(\frac{i}{\lfloor c_4'k^{b,*}\rfloor}\bigg)^{1/d}\bigg).
\end{align*}
Since $g_1(t):=1-t^{1/d}$ is a monotonically decreasing function for $0\leq t\leq 1$, 
we have
\begin{align}\label{equ::sumicdnkxlower}
\sum_{i=1}^{k^{b,*}} \big(\overline{R}_{s,(k^{b,*})}^b-\overline{R}_{s,(i)}^b\big)
&\gtrsim \frac{(k^{b,*})^{1/d+1}}{s^{1/d}}\int_{c_n/\lfloor c_4'k^{b,*}\rfloor}^1 g_1(t)\, dt\nonumber \\
&\gtrsim \frac{(k^{b,*})^{1/d+1}}{s^{1/d}} \int_{1/2}^1 g_1(t)\, dt\gtrsim  (k^{b,*})^{1/d+1}s^{-1/d}.
\end{align}
Combining this with the upper bound in \eqref{equ::upper} and inequality \eqref{equ::wibstarin}, we get
$\sum_{i=1}^{c_n} w_i^{b,*}\lesssim (\log n)/k^{b,*}$ for $b\in [B]$. Hence, we verify condition $(i)$ in Proposition \ref{pro::I456sur}. 

\vspace{+2mm}
\noindent
\textit{Verification of Condition $(ii)$:} 
Let $b\in [B]$ be fixed.
Since $k^{b,*}>c_n$ by \eqref{equ::kbcn}, we have
\begin{align*}
\sum_{i=1}^{k^{b,*}} \big(\overline{R}_{s,(k^{b,*}+1)}^b-\overline{R}_{s,(i)}^b\big)^2&\lesssim 		\sum_{i=1}^{k^{b,*}} \bigl(\overline{R}_{s,(k^{b,*}+1)}^b \bigr)^2
\lesssim {(k^{b,*})^{2/d+1}} {s^{-2/d}}.
\end{align*}
on the event $\Omega_1$.
Combining this with \eqref{equ::sumlambdax} in Lemma \ref{lem:existenceofsolution}, we have $(\log s)/B  \lesssim (k^{b,*})^{2/d+1}s^{-2/d}$, which implies the lower bound $
k^{b,*}\gtrsim s^{2/(2+d)}((\log s)/B)^{d/(2+d)}$.
Combining this with the choice of $s$ and $B$, we get the lower bound
$k^{b,*}\gtrsim (n/\log n)^{1/(d+2)}$.

Next, we derive the upper bound of $k^{b,*}$.
Using the lower bound in \eqref{equ::rkxminusrix} again, we have
\begin{align*}
\sum_{i=1}^{k^{b,*}} \big(\overline{R}_{s,(k^{b,*})}^b-\overline{R}_{s,(i)}^b\big)^2
& \geq \sum_{i=c_n}^{\lfloor c_4'k^{b,*}\rfloor } \big(\overline{R}_{s,(k^{b,*})}^b-\overline{R}_{s,(i)}^b\big)^2\\
&\geq (c_3')^2 \sum_{i=c_n}^{\lfloor c_4'k^{b,*}\rfloor } \bigl(\big(\lfloor c_4'k^{b,*}\rfloor/s\big)^{1/d}-(i/s)^{1/d}\big)^2
\\
& = \frac{(c_3')^2\lfloor c_4'k^{b,*}\rfloor^{2/d+1}}{s^{2/d}}\cdot \frac{1}{\lfloor c_4'k^{b,*}\rfloor}\sum_{i=c_n}^{\lfloor c_4'k^{b,*}\rfloor}\biggl( 1- \biggl(\frac{i}{\lfloor c_4'k^{b,*}\rfloor}\biggr)^{1/d}\biggr)^2.
\end{align*}
Since $\lfloor c_4'k^{b,*}\rfloor\geq 2c_n>144$ by \eqref{equ::kbcn}, we have $\lfloor c_4'k^{b,*}\rfloor\geq c_4'k^{b,*}/2$. Therefore, we get
\begin{align*}
\sum_{i=1}^{k^{b,*}} \big(\overline{R}_{s,(k^{b,*})}^b-\overline{R}_{s,(i)}^b\big)^2 
\geq \frac{(c_3')^2(c_4')^{2/d+1}}{2^{2/d+1}} \int_{1/2}^1 (1-t^{1/d})^2 \, dt \cdot(k^{b,*})^{2/d+1}s^{-2/d}.
\end{align*}
Combining this with \eqref{equ::sumlambdax} in Lemma \ref{lem:existenceofsolution}, we obtain \begin{align}\label{equ::losblower}
\frac{\log s}{B} > \frac{(c_3')^2(c_4')^{2/d+1}}{2^{2/d+1}} \int_{1/2}^1 (1-t^{1/d})^2 \, dt \cdot(k^{b,*})^{2/d+1}s^{-2/d}.
\end{align} 
Since $n\geq N_2\geq (c_6')^{2+d}$, using the choice of $s$ in \eqref{equ::choicesb1}, we get
\begin{align*}
\log s \leq \log c_6' + \frac{d+1}{d+2}\log n \leq \log n.
\end{align*}
Using this bound with the choices of $B$ and $s$ in \eqref{equ::choicesb1}, and the inequality \eqref{equ::losblower}, we derive
\begin{align}\label{equ::kbstarupper}
&k^{b,*}\leq \frac{2}{c_4'(c_3')^{2d/(2+d)}}\biggl(\int_{1/2}^1 (1-t^{1/d})^2 \, dt \biggr)^{-\frac{d}{d+2}}\biggl(\frac{\log s}{B}\biggr)^{\frac{d}{d+2}}s^{\frac{2}{d+2}}\nonumber\\
&\leq \frac{2(c_6')^{\frac{2}{d+2}}}{c_4'((c_3')^2 c_7')^{\frac{d}{d+2}}}\biggl(\int_{1/2}^1 (1-t^{1/d})^2 \, dt \biggr)^{-\frac{d}{d+2}}\biggl(\frac{n}{\log n}\biggr)^{\frac{1}{d+2}}=\frac{c_8'}{2}\biggl(\frac{n}{\log n}\biggr)^{\frac{1}{d+2}}.
\end{align}
Combining the lower and upper bounds of $k^{b,*}$, we conclude 
\begin{align}\label{equ::kbstar}
k^{b,*}\asymp (n/\log n)^{1/(d+2)}, \quad \forall b\in [B].
\end{align}
The inequalities \eqref{equ::choicesb1} and \eqref{equ::nchoice1} ensure that $s\geq c_1'$ for all $n>N_2$. 
Combining the upper bound of $k^{b,*}$ in \eqref{equ::kbstarupper} with the choice of $s$ in \eqref{equ::choicesb1} and the inequality \eqref{equ::nchoice3} yields $s\geq 2\overline{k}$ for all $n>N_2$. Therefore, we have $s\geq \max\{c_1',2\overline{k}\}$ for all $n>N_2$ and it is straightforward to verify that $\log s\asymp \log n$.
Additionally, the choice of $B$ in \eqref{equ::choicesb1} combined with \eqref{equ::nchoice2} ensures that $
B \geq 2(d^2+4)(\log n)/3$ for all $n>N_2$.
Moreover, using \eqref{equ::kbstar} and the choices of $B$ and $s$, we obtain
$B\lesssim (\overline{k}/(\log n))^{1+2/d}$. This complete the verification of condition $(ii)$ in Proposition \ref{pro::I456sur}.

\vspace{+2mm}
\noindent
\textit{Verification of Condition $(iii)$:}
Fix $b\in [B]$.
By inequality \eqref{equ::wixbound} in Lemma \ref{lem:existenceofsolution}, we have
\begin{align}\label{equ::sumiwi}
\frac{\sum_{i=1}^{k^{b,*}}i^{1/d}\big(\overline{R}_{s,(k^{b,*})}^b-\overline{R}_{s,(i)}^b\big)}{\sum_{i=1}^{k^{b,*}}\big(\overline{R}_{s,(k^{b,*}+1)}^b-\overline{R}_{s,(i)}^b\big)}\leq \sum_{i=1}^{k^{b,*}} i^{1/d}w_i^{b,*}\leq \frac{\sum_{i=1}^{k^{b,*}}i^{1/d}\big(\overline{R}_{s,(k^{b,*}+1)}^b-\overline{R}_{s,(i)}^b\big)}{\sum_{i=1}^{k^{b,*}}\big(\overline{R}_{s,(k^{b,*})}^b-\overline{R}_{s,(i)}^b\big)}.
\end{align}	
We first evaluate the numerator on the left-hand side.
Since $c_4'<1$ and $\lfloor c_4'k^{b,*}\rfloor\geq 2c_n$ by \eqref{equ::kbcn}, we have
\begin{align*}
\sum_{i=1}^{k^{b,*}}i^{1/d}\big(\overline{R}_{s,(k^{b,*})}^b-\overline{R}_{s,(i)}^b\big)\geq 	\sum_{i=c_n}^{\lfloor c_4'k^{b,*}\rfloor}i^{1/d}\big(\overline{R}_{s,(k^{b,*})}^b-\overline{R}_{s,(i)}^b\big)
\end{align*}
Applying inequality \eqref{equ::rkxminusrix}, we obtain
\begin{align*}
&\sum_{i=1}^{k^{b,*}}i^{1/d}\big(\overline{R}_{s,(k^{b,*})}^b-\overline{R}_{s,(i)}^b\big)
\gtrsim s^{-1/d}\sum_{i=c_n}^{\lfloor c_4'k^{b,*}\rfloor } i^{1/d} \bigl(\lfloor c_4'k^{b,*}\rfloor^{1/d}-i^{1/d}\bigr)\\
& \gtrsim s^{-1/d}(k^{b,*})^{2/d+1} \cdot \frac{1}{\lfloor c_4'k^{b,*}\rfloor} \sum_{i=c_n}^{\lfloor c_4'k^{b,*}\rfloor}\bigg(\frac{i}{\lfloor c_4'k^{b,*}\rfloor}\bigg)^{1/d}
\biggl(1-\biggl(\frac{i}{\lfloor c_4'k^{b,*}\rfloor}\biggr)^{1/d}\biggr).
\end{align*}
Define $g_2(t):=t^{1/d}(1-t^{1/d})$ for $t\in [0,1]$. Since $g_2(t)\leq 1$ for $t\in [0,1]$, we obtain the lower bound
\begin{align}\label{equ::sumicdnkx}
\sum_{i=1}^{k^{b,*}} i^{1/d}\big(\overline{R}_{s,(k^{b,*})}^b-\overline{R}_{s,(i)}^b\big)\gtrsim \frac{(k^{b,*})^{2/d+1}}{s^{1/d}} \bigg(\int_{(c_n-1)/\lfloor c_4'k^{b,*}\rfloor}^1 g_2(t)\, dt-\frac{2}{\lfloor c_4'k^{b,*}\rfloor} \bigg).
\end{align}
From \eqref{equ::n4} and \eqref{equ::kbcn}, we have ${\lfloor c_4'k^{b,*}\rfloor}^{-1}\leq 1/(2c_n)\leq \int_{1/2}^1 g_2(t)\, dt/4$ and $(c_n-1)/\lfloor c_4'k^{b,*}\rfloor\leq 1/2$ for all $n>N_2$.
Therefore, we obtain the following lower bound:
\begin{align*}
\sum_{i=1}^{k^{b,*}}i^{1/d}\big(\overline{R}_{s,(k^{b,*})}^b-\overline{R}_{s,(i)}^b\big)
\gtrsim \frac{(k^{b,*})^{2/d+1}}{2s^{1/d}} \int_{1/2}^1 g_2(t)\, dt
\gtrsim  \frac{(k^{b,*})^{2/d+1}}{s^{1/d}}.
\end{align*} 
On the other hand, on the event $\Omega_1$, since $k^{b,*}> c_n$ by \eqref{equ::kbcn}, we have the upper bound:
\begin{align*}
\sum_{i=1}^{k^{b,*}}\big(\overline{R}_{s,(k^{b,*}+1)}^b-\overline{R}_{s,(i)}^b\big)\leq k^{b,*}\overline{R}_{s,(k^{b,*}+1)}^b\lesssim (k^{b,*})^{1+1/d}s^{-1/d}.
\end{align*} 
Combining these bounds with \eqref{equ::sumiwi}, we obtain $
\sum_{i=1}^{k^{b,*}} i^{1/d}w_i^{b,*}\gtrsim (k^{b,*})^{1/d}$.
By similar arguments, we can also derive the upper bound $\sum_{i=1}^{k^{b,*}} i^{1/d}w_i^{b,*}$
$\lesssim (k^{b,*})^{1/d}$ from the right-hand term of inequality \eqref{equ::sumiwi}. 
Hence, we verify the first part of condition \textit{(iii)}.

Using the formula for $w^{b,*}_i$ in \eqref{equ:structureofweights}, along with the equality \eqref{equ::mubequality}, and noting that
$\overline{R}_{s,( k^{b,*})}^b< \mu^b\leq  \overline{R}_{s,( k^{b,*}+1)}^b$ from Lemma \ref{lem:existenceofsolution}, we obtain
\begin{align*}	
\|w^{b,*}\|_2 =\frac{\sqrt{\sum_{i=1}^{k^{b,*}}\big(\mu^b-   \overline{R}_{s,(i)}^b\big)^2}}{\sum_{i=1}^{k^{b,*}}\big(\mu^b-\overline{R}_{s,(i)}^b\big)}
=\frac{((\log s)/B)^{1/2}}{\sum_{i=1}^{k^{b,*}}\big(\mu^b-   \overline{R}_{s,(i)}^b\big)}
<\frac{((\log s)/B)^{1/2}}{\sum_{i=1}^{k^{b,*}} (\overline{R}_{s,(k^{b,*})}^b-\overline{R}_{s,(i)}^b)}.
\end{align*}
Combining this with \eqref{equ::sumicdnkxlower}, the choice of $B$ and $s$, and the order of $k^{b,*}$ in \eqref{equ::kbstar}, we derive
\begin{align*}
\|w^{b,*}\|_2\lesssim ((\log s)/B)^{1/2}/ \big((k^{b,*})^{1/d+1}s^{-1/d}\big)\lesssim (k^{b,*})^{-1/2}.
\end{align*}
Finally, using the Cauchy–Schwarz inequality, we obtain $1=\|w^{b,*}\|_1^2\leq k^{b,*}\|w^{b,*}\|_2^2$.
This implies the lower bound $\|w^{b,*}\|_2 \gtrsim (k^{b,*})^{-1/2}$. 
Therefore, combining the upper and lower bounds for $\|w^{b,*}\|_2$, we verify the second part of condition \textit{(iii)}.
This completes the verification of Condition \textit{(iii)} in Proposition \ref{pro::I456sur}.

\vspace{+2mm}
\noindent
\textit{Verification of Condition $(iv)$:}
Let $c_3'$ be the constant specified in Lemma \ref{lem::Rrho}, and let $c_8'$ be the constant defined in \eqref{equ::nchoice3}.
We define $h_n:=\lceil  c_8' (n/(\log n))^{1/(d+2)}/2 \rceil+1$, and define the following sequence 
\begin{align*}
 C_{n,i}:= 
\begin{cases}
c_3'(h_n/s)^{1/d}, & 1\leq i \leq (h_n\wedge s);\\
0, & \text{otherwise}.
\end{cases}
\end{align*}

By \eqref{equ::kbstarupper}, we have
$k^{b,*}+1\leq c_8' (n/\log n)^{1/(d+2)}/2+1\leq h_n$ for all $b\in [B]$. 
Therefore, for $1\leq i\leq k^{b,*}$, we have $i\leq  (h_n\wedge s)$, and thus $C_{n,i}=c_3'(h_n/s)^{1/d}$. 
By \eqref{equ::kbcn}, we have $k^{b,*}> c_n$ for all $b\in [B]$. 
Therefore, on the event $\Omega_1$, we have 
\begin{align}\label{equ::cnilower}
C_{n,i} = c_3'(h_n/s)^{1/d} \geq  c_3'((k^{b,*}+1)/s)^{1/d} \geq \overline{R}_{s,(k^{b,*}+1)}^b
\end{align}
for $1\leq i\leq k^{b,*}$ and $b\in [B]$.
Using the expression for $k^{b,*}$ in \eqref{equ::kbstar}, the choice of $s$, and the inequality \eqref{equ::sumicdnkxlower}, we derive 
\begin{align*}
\sum_{i=1}^{k^{b,*}} \big(\overline{R}_{s,(k^{b,*})}^b-\overline{R}_{s,(i)}^b\big)\gtrsim \biggl( \frac{n}{\log n}\biggr)^{\frac{1/d+1}{d+2}}\cdot  \biggl( \frac{n}{\log n}\biggr)^{-\frac{1/d+1}{d+2}}
\gtrsim 1.
\end{align*}
Consequently, using \eqref{equ::wixbound} from Lemma \ref{lem:existenceofsolution} and inequality \eqref{equ::cnilower}, we obtain
\begin{align*}
w_i^{b,*} \leq
\frac{\overline{R}_{s,(k^{b,*}+1)}^b-\overline{R}_{s,(i)}^b}{\sum_{i=1}^{k^{b,*}}\big(\overline{R}_{s,(k^{b,*})}^b-\overline{R}_{s,(i)}^b\big)}
\lesssim \overline{R}_{s,(k^{b,*}+1)}^b -\overline{R}_{s,(i)}^b\leq \overline{R}_{s,(k^{b,*}+1)}^b\leq C_{n,i}
\end{align*}
for $1\leq i\leq k^{b,*}$ and $b\in [B]$.
Since $w_i^{b,*}=0$ for $i>k^{b,*}$ and $b\in [B]$, we conclude that $w_i^{b,*} \lesssim C_{n,i}$ for all $b\in [B]$ and $i\in [s]$.

We now analyze the summation $\sum_{i=1}^s i^{1/d-1/2} C_{n,i}$.
Since $h_n\wedge s\leq h_n$, we have 
\begin{align*}
\sum_{i=1}^s i^{1/d-1/2} C_{n,i} 
= c_3'\sum_{i=1}^{h_n\wedge s}i^{1/d-1/2} (h_n/s)^{1/d}
\lesssim h_n^{1/d+1/2}(h_n/s)^{1/d} \lesssim (n/\log n)^{\frac{1/d-1/2}{d+2}}.
\end{align*}
On the other hand, from \eqref{equ::kbstar}, we know that $\overline{k}\gtrsim (n/\log n)^{1/(d+2)}$, which implies $\overline{k}^{1/d-1/2}\gtrsim (n/\log n)^{\frac{1/d-1/2}{d+2}}$. Therefore, we conclude that 
$\sum_{i=1}^s i^{1/d-1/2} C_{n,i}\lesssim \overline{k}^{1/d-1/2}$.
This completes the verification of condition $(iv)$ in Proposition \ref{pro::I456sur}.

Finally, statement 2 has been verified in \eqref{equ::kbstar}, completing the proof of Proposition \ref{lem::order:k:w:lambda}.
\end{proof}

Within the same theoretical framework as the preceding proof, the following lemma establishes a finite sample bound on the number of nonzero weights returned by SRM without bagging. This result is essential for the comparison of search complexity in Section \ref{sec:complexity}.
Following the approach in Section \ref{sec::datadrivennnad}, we randomly partition the data into two disjoint subsets—one for weight selection and the other for computing $k$-distances. 
For convenience, we assume without loss of generality that $n$ is an odd number in the following lemma.

\begin{lemma}\label{lemm::kbstarnobag}
Let $D_s$ be a subset of size $s = n/2$ randomly drawn from the dataset $D_n$. Define $\overline{R}_{s,(i)}$ as the average $i$-distance for any integer $i \leq s$ on the subset $D_s$, following the definition in Proposition \ref{pro::I456sur}. Furthermore, let $w^*$ be the solution to the following SRM problem:
\begin{align*}
w^*:=\argmin_{w\in \mathcal{W}_s} \; \sqrt{\log s} \cdot \|w\|_2 + \sum_{i=1}^s w_i \overline{R}_{s,(i)},
\end{align*}
and define $k^* := \sup\{i \in [s] : w_i^* \neq 0\}$. Then, there exists an integer $N_3 \in \mathbb{N}$ such that for all $n > N_3$, with probability $\mathrm{P}^s$ at least $1 - 1/n^2$, we have $k^* \asymp n^{2/(d+2)}(\log n)^{d/(d+2)}$.
\end{lemma}
\begin{proof}[of Lemma \ref{lemm::kbstarnobag}]
Let $c_3'$ denote the constant specified in Lemma \ref{lem::Rrho}, $c_4'$ denote the constant introduced in the proof of Proposition \ref{lem::order:k:w:lambda}, and $\{c_n\}$ denote the sequence defined in Lemma \ref{lem::Rrho}.
By the definition of $c_n$, there exists $n_6 \in \mathbb{N}$ such that for all $n \geq n_6$, we have 
\begin{align}\label{equ::contraana}
(n/2)^{2/d} c_n^{-1-2/d} \log(n/2)/(3^{2/d+1} (c_3')^2 c_4'^{-1-2/d}) > 1.
\end{align}
Define $N_3 := \max\{\lceil 2c_1'\rceil+2, n_6\}$, where $c_1'$ is as specified in Lemma \ref{lem::Rrho}. Let $\Omega_1$ denote the event defined by \eqref{equ::infrksuprk} and \eqref{equ::pbxrkxkn} in Lemma \ref{lem::Rrho} with $B = 1$ and $s=n/2$. By Lemma \ref{lem::Rrho}, for all $n > N_3$, we have $s \geq c_1'$, ensuring that the event $\Omega_1$ holds with probability $\mathrm{P}^{s}$ at least $1 - 1/n^2$. In the subsequent argument, we assume that $\Omega_1$ holds and $n \geq N_3$. 

Since \eqref{equ::contraana} is analogous to \eqref{equ::BS} in the proof of Proposition \ref{lem::order:k:w:lambda}, a similar argument as in the ``Verification of Condition $(i)$'' part of that proof shows that $k^* \geq \lceil 2c_n/c_4' \rceil$ by contradiction. Following the reasoning in the ``Verification of Condition $(ii)$'' part, we obtain $
\log s \asymp (k^{*})^{2/d+1} s^{-2/d}$.
Substituting $s = n/2$ into the above inequality yields $k^* \asymp n^{2/(d+2)}(\log n)^{d/(d+2)}$. This completes the proof of Lemma \ref{lemm::kbstarnobag}.
\end{proof}
\vspace{-1.5mm}

\subsubsection{Proofs Related to Section \ref{sec::convergencebrdde}}\label{sec::proofstheorem4}

In this section, we present the proof related to BRDDE. The weights $w^{b,*}$ are derived using SRM based on the data for BRDDE, whereas Proposition \ref{pro::I456} assumes that the weights are fixed and independent of the data. As a result, Proposition \ref{pro::I456} cannot be directly applied to establish the convergence rate of our density estimator. However, Proposition \ref{lem::order:k:w:lambda} ensures that the weights returned by SRM satisfy the required conditions in Proposition \ref{pro::I456sur} with high probability. Therefore, by making slight modifications to the proof of Proposition \ref{pro::I456}, we can establish the error decomposition of BRDDE as stated in Proposition \ref{pro::I456tilde}.

Before proceeding, we introduce additional notation.
Regarding the expression of $f^{B,*}(x)$ in \eqref{equ::fbstar}, we define the error term $(I')$, $(II')$, and $(III')$ as follows, corresponding to $(I)$, $(II)$, and $(III)$ in \eqref{equ::i4}, \eqref{equ::i5}, and \eqref{equ::i6}:
\begin{align}
(I') &:= \frac{1}{V_d R_n^{B,*}(x)^d}
\sum_{j=0}^{d-1} \biggl( \frac{1}{B} \sum_{b=1}^B \sum_{i=1}^s w_i^{b,*}  \gamma_{s,i} \biggr)^j \bigl(V_d^{1/d} f(x)^{1/d} R_n^{B,*}(x) \bigr)^{d-1-j}.\label{equ::i4tilde}\\
(II') 
&:= \sum_{i=1}^s \Biggl|\frac{1}{B}\sum_{b=1}^B w_i^{b,*}  \big(\gamma_{s,i} - \mathrm{P}(B(x,\widetilde{R}_{s,(i)}^b(x)))^{1/d}\big) \Biggr|, 
\label{equ::i5tilde} 
\\
(III')
& := \sum_{i=1}^s \Biggl|\frac{1}{B}\sum_{b=1}^B  w_i^{b,*} \bigl( \mathrm{P}(B(x,\widetilde{R}_{s,(i)}^b(x)))^{1/d} - V_d^{1/d} f(x)^{1/d} \widetilde{R}_{s,(i)}^b(x) \bigr) \Biggr|.
\label{equ::i6tilde}
\end{align}
Following a similar derivation as in \eqref{equ::flambdabxerror}, we obtain the error decomposition:
\begin{align}\label{equ::errordecfbstarx}
\bigl| f^{B,*}_n(x) - f(x) \bigr|
\leq (I') \cdot (II') + (I') \cdot (III').
\end{align}

\begin{proposition}\label{pro::I456tilde}
Let Assumption \ref{asp::holder} hold.  
Let $(I')$, $(II')$, and $(III')$ be defined in \eqref{equ::i4tilde}, \eqref{equ::i5tilde}, and \eqref{equ::i6tilde}, respectively.  
Then, there exists an integer $N_4 \in \mathbb{N}$ such that for all $n > N_4$ and any $x$ satisfying $B(x,\widetilde{R}_{s,(k^{b,*})}^b(x))\subset [0,1]^d$ for all $b \in [B]$, with probability $\mathrm{P}^n$ at least $1 - 2/n^2$, $(I')$, $(II')$, and $(III')$ have upper bounds of the same asymptotic order as $(I)$, $(II)$, and $(III)$ in Proposition \ref{pro::I456}.
\end{proposition}

\begin{proof}[of Proposition \ref{pro::I456tilde}]
Let $\Omega_5$ denote the event defined by the statements in Proposition \ref{lem::order:k:w:lambda}.
Applying Lemma \ref{lem::Rrho} to the subset  $\{\widetilde{D}_s^b\}_{b=1}^B$, we obtain that, with probability at least $1 - 1/n^2$, the following holds:
$c_2' (i/s)^{1/d}\leq  \widetilde{R}_{s,(i)}^b(x)\leq c_3' (i/s)^{1/d}$ and $
\mathrm{P}(B(x,\widetilde{R}_{s,(i)}^b(x)))\asymp i/s$
for all $x \in \mathcal{X}$, $b \in [B]$, and $c_n \leq i \leq s$. We define this event as $\Omega_6$.

Since the datasets $D_s^b$ and $\widetilde{D}_s^b$
are independent for $b\in [B]$ and $w^{b,*}$ is the solution to the SRM in \eqref{equ::bagwstar}, we have 
\begin{align*}
&\mathbb{E}\bigl[w_i^{b,*}((\mathrm{P}(B(x,\widetilde{R}_{s,(i)}^b(x))))^{1/d}-\gamma_{s,i})\big| \Omega_5\bigr] \\
& = \mathbb{E}\bigl[w_i^{b,*}\big| \Omega_5\bigr]\cdot \mathbb{E}\bigl[(\mathrm{P}(B(x,\widetilde{R}_{s,(i)}^b(x))))^{1/d}-\gamma_{s,i}\bigr]=0
\end{align*}
for a fixed $x\in \mathcal{X}$,
where the last equality follows from \eqref{equ::pbxribx} and the expectation is taken with respect to the empirical measure $\widetilde{D}_s^b$, conditional on the event $\Omega_5$. 
Following the proof of Lemma \ref{lem::i5lem} and conditioning on the event $\Omega_5$, for all $n > n_1=2d^d+1$, we have 
\begin{align*}
\mathrm{P}\Biggl(\sup_{x\in \mathcal{X},i\in [\, \overline{k} \,]}\Biggl|\frac{1}{B} \sum_{b=1}^B w^{b,*}_i \bigl(\mathrm{P}(B(x,\widetilde{R}_{s,(i)}^b(x)))^{1/d}-\gamma_{s,i}\bigr)\Biggr|\lesssim C_{n,i} \biggl(\frac{i}{s}\biggr)^{1/d} \sqrt{\frac{\log n}{i B}}
\Bigg| \Omega_5\Biggr)
\geq 1-\frac{1}{n^2}.
\end{align*}
Since Proposition \ref{lem::order:k:w:lambda} ensures that $\mathrm{P}(\Omega_5)\geq 1-1/n^2$,
applying the conditional probability formula yields 
\begin{align*}
\mathrm{P}^n\Biggl(\sup_{x\in \mathcal{X},i\in [\, \overline{k} \,]}\Biggl|\frac{1}{B} \sum_{b=1}^B w^{b,*}_i \bigl(\mathrm{P}(B(x,\widetilde{R}_{s,(i)}^b(x)))^{1/d}-\gamma_{s,i}\bigr)\Biggr|\lesssim C_{n,i} \biggl(\frac{i}{s}\biggr)^{1/d} \sqrt{\frac{\log n}{i B}}\Biggr) \geq 1-\frac{2}{n^2}
\end{align*}
for all $n>n_1\vee N_2$, where $N_2$ is the integer specified in Proposition \ref{lem::order:k:w:lambda}.
Denote this event as $\Omega_7$. 
By the union bound, the event $\Omega_5\cap \Omega_6\cap \Omega_7$ holds with probability $\mathrm{P}^n$ at least $1-4/n^2$ for all $n>n_1\vee N_2$. 
Since the events $\Omega_6$ and $\Omega_7$ correspond to the events $\Omega_1$ and $\Omega_2$ in the proof of Proposition \ref{pro::I456}, respectively, and given that  conditions $(i)-(iv)$ in Proposition \ref{pro::I456sur} hold for $w^{b,*}$ and $k^{b,*}$ on $\Omega_5$, there exists $N_4\in \mathbb{N}$ such that the upper bound for $(I')$, $(II')$, and $(III')$ follow for all $n>N_4$ by similar arguments as in the proof of Proposition \ref{pro::I456}. (Note that $N_4$ may differ $N_1$ in that proposition.) The details are omitted.
\end{proof}

\begin{proof}[of Theorem \ref{pro::rateBWDDE}]
Let $\Omega_5$, $\Omega_6$, and $\Omega_7$ be the events defined in the proof of Proposition \ref{pro::I456tilde}.  
Following the arguments therein, the event $\Omega_5 \cap \Omega_6 \cap \Omega_7$ holds with probability $\mathrm{P}^n$ at least $1 - 4/n^2$ for all $n \geq N_4$.  
For the remainder of the proof, we assume that the event $\Omega_5 \cap \Omega_6 \cap \Omega_7$ holds and that $n > N_2^*:=N_4$.

From \eqref{equ::kbstarupper} in the proof of  Proposition \ref{lem::order:k:w:lambda}, we have 
$\overline{k}\leq c_8' (n/\log n)^{1/(d+2)}/2$.
Define
\begin{align*}
\Delta_n:=[c_3'(c_8' (n/\log n)^{1/(d+2)}/(2s))^{1/d}, 1-
c_3'(c_8' (n/\log n)^{1/(d+2)}/(2s))^{1/d}]^d.
\end{align*}
On the event $\Omega_6$, for all $x\in \Delta_n$ and any $b\in [B]$, we have 
\begin{align*}
\widetilde{R}_{s,(k^{b,*})}^b(x)\leq c_3'(k^{b,*}/s)^{1/d}\leq c_3'(c_8' (n/\log n)^{1/(d+2)}/(2s))^{1/d}.
\end{align*}
Thus, for any $y\in B\bigl(x,\widetilde{R}_{s,(k^{b,*})}^b(x)\bigr)$, we have 
\begin{align*}
d(y,\mathbb{R}^d\setminus[0,1]^d)\geq c_3'(c_8' (n/\log n)^{1/(d+2)}/(2s))^{1/d}-\widetilde{R}_{s,(k^{b,*})}^b(x) \geq 0.
\end{align*}
This implies that 
$B\bigl(x,\widetilde{R}_{s,(k^{b,*})}^b(x)\bigr)\subset [0,1]^d$
for all $x\in \Delta_n$ and $b\in [B]$. 
Therefore, from Proposition \ref{pro::I456tilde} and inequality \eqref{equ::errordecfbstarx}, we obtain
\begin{align*}
\big|f_n^{B,*}(x)-f(x)\big|
\lesssim ((\log n)/\overline{k})^{1+1/d} + \bigl((\log n)/ (\overline{k} B)\bigr)^{1/2} + (\overline{k}/s)^{1/d}, \quad x \in \Delta_n.
\end{align*}
Using $k^{b,*}\asymp (n/\log n)^{1/(d+2)}$ from Proposition \ref{lem::order:k:w:lambda}  and substituting the choices of $B$ and $s$,  we obtain
\begin{align*}
\bigl| f_n^{B,*}(x) - f(x) \bigr|
\lesssim n^{-1/(d+2)} (\log n)^{(d+3)/(d+2)}, \qquad x \in \Delta_n
\end{align*}
Integrating over $\Delta_n$, we get
\begin{align}\label{equ::fbfndeltan}
\int_{\Delta_n} |f_n^{B,*}(x)-f(x)| \, dx 
\lesssim n^{-1/(d+2)}(\log n)^{(d+3)/(d+2)}.
\end{align}

On the other hand, on the event $\Omega_5$, condition $(iii)$ in Proposition \ref{pro::I456sur} holds for $w^{b,*}$ and $k^{b,*}$, which implies that $\sum_{i=1}^s w_i^{b,*} i^{1/d}\lesssim (k^{b,*})^{1/d}$ for all $b\in [B]$. Since $\gamma_{s,i}<((i+1/d)/(s+1/d))^{1/d}\leq (2i/s)^{1/d}$ by \eqref{equ::gammasiupper}, we have
\begin{align*}\frac{1}{B} \sum_{b=1}^B \sum_{i=1}^s w_i^{b,*} \gamma_{s,i} \lesssim\frac{1}{B}\sum_{b=1}^B\sum_{i=1}^s w_i^{b,*} (i/s)^{1/d}\lesssim \frac{1}{B}\sum_{b=1}^B (k^{b,*}/s)^{1/d}\lesssim (\overline{k}/s)^{1/d}.
\end{align*}
Following similar arguments as in \eqref{equ::rnbxlower} from Proposition \ref{pro::I456}, we obtain $R_n^{B,*}(x)\gtrsim (\underline{k} / s)^{1/d}$ for all $x\in \mathcal{X}$ on the event $\Omega_5\cap\Omega_6$.
This implies
\begin{align*}
f_n^{B,*}(x)
= \frac{1}{V_d R_n^{B,*}(x)^d}
\biggl( \frac{1}{B} \sum_{b=1}^B \sum_{i=1}^s w_i^{b,*} \gamma_{s,i} \biggr)^d
\lesssim \overline{k}/\underline{k}
\lesssim 1, 
\qquad 
x \in \mathcal{X}.
\end{align*}
Since $\|f\|_{\infty}\leq \overline{c}$ from Assumption \ref{asp::holder}, we conclude that 
$ \|f_n^{B,*}-f\|_{\infty} $ is bounded by a constant.
Therefore, we have 
\begin{align*}
\int_{\mathcal{X}\setminus \Delta_n} |f_n^{B,*}(x)-f(x)| dx \lesssim \mu(\mathcal{X}\setminus\Delta_n)\lesssim ((n/\log n)^{1/(d+2)}/s)^{1/d}\\
\lesssim n^{-1/(d+2)}(\log n)^{(d+3)/(d+2)}.
\end{align*}
Finally, combining this with \eqref{equ::fbfndeltan}, we have
\begin{align*}
\int_{\mathcal{X}}\bigl|f_n^{B,*}(x) - f(x)\bigr|\, dx & = \biggl( \int_{\Delta_n} +\int_{\mathcal{X}\setminus\Delta_n} 
\biggr) \bigl|f_n^{B,*}(x) - f(x)\bigr|\, dx
\lesssim n^{-1/(d+2)}(\log n)^{(d+3)/(d+2)},
\end{align*}
which completes the proof.
\end{proof}

\subsection{Proofs Related to the Convergence Rates of BRDAD} \label{sec::proofsbrdadauc}

In this subsection, we first present the proofs for learning the AUC regret in Section \ref{sec::proofaucregret}.
Then we provide the proof of Theorem \ref{thm::bagaverage} in Section \ref{sec::proofsbrdad}.

\subsubsection{Proofs Related to Section \ref{sec::analysisaucregret}}
\label{sec::proofaucregret}

\begin{proof}[of Proposition \ref{lem:lemfrometatoauc}]
Under the Huber contamination model in Assumption \ref{asp::huber}, let $\eta(x)$ be defined as in \eqref{equ::etax}, and define $\widehat{\eta}(x)=\Pi\cdot f_n^{B,*}(x)^{-1}$. 
By the expression of $f_n^{B,*}(x)$ in \eqref{equ::fbstar}, 
it follows that 
$\eins \{ R_n^{B,*}(X) - R_n^{B,*}(X') > 0 \} = \eins \{ \widehat{\eta}(X) - \widehat{\eta}(X') > 0 \}$  and
$\eins \{ R_n^{B,*}(X) = R_n^{B,*}(X') \} = \eins \{ \widehat{\eta}(X) = \widehat{\eta}(X') \}$.
Consequently, we obtain 
\begin{align*}
& \mathrm{AUC}(R_n^{B,*})
\\
&= \mathbb{E}\bigl[\eins\{(Y-Y')(R_n^{B,*}(X)-R_n^{B,*}(X')>0)\}+\eins\{R_n^{B,*}(X)=R_n^{B,*}(X')\}/2|Y\neq Y'\bigr]\\
&=\mathbb{E}\bigl[\eins\{(Y-Y')(\widehat{\eta}(X)-\widehat{\eta}(X')>0)\}+\eins\{\widehat{\eta}(X)=\widehat{\eta}(X')\}/2|Y\neq Y'\bigr]
= \mathrm{AUC}(\widehat{\eta}).
\end{align*}
Therefore, we have 
$\mathrm{Reg}^{\mathrm{AUC}}(R_n^{B,*})=\mathrm{Reg}^{\mathrm{AUC}}(\widehat{\eta})$.
Applying \citet[Corollary 11]{agarwal2013surrogate}, we obtain
\begin{align}\label{equ::regaucrsw}
\mathrm{Reg}^{\mathrm{AUC}}(R_n^{B,*})  =\mathrm{Reg}^{\mathrm{AUC}}(\widehat{\eta}) 
\leq \frac{1}{\Pi(1-\Pi)}\int_{\mathcal{X}} |\widehat{\eta}(x)-\eta(x)| d \mathrm{P}_X(x). 
\end{align}
From Assumption \ref{asp::holder}, we have $\|f\|_{\infty}\geq \underline{c}$, and given that  $\|f_n^{B,*}\|_{\infty}\geq c$, it follows that
\begin{align*}
|\widehat{\eta}(x)-\eta(x)| =  
\frac{\Pi\bigl|f_n^{B,*}(x)-f(x)\bigr|}{f_n^{B,*}(x)f(x)}\lesssim \bigl|f_n^{B,*}(x)-f(x)\bigr|. 
\end{align*}
Combining this with \eqref{equ::regaucrsw} and the condition $\|f\|_{\infty}\leq c$ from Assumption \ref{asp::holder}, we establish the desired result.
\end{proof}

\subsubsection{Proofs Related to Section \ref{sec::ratebrdad}}
\label{sec::proofsbrdad}

\begin{proof}[of Theorem \ref{thm::bagaverage}]
Let $\Omega_5$, $\Omega_6$, and $\Omega_7$ be the events defined in the proof of Proposition \ref{pro::I456tilde}.  
Following the arguments therein, the event $\Omega_5 \cap \Omega_6 \cap \Omega_7$ holds with probability $\mathrm{P}^n$ at least $1 - 4/n^2$ for all $n \geq N_4$.  
For the subsequent arguments, we assume that $\Omega_5\cap \Omega_6\cap \Omega_7$ holds and that $n > N_2^*=N_4$.

Let $\{c_n\}$ denote the sequence from Lemma \ref{lem::Rrho}.
On the event $\Omega_6$, for all $x\in \mathcal{X}$, we have 
\begin{align*}
R_s^{b,*}(x) 
& = \sum_{i=1}^s w_i^{b,*} \widetilde{R}_{s,(i)}^b(x) 
= \sum_{i=1}^{c_n} w_i^{b,*} \widetilde{R}_{s,(i)}^b(x) + \sum_{i=c_n+1}^s w_i^{b,*} \widetilde{R}_{s,(i)}^b(x)
\\
& \leq \widetilde{R}_{s,(c_n)}^{b,*}(x) + \sum_{i=c_n+1}^s w_i^{b,*} \widetilde{R}_{s,(i)}^b(x)  
\lesssim (c_n / s)^{1/d} + \sum_{i=1}^{s} w_i^{b,*} (i/s)^{1/d},
\end{align*}
By Proposition \ref{lem::order:k:w:lambda}, on the event $\Omega_5$, we have $\sum_{i=1}^s w_{i}^{b,*}i^{1/d}\lesssim \overline{k}^{1/d}$ and $\underline{k}\gtrsim (\log n)^2$.
This implies that $
R_s^{b,*}(x) \lesssim (c_n / s)^{1/d}  + (\overline{k} / s)^{1/d}\lesssim  (\overline{k} / s)^{1/d}$.
Averaging over $b$ in $[B]$, we obtain 
\begin{align}\label{equ::rnbstarlower}
R_n^{B,*}(x) = \frac{1}{B}\sum_{b=1}^B R_s^{b,*}(x) \lesssim  (\overline{k} / s)^{1/d}
\end{align}
On the other hand, using \eqref{equ::gammasiupper}, we have 
\begin{align*}
\sum_{i=1}^s w_i^{b,*} \gamma_{s,i} >\sum_{i=1}^s w_i^{b,*} \biggl(\frac{i+1/d-1}{s+1+1/d}\biggr)^{1/d} \gtrsim \sum_{i=1}^s w_i^{b,*}(i/s)^{1/d},
\end{align*}
where we use the inequality $(i+1/d-1)/(s+1+1/d)\geq (i-1)/(s+2)\geq (i-1)/(2s)$. 
Applying Proposition \ref{lem::order:k:w:lambda} again, we get $\sum_{i=1}^s w_{i}^{b,*}i^{1/d}\gtrsim \overline{k}^{1/d}$. Averaging over $b$ in $[B]$, we have
\begin{align*}
\frac{1}{B}\sum_{b=1}^B\sum_{i=1}^s w_i^{b,*} \gamma_{s,i}  \gtrsim (\overline{k}/s)^{1/d}. 
\end{align*}
Combining this with \eqref{equ::rnbstarlower}, we conclude that $f_n^{B,*}(x)$ is lower bounded by a constant for $x\in \mathcal{X}$. 
Consequently, by Theorem \ref{pro::rateBWDDE} and Proposition \ref{lem:lemfrometatoauc}, we obtain the desired assertion.
\end{proof}

\section{Conclusion} \label{sec::Conclusion}
In this paper, we propose a distance-based algorithm, \textit{Bagged Regularized $k$-Distances for Anomaly Detection} (\textit{BRDAD}), to address challenges in unsupervised anomaly detection. BRDAD mitigates the sensitivity of hyperparameter selection by formulating the problem as a convex optimization task and incorporates bagging to enhances computational efficiency. 
From a theoretical perspective, we establish fast convergence rates for the AUC regret of BRDAD and show that the bagging technique substantially reduces computational complexity. As a by-product, we derive optimal convergence rates for the $L_1$-error of \textit{Bagged Regularized $k$-Distances for Density Estimation} (\textit{BRDDE}), which shares the same weights as BRDAD, further validating the effectiveness of the \textit{Surrogate Risk Minimization} (\textit{SRM}) framework for density estimation.  
On the experimental side, BRDAD is evaluated against distance-based, forest-based, and kernel-based methods on various anomaly detection benchmarks, demonstrating superior performance. Additionally, parameter analysis reveals that choosing an appropriate number of bagging rounds improves performance, making the method well-suited for practical applications.

\acks{The authors would like to thank the reviewers and the action editor for their help and advice, which led
to a significant improvement of the article. 
Hanfang Yang and Yuheng Ma are corresponding authors. 
The research is supported by the Special Funds
of the National Natural Science Foundation of China (Grant No. 72342010). Yuheng Ma
is supported by the Outstanding Innovative Talents Cultivation Funded Programs 2024 of
Renmin University of China. This research is also supported by Public Computing Cloud,
Renmin University of China.
}

\bibliography{CAIBib}

\end{document}